\documentclass[usenames,dvipsnames]{article}

\usepackage[final,nonatbib]{neurips_2020}

\usepackage[utf8]{inputenc} %
\usepackage[T1]{fontenc}    %
\usepackage{hyperref}       %
\usepackage{url}            %
\usepackage{booktabs}       %
\usepackage{amsfonts}       %
\usepackage{nicefrac}       %
\usepackage{microtype}      %
\usepackage{graphicx}
\usepackage[flushleft]{threeparttable}

\usepackage{float}
\usepackage{subfigure}
\usepackage{multirow}
\usepackage{xspace}
\usepackage{caption}
\usepackage{enumitem}

\usepackage{caption}
\usepackage{xspace}
\newcommand{\iid}{i.i.d.\ }
\newcommand{\algopt}{FedDF\xspace} 
\newcommand{\FL}{Federated Learning\xspace} 
\newcommand{\Fl}{Federated learning\xspace} 
\newcommand{\fl}{federated learning\xspace} 
\newcommand{\Ft}{Federated fine-tuning\xspace} 
\newcommand{\ft}{federated fine-tuning\xspace} 
 
\newcommand{\fedavg}{\textsc{FedAvg}\xspace} 
\newcommand{\fedavgM}{\textsc{FedAvgM}\xspace} 
 
\newcommand{\fedprox}{\textsc{FedProx}\xspace} 
\newcommand{\fedma}{\textsc{FedMA}\xspace}

\newcommand{\avglogits}{\textsc{AvgLogits}\xspace}

\usepackage{amsmath,amsthm,amssymb}
\newtheorem{theorem}{Theorem}[section]
\newtheorem{lemma}[theorem]{Lemma}

\newtheorem{remark}[theorem]{Remark}

\providecommand{\abs}[1]{\left\lvert#1\right\rvert}
\providecommand{\norm}[1]{\left\lVert#1\right\rVert}

\providecommand{\R}{\mathbb{R}} %

\providecommand{\Prob}[2]{{\rm Pr}_{#1}\left[#2\right] }

\providecommand{\derive}[2]{\frac{{\partial}{#1}}{\partial #2} }  %

\DeclareMathOperator*{\argmin}{arg\,min}

\providecommand{\dd}{\mathbf{d}}

\providecommand{\pp}{\mathbf{p}}
\providecommand{\qq}{\mathbf{q}}

\providecommand{\vv}{\mathbf{v}}

\providecommand{\xx}{\mathbf{x}}

\providecommand{\cC}{\mathcal{C}}
\providecommand{\cD}{\mathcal{D}}

\providecommand{\cH}{\mathcal{H}}

\providecommand{\cL}{\mathcal{L}}
\providecommand{\cM}{\mathcal{M}}

\providecommand{\cP}{\mathcal{P}}

\providecommand{\cS}{\mathcal{S}}

\providecommand{\cX}{\mathcal{X}}
\providecommand{\cY}{\mathcal{Y}}

\providecommand{\hxx}{\hat{\xx}}

\providecommand{\loss}{\ell}

\providecommand{\kl}{\text{KL}}
\providecommand{\lepochs}{\ensuremath{E}}  %

\providecommand{\clientfrac}{\ensuremath{C}}    %

\usepackage{algorithm}
\usepackage[noend]{algpseudocode}

\newcommand{\mycaptionof}[2]{\captionof{#1}{#2}}

\errorcontextlines\maxdimen

\makeatletter
    \newcommand*{\algrule}[1][\algorithmicindent]{\makebox[#1][l]{\hspace*{.5em}\thealgruleextra\vrule height \thealgruleheight depth \thealgruledepth}}%
\newcommand*{\thealgruleextra}{}
\newcommand*{\thealgruleheight}{.75\baselineskip}
\newcommand*{\thealgruledepth}{.25\baselineskip}

\newcount\ALG@printindent@tempcnta
\def\ALG@printindent{%
    \ifnum \theALG@nested>0%
        \ifx\ALG@text\ALG@x@notext%
        \else
            \unskip
            \addvspace{-1pt}%
            \ALG@printindent@tempcnta=1
            \loop
                \algrule[\csname ALG@ind@\the\ALG@printindent@tempcnta\endcsname]%
                \advance \ALG@printindent@tempcnta 1
            \ifnum \ALG@printindent@tempcnta<\numexpr\theALG@nested+1\relax%
            \repeat
        \fi
    \fi
    }%
\usepackage{etoolbox}
\patchcmd{\ALG@doentity}{\noindent\hskip\ALG@tlm}{\ALG@printindent}{}{\errmessage{failed to patch}}
\makeatother

\newbox\statebox
\newcommand{\myState}[1]{%
    \setbox\statebox=\vbox{#1}%
    \edef\thealgruleheight{\dimexpr \the\ht\statebox+1pt\relax}%
    \edef\thealgruledepth{\dimexpr \the\dp\statebox+1pt\relax}%
    \ifdim\thealgruleheight<.75\baselineskip
        \def\thealgruleheight{\dimexpr .75\baselineskip+1pt\relax}%
    \fi
    \ifdim\thealgruledepth<.25\baselineskip
        \def\thealgruledepth{\dimexpr .25\baselineskip+1pt\relax}%
    \fi
    \State #1%
    \def\thealgruleheight{\dimexpr .75\baselineskip+1pt\relax}%
    \def\thealgruledepth{\dimexpr .25\baselineskip+1pt\relax}%
}

\usepackage{xcolor}
\usepackage{transparent}
\definecolor{newcolor}{rgb}{0.8,1,1}
\newcommand{\fedavgcolor}{Thistle}
\newcommand{\feddfcolor}{SpringGreen}

\usepackage{tikz}
\usetikzlibrary{fit,calc}

\pgfdeclarelayer{back}
\pgfdeclarelayer{front}
\pgfsetlayers{back,main,front}

\makeatletter
\pgfkeys{%
  /tikz/on layer/.code={
    \pgfonlayer{#1}\begingroup
    \aftergroup\endpgfonlayer
    \aftergroup\endgroup
  },
  /tikz/node on layer/.code={
    \gdef\node@@on@layer{%
      \setbox\tikz@tempbox=\hbox\bgroup\pgfonlayer{#1}\unhbox\tikz@tempbox\endpgfonlayer\egroup}
    \aftergroup\node@on@layer
  },
  /tikz/end node on layer/.code={
    \endpgfonlayer\endgroup\endgroup
  }
}
\def\node@on@layer{\aftergroup\node@@on@layer}
\makeatother

\colorlet{fedavgcolorablock}{\fedavgcolor}
\colorlet{feddfcolorblock}{\feddfcolor}

\usepackage[textwidth=3.5cm]{todonotes}

\title{
Ensemble Distillation for Robust Model Fusion in Federated Learning
}

\author{%
Tao Lin\thanks{Equal contribution.}~,
Lingjing Kong$^*$,
Sebastian U. Stich,
Martin Jaggi. \\
MLO, EPFL, Switzerland \\
\texttt{\{tao.lin, lingjing.kong, sebastian.stich, martin.jaggi\}@epfl.ch}
}

\begin{document}

\maketitle

\begin{abstract}
	\FL (FL) is a machine learning setting where many devices collaboratively train a machine learning model while keeping the training data decentralized.
	In most of the current training schemes the central model is refined by averaging the parameters of the server model and the updated parameters from the client side.
	However, directly averaging model parameters is only possible if all models have the same structure and size,
	which could be a %
	restrictive constraint in many scenarios.\\
	In this work we investigate more powerful and more flexible aggregation schemes for FL.
	Specifically, we propose ensemble distillation for model fusion,
	i.e.\ training the central classifier through unlabeled data on the outputs of the models from the clients.
	This knowledge distillation technique mitigates privacy risk and cost to the same extent as the baseline FL algorithms,
	but allows flexible aggregation over heterogeneous client models that can differ e.g.\ in size, numerical precision or structure.
	We show in extensive empirical experiments on various CV/NLP datasets (CIFAR-10/100, ImageNet, AG News, SST2) and settings (heterogeneous models/data)
	that the server model can be trained much faster, requiring fewer communication rounds than any existing FL technique so far.
\end{abstract}

\section{Introduction}

Federated Learning (FL)~%
has emerged as an important machine learning paradigm in which a federation of clients participate in collaborative training of a centralized model~\cite{shokri2015privacy,mcmahan2016communication,smith2017federated,%
	caldas2018leaf,bonawitz2019federated,li2019federated,kairouz2019advances}.
The clients send their model parameters to the server but never their private training datasets, thereby ensuring a basic level of privacy.
Among the key challenges in federated training are communication overheads and delays (one would like to train the central model with as few communication rounds as possible),
and client heterogeneity: the training data (non-i.i.d.-ness), as well as hardware and computing resources, can change drastically among clients,
for instance when training on commodity mobile devices.

Classic training algorithms in FL, such as federated averaging (\fedavg)~\cite{mcmahan2016communication} and its recent adaptations~\cite{mohri2019agnostic,li2020fair,hsu2019measuring,%
	karimireddy2019scaffold,hsu2020federated,reddi2020adaptive}, are all based on directly averaging of the participating client's \emph{parameters} and can hence only be applied if all client's models have the same size and structure. %
In contrast, ensemble learning methods~\cite{Shan2017learning,furlanello2018born,anil2018large,Dvornik_2019_ICCV,park2019feed,liu2019knowledge,wu2019personid}
allow to combine multiple heterogeneous weak classifiers by averaging the \emph{predictions} of the individual models instead.
However, applying ensemble learning techniques directly in FL is infeasible in practice due to the large number of participating clients,
as it requires keeping weights of all received models on the server and performing naive ensembling (logits averaging) for inference.

To enable \underline{fed}erated learning in more realistic settings,
we propose to use ensemble \underline{d}istillation~\cite{bucilua2006model,hinton2015distilling}
for robust model \underline{f}usion (\algopt).
Our scheme leverages unlabeled data or artificially generated examples
(e.g.\ by a GAN's generator~\cite{goodfellow2014generative})
to aggregate knowledge from all received (heterogeneous) client models.
We demonstrate with thorough empirical results that
our ensemble distillation approach not only addresses the existing quality loss issue~\cite{hsieh2019non}
of Batch Normalization (BN)~\cite{ioffe2015batch} for networks in a homogeneous FL system,
but can also break the knowledge barriers among heterogeneous client models.
Our main contributions are:
\begin{itemize}[nosep,leftmargin=12pt]
	\item We propose a distillation framework for robust federated model fusion,
	      which allows for heterogeneous client models and data,
	      and is robust to the choices of neural architectures.
	\item We show in extensive numerical experiments on various CV/NLP datasets
	      (CIFAR-10/100, ImageNet, AG News, SST2) and settings (heterogeneous models and/or data)
	      that the server model can be trained much faster,
	      requiring fewer communication rounds than any existing FL technique.
\end{itemize}
We further provide insights on when \algopt can outperform \fedavg
(see also Fig.~\ref{fig:illustration_problems_in_fl}
that highlights an intrinsic limitation of parameter averaging based approaches)
and what factors influence \algopt.

\begin{figure*}[t]
	\centering
	\includegraphics[width=1.\textwidth]{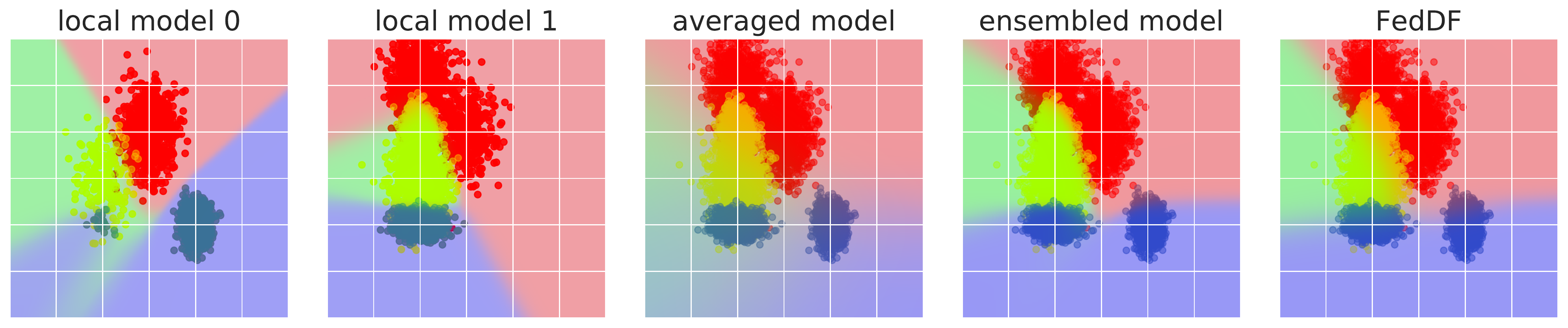}
	\caption{\small
		\textbf{Limitations of \fedavg.}
		We consider a toy example of a 3-class classification task with a 3-layer MLP,
		and display the decision boundaries (probabilities over RGB channels) on the input space.
		The left two figures show the individually trained local models.
		The right three figures evaluate aggregated models
		and the global data distribution;
		the averaged model results in much blurred decision boundaries.
		The used datasets are displayed in Figure~\ref{fig:illustration_problems_in_fl_detail} (Appendix~\ref{subsec:description_toy_example}).
	}
	\vspace{-1.5em}
	\label{fig:illustration_problems_in_fl}
\end{figure*}

\section{Related Work}
\paragraph{Federated learning.}
The classic algorithm in FL, \fedavg~\cite{mcmahan2016communication}, or local SGD~\cite{lin2020dont} when all devices are participating,
performs weighted parameter average over the client models after several local SGD updates with weights proportional to the size of each client's local data.
Weighting schemes based on client loss are investigated in~\cite{mohri2019agnostic,li2020fair}.
To address the difficulty of directly averaging model parameters,
\cite{singh2019model,Wang2020Federated} propose to
use optimal transport and other alignment schemes to
first align or match individual neurons of the neural nets layer-wise
before averaging the parameters. %
However, these layer-based alignment schemes necessitate client models with the same number of layers and structure,
which is restrictive in heterogeneous systems in practice.
\\
Another line of work aims to improve local client training,
i.e., client-drift problem caused by the heterogeneity of local data~\cite{li2018federated,karimireddy2019scaffold}.
For example, \fedprox~\cite{li2018federated} incorporates a proximal term for the local training.
Other techniques like acceleration, recently appear in~\cite{hsu2019measuring,hsu2020federated,reddi2020adaptive}.

\paragraph{Knowledge distillation.}
Knowledge distillation for neural networks is first introduced in~\cite{bucilua2006model,hinton2015distilling}.
By encouraging the student model to approximate the output logits of the teacher model,
the student is able to imitate the teacher's behavior
with marginal quality loss~\cite{romero2014fitnets,zagoruyko2016paying,kim2018paraphrasing,tung2019similarity,koratana19a,huang2017like,ahn2019variational,tian2019contrastive}.
Some work study the ensemble distillation,
i.e., distilling the knowledge of an ensemble of teacher models to a student model.
To this end, existing approaches either average the logits from the ensemble of teacher models~\cite{Shan2017learning,furlanello2018born,anil2018large,Dvornik_2019_ICCV},
or extract knowledge from the feature level~\cite{park2019feed,liu2019knowledge,wu2019personid}.
\\
Most of these schemes rely on using the original training data for the distillation process.
In cases where real data is unavailable,
some recent work \cite{nayak2019zero,micaelli2019zero} demonstrate that distillation
can be accomplished by crafting pseudo data either from the weights of the teacher model or
through a generator adversarially trained with the student.
\algopt can be combined with all of these approaches. In this work, we consider
unlabeled datasets for ensemble distillation,
which could be either collected from other domains or directly generated from a pre-trained generator.

\paragraph{Comparison with close FL work.}
Guha \textit{et al}.~\cite{guha2019one} propose ``one-shot fusion'' through unlabeled data
for SVM loss objective,
whereas we consider multiple-round scenarios on diverse neural architectures and tasks.
FD~\cite{jeong2018communication} utilizes distillation to reduce FL communication costs.
To this end, FD synchronizes logits per label which are accumulated during the local training.
The averaged logits per label (over local steps and clients)
will then be used as a distillation regularizer for the next round's local training.
Compared to \fedavg, FD experiences roughly 15\% quality drop on MNIST.
In contrast, \algopt shows superior learning performance over \fedavg
and can significantly reduce the number of communication rounds
to reach target accuracy on diverse challenging tasks.
\\
FedMD~\cite{li2019fedmd} and the recently proposed Cronus~\cite{chang2019cronus}
consider learning through averaged logits per sample on a public dataset.
After the initial pre-training on the labeled public dataset,
FedMD learns on the public and private dataset iteratively for personalization,
whereas in Cronus, the public dataset (with soft labels) is used jointly with local private data for the local training.
As FedMD trains client models simultaneously on both labeled public and private datasets,
the model classifiers have to
include all classes from both datasets.
Cronus, in its collaborative training phase, mixes public and private data for local training.
Thus for these methods, the public dataset construction requires careful deliberation and
even prior knowledge on clients' private data.
Moreover, how these modifications impact local training quality remains unclear.
\algopt faces no such issues: we show that \algopt is robust to distillation dataset selection
and the distillation is performed on the server side, leaving local training unaffected.
We include a detailed discussion with FedMD, Cronus in Appendix~\ref{sec:detailed_related_work}.
\looseness=-1

When preparing this version, we also notice other contemporary work~\cite{sun2020federated, chen2020feddistill, zhou2020distilled,he2020group}
and we defer discussions to Appendix~\ref{sec:detailed_related_work}.

\section{Ensemble Distillation for Robust Model Fusion}
\vspace{-1em} %

\begin{algorithm}[!h]
	\begin{algorithmic}[1]
		\Procedure{Server}{}
		\For{each communication round $t = 1, \dots, T$}
		\myState{$S_t \leftarrow$ random subset ($C$ fraction) of the $K$ clients}
		\For{ each client $k \in \cS_t$ \textbf{in parallel} }
		\myState{ $\hxx_{t}^k \leftarrow \text{Client-LocalUpdate}(k, \xx_{t - 1})$}  \Comment{detailed in Algorithm~\ref{alg:client_update}.}
		\EndFor

		\myState{ initialize for model fusion $\xx_{t, 0} \leftarrow \sum_{k \in \cS_t} \frac{n_k}{ \sum_{k \in S_t} n_k } \hxx_{t}^k$ } \label{ln:average}
		\For{ $j$ in $\{ 1, \dots, N \}$ }
		\myState{ sample a mini-batch of samples $\dd$, from e.g.\ (1) an unlabeled dataset, (2) a generator}
		\myState{ use ensemble of $ \{ \hxx_{t}^k \}_{k \in \cS_t}$ to update server student $\xx_{t, j-1}$ through \avglogits}
		\EndFor
		\myState{ $\xx_{t} \leftarrow \xx_{t, N}$}

		\EndFor
		\myState{\Return $\xx_T$}
		\EndProcedure
	\end{algorithmic}

	\mycaptionof{algorithm}{\small
		Illustration of \algopt on $K$ homogeneous clients
		(indexed by $k$) for $T$ rounds, $n_k$ denotes the number of data points per client and $C$ the fraction of clients participating in each round.
		The server model is initialized as $\xx_{0}$.
		While \fedavg just uses the averaged models $\xx_{t,0}$,
		we perform $N$ iterations of server-side model fusion on top (line 7 -- line 10).
	}
	\label{alg:homogeneous_framework}
\end{algorithm}

\vspace{-1.em}
In this section, we first introduce the core idea of the proposed Federated Distillation Fusion (\algopt).
We then comment on its favorable characteristics
and discuss possible extensions.

\paragraph{Ensemble distillation.}
We first discuss the key features of \algopt for the special case of homogeneous models, i.e. when all clients share the same network architecture (Algorithm~\ref{alg:homogeneous_framework}).
For model fusion, the server
distills the ensemble of  $\abs{\cS_t}$ client teacher models to one single server student model. For the distillation, the teacher models are evaluated on mini-batches
of unlabeled data on the server (forward pass)
and their logit outputs (denoted by $f(\hat \xx^k_t,\dd)$ for mini-batch $\dd$)
are used to train the student model on the server:
\begin{small}
	\begin{align} \label{eq:distillation_scheme}
		\xx_{t, j} & := \xx_{t, j - 1} - \eta
		\derive{ \kl \left( \sigma \left( \frac{1}{ \abs{ \cS_t } } \sum_{k \in \cS_t} f ( \hxx^k_t, \dd) \right), \sigma \left( f ( \xx_{t, j-1}, \dd) \right) \right) }{ \xx_{t, j-1} }  \tag{\avglogits} \,.
	\end{align}%
\end{small}%
Here $\kl$ stands for Kullback–Leibler divergence,
$\sigma$ is the softmax function,
and $\eta$ is the stepsize.

\algopt can easily be extended to heterogeneous FL systems
(Algorithm~\ref{alg:heterogeneous_framework} and Figure~\ref{fig:diagram_heterogeneous_system}
in Appendix~\ref{appendix:algo_description}).
We assume the system contains $p$ distinct model prototype groups
that potentially differ in neural architecture, structure and numerical precision.
By ensemble distillation, each model architecture group acquires knowledge
from logits averaged over \textit{all} received models,
thus mutual beneficial information can be shared across architectures;
in the next round, each activated client receives the corresponding fused prototype model.
Notably, as the fusion takes place on the server side, there is no additional burden and interference on clients.

\paragraph{Utilizing unlabeled/generated data for distillation.}
Unlike most existing ensemble distillation methods that
rely on \emph{labeled} data from the training domain,
we demonstrate the feasibility of achieving model fusion by using \emph{unlabeled} datasets from other domains
for the sake of privacy-preserving FL.
Our proposed method also allows the use of synthetic data from a pre-trained generator (e.g.\ GAN\footnote{
	GAN training is not involved in all stages of FL and cannot steal clients' data.
	Data generation is done by the (frozen) generator before the FL training by performing inference on random noise.
	Adversarially involving GAN's training during the FL training may cause the privacy issue,
	but it is beyond the scope of this paper.
})
as distillation data to alleviate potential limitations (e.g.\ acquisition, storage) of real unlabeled datasets.

\paragraph{Discussions on privacy-preserving extension.}
Our proposed model fusion framework in its simplest form---like most existing FL methods---requires to
exchange models between the server and each client,
resulting in potential privacy leakage due to e.g. memorization present in the models.
Several existing protection mechanisms can be added to our framework to protect clients from adversaries.
These include adding differential privacy~\cite{geyer2017differentially} for client models,
or performing hierarchical and decentralized model fusion
through synchronizing locally inferred logits e.g.\ on \emph{random} public data\footnote{
	For instance, these data can be generated locally from identical generators with a controlled random state.
}, as in the recent work~\cite{chang2019cronus}.
We leave further explorations of this aspect for future work.

\section{Experiments}\label{sec:experiments}%
\subsection{Setup}%
\paragraph{Datasets and models.}
We evaluate the learning of different SOTA FL methods on both CV and NLP tasks,
on architectures of ResNet~\cite{he2016deep}, VGG~\cite{simonyan2014very}, ShuffleNetV2~\cite{ma2018shufflenet} and DistilBERT~\cite{sanh2019distilbert}.
We consider \fl CIFAR-10/100~\cite{krizhevsky2009learning} and
ImageNet~\cite{krizhevsky2012imagenet} (down-sampled to image resolution 32 for computational feasibility~\cite{chrabaszcz2017downsampled}) from scratch for CV tasks;
while for NLP tasks,
we perform federated fine-tuning on a 4-class news classification dataset (AG News~\cite{zhang2015character})
and a 2-class classification task (Stanford Sentiment Treebank, SST2~\cite{socher-etal-2013-recursive}).
The validation dataset is created for CIFAR-10/100, ImageNet, and SST2,
by holding out $10\%$, $1\%$ and $1\%$ of the original training samples respectively;
the remaining training samples are used as the training dataset (before partitioning client data) and the whole procedure is controlled by random seeds.
We use validation/test datasets on the server and report the test accuracy over three different random seeds.

\paragraph{Heterogeneous distribution of client data.}\ %
We use the Dirichlet distribution as in~\cite{yurochkin2019bayesian,hsu2019measuring}
to create disjoint non-\iid client training data.
The value of $\alpha$ controls the degree of non-i.i.d.-ness:
$\alpha \!=\! 100$ mimics identical local data distributions,
and the smaller $\alpha$ is, the more likely the clients hold examples from only one class (randomly chosen).
Figure~\ref{fig:understanding_learning_behaviors_resnet8_cifar10_kt_different_non_iid_degrees}
visualizes how samples are distributed among $20$ clients for CIFAR-10 on different $\alpha$ values;
more visualizations are shown in Appendix~\ref{appendix:detailed_exp_setup}.

\paragraph{Baselines.}
\algopt is designed for effective model fusion on the server,
considering the accuracy of the global model on the test dataset.
Thus we omit the comparisons to methods designed
for personalization (e.g.\ FedMD~\cite{li2019fedmd}), security/robustness (e.g. Cronus~\cite{chang2019cronus}),
and communication efficiency
(e.g.~\cite{jeong2018communication}, known for poorer performance than \fedavg).
We compare \algopt with SOTA FL methods,
including
1) \fedavg~\cite{mcmahan2016communication},
2) \fedprox~\cite{li2018federated} (for better local training under heterogeneous systems),
3) accelerated \fedavg a.k.a.\ \fedavgM\footnote{
	The performance of \fedavgM is coupled
	with local learning rate, local training epochs, and the number of communication rounds.
	The preprints~\cite{hsu2019measuring,hsu2020federated} consider small learning rate for at least 10k communication rounds;
	while we use much fewer communication rounds, which sometimes result in different observations.
}~\cite{hsu2019measuring,hsu2020federated},
and 4) \fedma\footnote{
	\fedma does not support BN or residual connections,
	thus the comparison is only performed on VGG-9.
}~\cite{Wang2020Federated} (for better model fusion).
We elaborate on the reasons for omitted numerical comparisons in Appendix~\ref{sec:detailed_related_work}.

\paragraph{The local training procedure.}
The FL algorithm randomly samples a fraction ($\clientfrac$) of clients
per communication round for local training.
For the sake of simplicity,
the local training in our experiments uses a constant learning rate (no decay),
no Nesterov momentum acceleration, and no weight decay.
The hyperparameter tuning procedure is deferred to Appendix~\ref{appendix:detailed_exp_setup}.
Unless mentioned otherwise the learning rate is set to $0.1$ for ResNet-like nets, $0.05$ for VGG,
and $1e{-}5$ for DistilBERT.

\paragraph{The model fusion procedure.}
We evaluate the performance of \algopt
by utilizing either randomly sampled data from existing (unlabeled) datasets\footnote{
	Note the actual computation expense for distillation is determined by
	the product of the number of distillation steps and distillation mini-batch size ($128$ in all experiments),
	rather than the distillation dataset size.
}
or BigGAN's generator~\cite{brock2018large}.
Unless mentioned otherwise we use CIFAR-100 and downsampled ImageNet (image size $32$)
as the distillation datasets for \algopt on CIFAR-10 and CIFAR-100 respectively.
Adam with learning rate $1e{-}3$ (w/ cosine annealing) is used to distill knowledge from the ensemble of received local models.
We employ early-stopping to stop distillation after the validation performance plateaus for $1e{3}$
steps (total $1e{4}$ update steps).
The hyperparameter used for model fusion is kept constant over all tasks.

\subsection{Evaluation on the Common \FL Settings}\label{subsec:evaluations_on_common_fl_settings}%
\paragraph{Performance overview for different FL scenarios.}
We can observe from Figure~\ref{fig:understanding_learning_behaviors_resnet8_cifar10_kt_different_non_iid_degrees}
that \algopt consistently outperforms \fedavg for all client fractions and non-\iid degrees
when the local training is reasonably sufficient (e.g.\ over 40 epochs).

\algopt benefits from larger numbers of local training epochs. This is because the performance of
the model ensemble is highly dependent on the diversity among its individual models~\cite{kuncheva2003measures,sollich1996learning}.
Thus longer local training leads to greater diversity and quality of the ensemble and hence a better distillation result for the fused model.
This characteristic is desirable in practice as it helps reduce the communication overhead in FL systems.
In contrast, the performance of \fedavg saturates and even degrades with the increased number of local epochs,
which is consistent with observations in~\cite{mcmahan2016communication,caldas2018leaf,Wang2020Federated}.
As \algopt focuses on better model fusion on the server side,
it is orthogonal to recent techniques (e.g.~\cite{shoham2019overcoming,karimireddy2019scaffold,deng2020adaptive})
targeting the issue of non-\iid local data.
We believe combining \algopt with these techniques can lead to a more robust FL, which we leave as future work\footnote{
	We include some preliminary results to illustrate the compatibility of~\algopt in Table~\ref{tab:resnet_8_cifar10_compatibility_with_other_methods}
	(Appendix~\ref{appendix:ablation_study}).
}.

\paragraph{Ablation study of \algopt.}
We provide detailed ablation study for \algopt in Appendix~\ref{appendix:ablation_study}
to identify the source of the benefits.
For example, Table~\ref{tab:resnet_8_cifar10_importance_of_init}
justifies the importance of using the uniformly averaged local models as a starting model
(line 6 in Algorithm~\ref{alg:homogeneous_framework} and line 11 in Algorithm~\ref{alg:heterogeneous_framework}),
for the quality of ensemble distillation in~\algopt.
We further investigate the effect of different optimizers (for on-server ensemble distillation)
on the \fl performance
in Table~\ref{tab:resnet_8_cifar10_different_local_training_scheme}
and Table~\ref{tab:resnet_8_cifar10_effect_of_different_optimizers_for_distillation}.

\begin{figure*}[!t]
	\centering
	\subfigure[\small
		$\alpha \!=\! 100$.
	]{
		\begin{minipage}{.31\textwidth}
			\includegraphics[width=1\textwidth,]{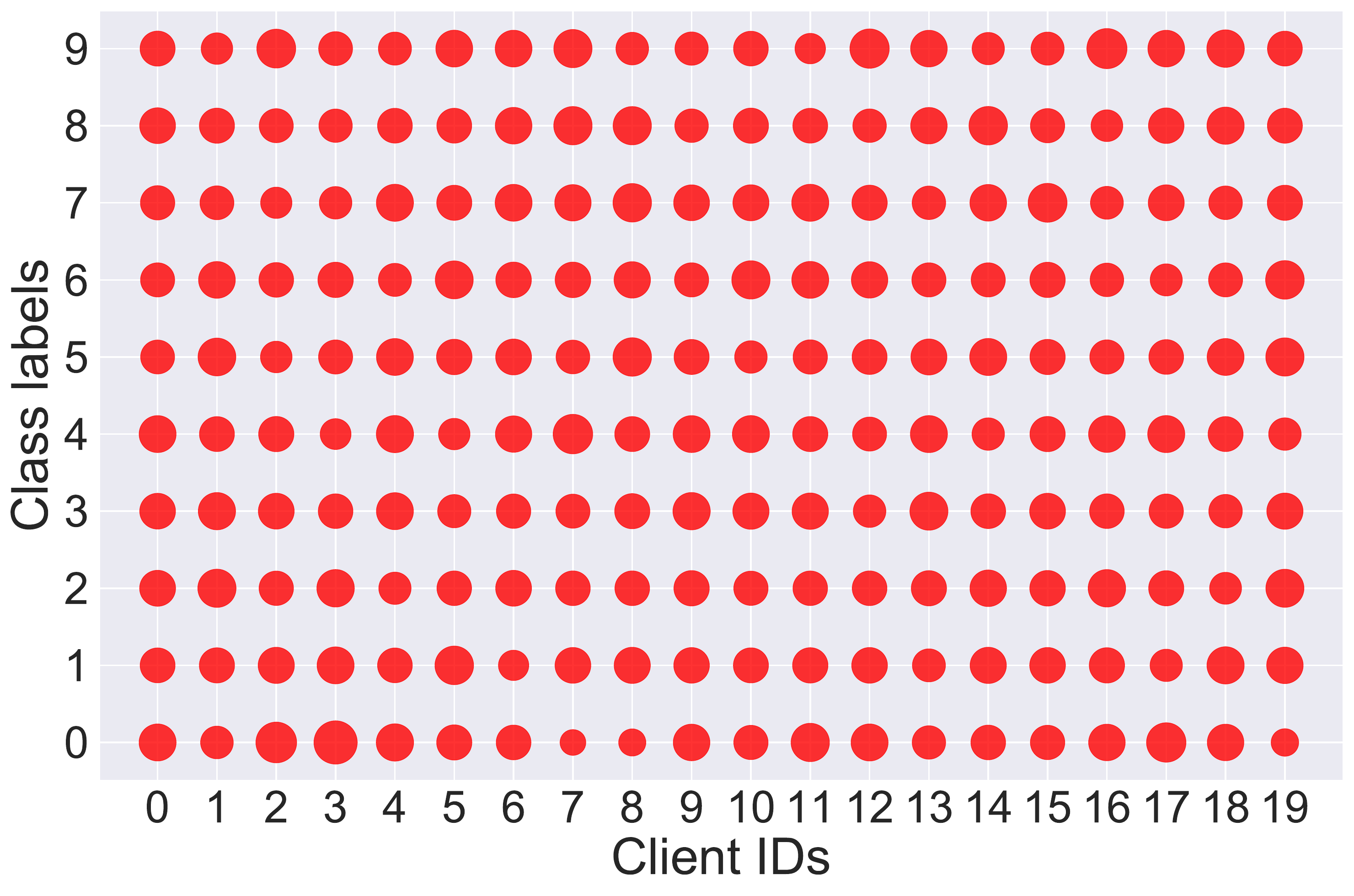}
			\vfill
			\includegraphics[width=1\textwidth,]{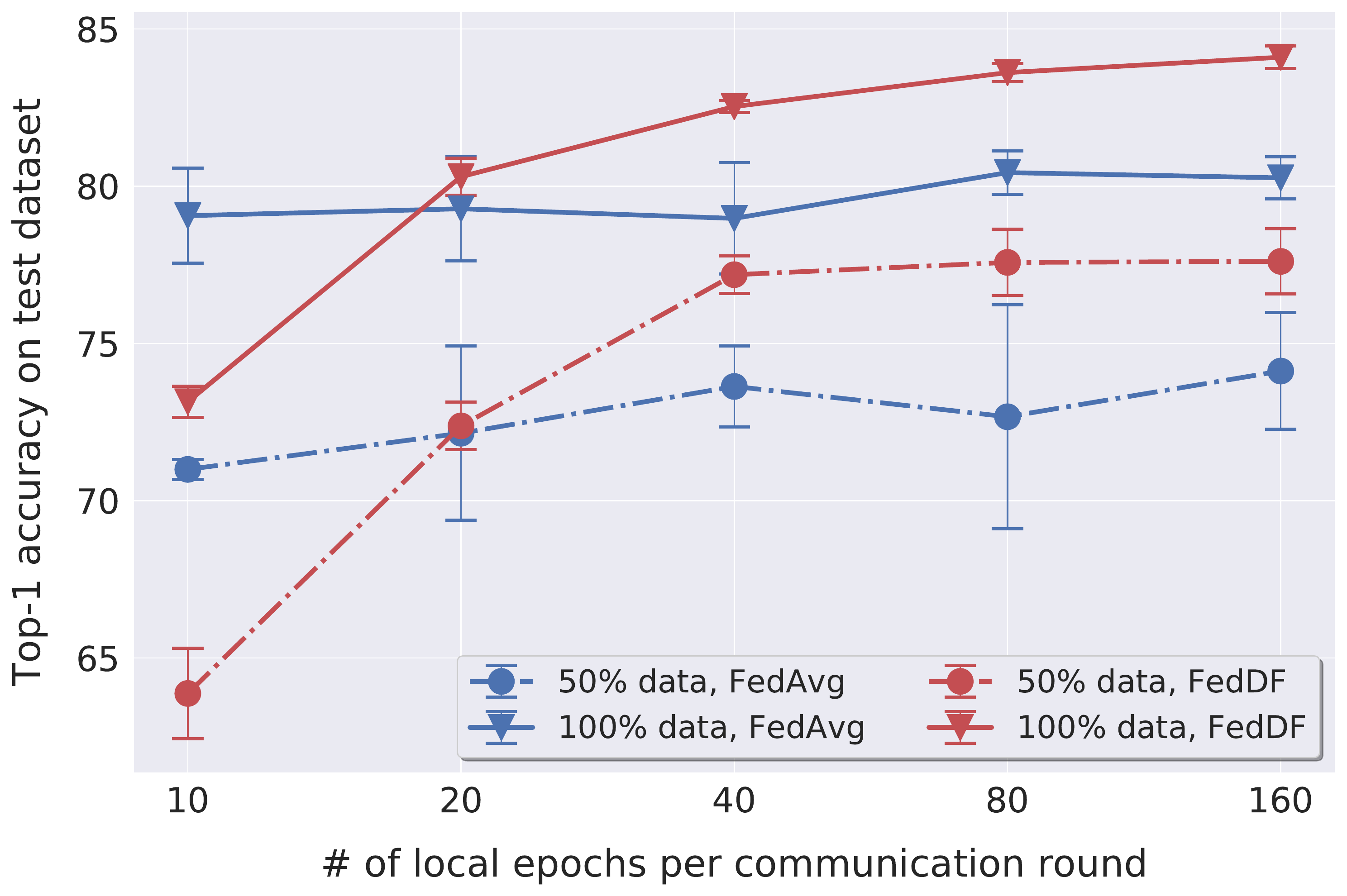}
		\end{minipage}
		\label{fig:resnet8_cifar10_non_iid_100_fedavg_vs_kt_results}
	}
	\hfill
	\subfigure[\small
		$\alpha \!=\! 1$.
	]{
		\begin{minipage}{.31\textwidth}
			\includegraphics[width=1\textwidth,]{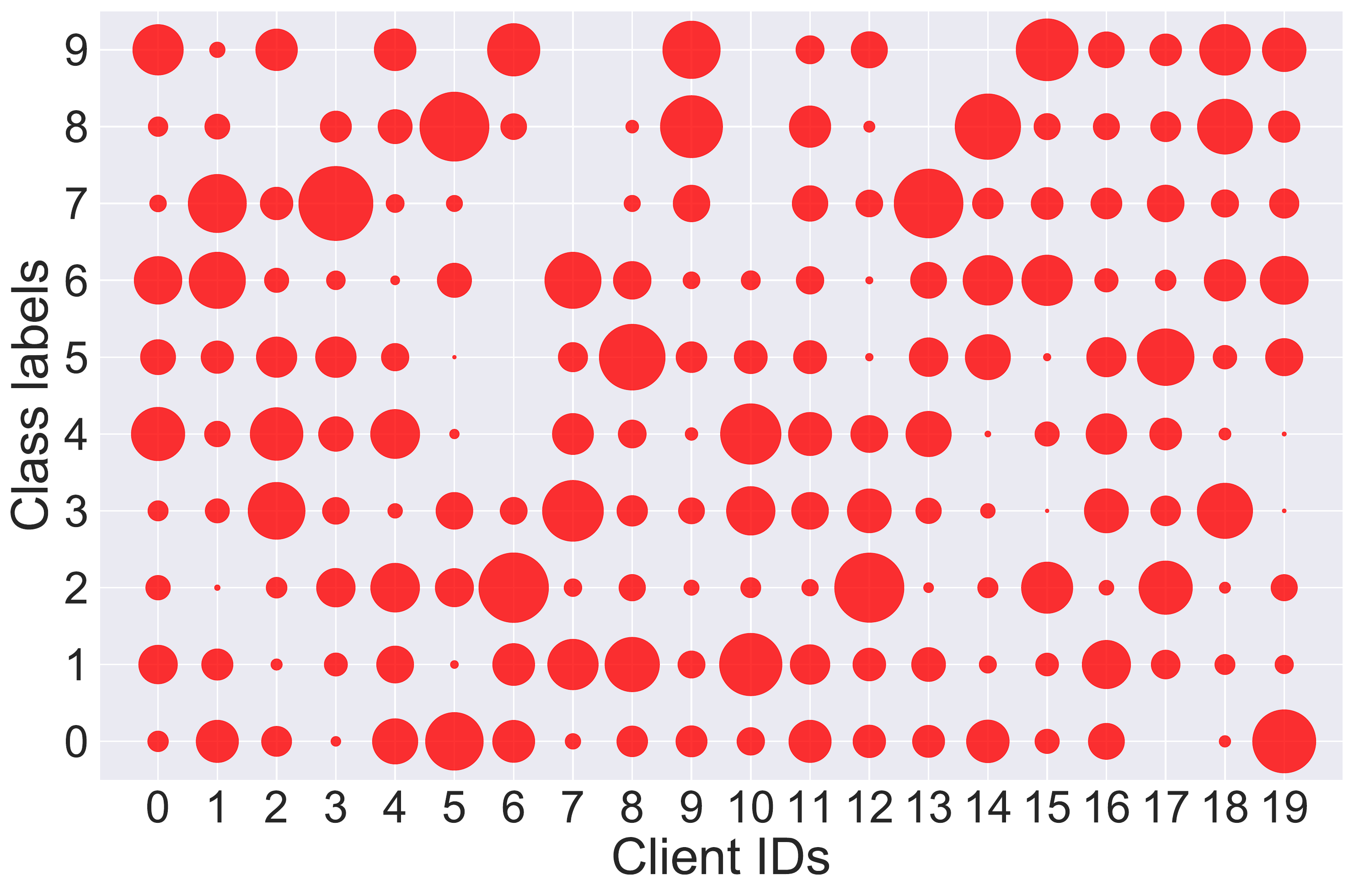}
			\vfill
			\includegraphics[width=1\textwidth,]{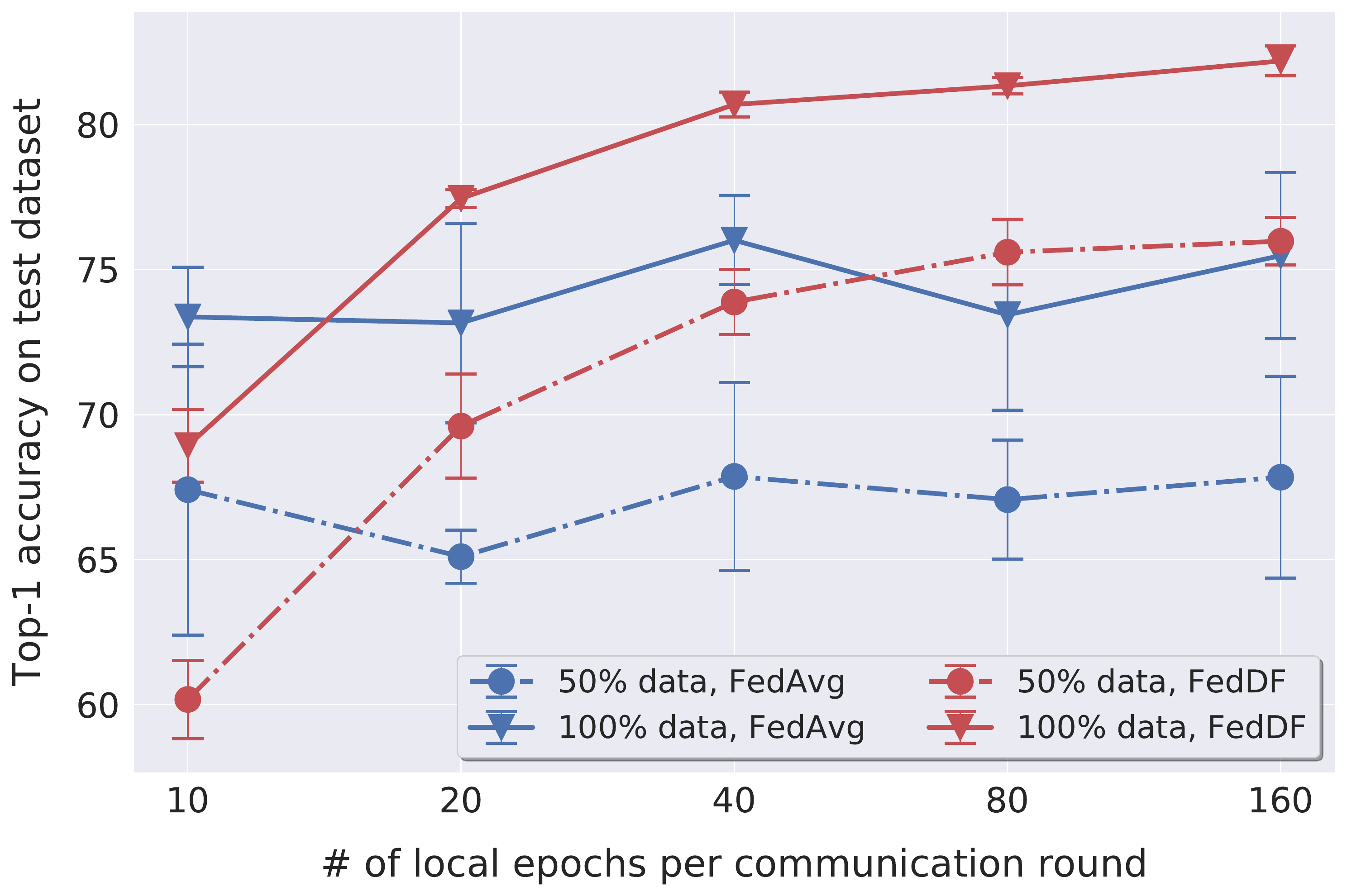}
		\end{minipage}
		\label{fig:resnet8_cifar10_non_iid_1_fedavg_vs_kt_results}
	}
	\hfill
	\subfigure[\small
		$\alpha \!=\! 0.01$.
	]{
		\begin{minipage}{.31\textwidth}
			\includegraphics[width=1\textwidth,]{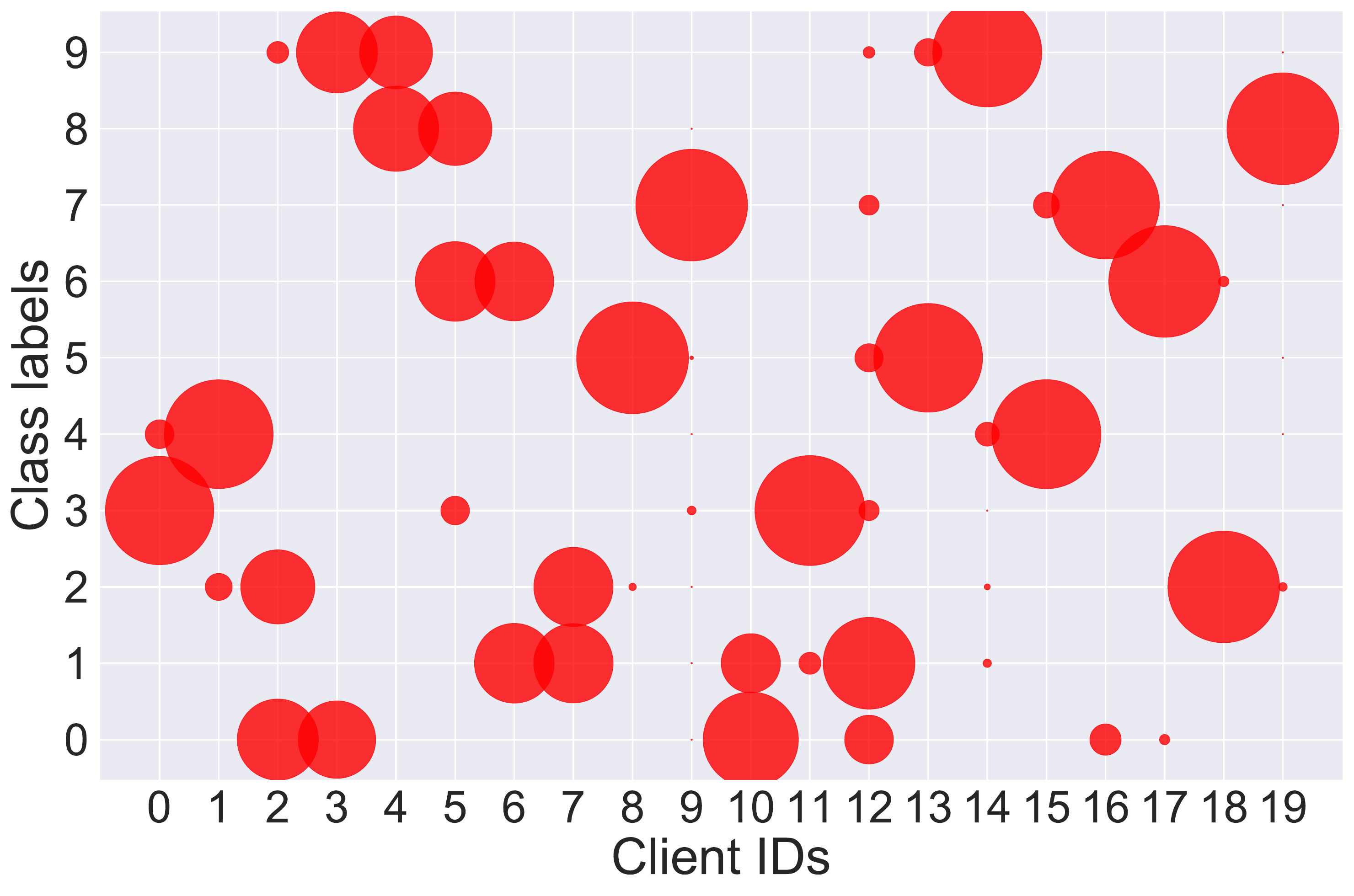}
			\vfill
			\includegraphics[width=1\textwidth,]{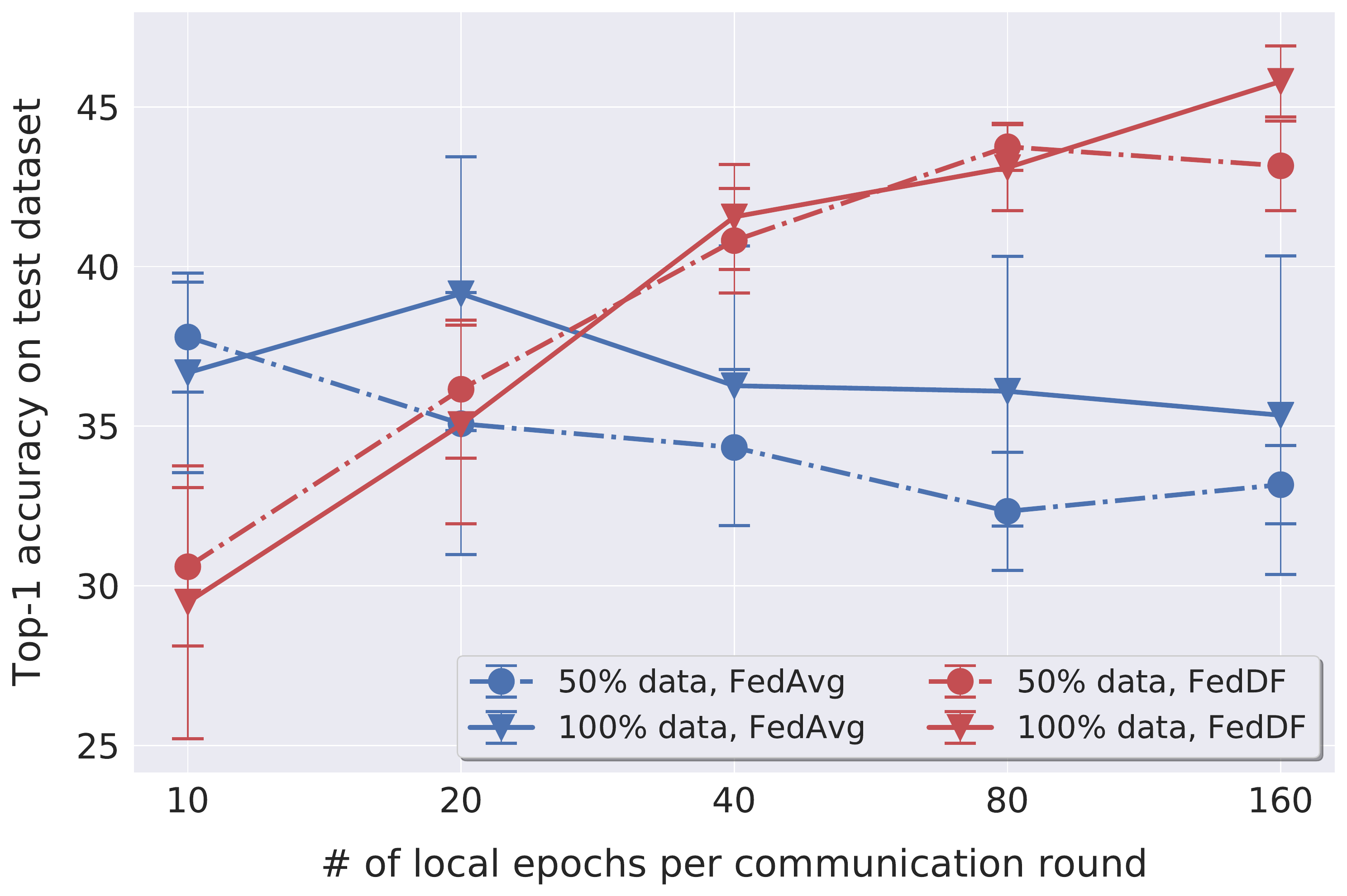}
		\end{minipage}
		\label{fig:resnet8_cifar10_non_iid_001_fedavg_vs_kt_results}
	}
	\vspace{-1.em}
	\caption{\small
		\textbf{Top:} \textbf{Illustration of \# of samples per class allocated to each client} (indicated by dot sizes), for different Dirichlet distribution $\alpha$ values.
		\textbf{Bottom:} \textbf{Test performance} of \textbf{\algopt} and \textbf{\fedavg} on \textbf{CIFAR-10} with \textbf{ResNet-8},
		for different local training settings:
		non-\iid degrees $\alpha$, data fractions,
		and \# of local epochs per communication round.
		We perform $100$ communication rounds, and active clients are sampled with ratio $\clientfrac \!=\! 0.4$ from a total of $20$ clients.
		Detailed learning curves in these scenarios can be found in Appendix~\ref{appendix:standard_fl_scenario}.
	}
	\vspace{-0.5em}
	\label{fig:understanding_learning_behaviors_resnet8_cifar10_kt_different_non_iid_degrees}
\end{figure*}

\begin{table}[!t]
	\vspace{-0.5em}
	\caption{\small
		\textbf{Evaluating different FL methods in different scenarios}
		(i.e.\ different client sampling fractions, \# of local epochs and target accuracies),
		in terms of \textbf{the number of communication rounds to reach target top-1 test accuracy}.
		We evaluate on ResNet-8 with CIFAR-10.
		For each communication round,
		a fraction $\clientfrac$ of the total $20$ clients are randomly selected.
		$T$ denotes the specified target top-1 test accuracy.
		Hyperparameters are fine-tuned for each method (\fedavg, \fedprox, and \fedavgM);
		\algopt uses the optimal learning rate from \fedavg.
		The performance upper bound of (tuned) centralized training is $86\%$ (trained on all local data).
	}
	\label{tab:cifar10_resnet8_detailed_comparison}
	\centering
	\resizebox{1.\textwidth}{!}{%
		\begin{tabular}{lcllllll}
			\toprule
			& 				   								& \multicolumn{6}{c}{The number of communication rounds to reach target performance $T$}                                                     \\ \cmidrule(lr){3-8}
			&  \parbox{2cm}{\centering Local \\ epochs}  	& \multicolumn{2}{c}{$\clientfrac \!=\! 0.2$} 	& \multicolumn{2}{c}{$\clientfrac \!=\! 0.4$}	& \multicolumn{2}{c}{$\clientfrac \!=\! 0.8$}	\\ \cmidrule(lr){3-4} \cmidrule(lr){5-6} \cmidrule(lr){7-8}
			                        &      & $\alpha \!=\! 1, T\!=\!80\%$ & $\alpha \!=\! 0.1, T\!=\! 75\%$ & $\alpha \!=\! 1, T\!=\!80\%$ & $\alpha \!=\! 0.1, T\!=\! 75\%$ & $\alpha \!=\! 1, T\!=\!80\%$ & $\alpha \!=\! 0.1, T\!=\! 75\%$ \\ \midrule
			\fedavg                 & $1$  & $350 \pm 31$                 & $546 \pm 191$                   & $246 \pm 41$                 & $445 \pm 8$                     & $278 \pm 83$                 & $361 \pm 111$                   \\
			                        & $20$ & $144 \pm 51$                 & $423 \pm 105$                   & $97 \pm 29$                  & $309 \pm 88$                    & $103 \pm 26$                 & $379 \pm 151$                   \\
			                        & $40$ & $130 \pm 13$                 & $312 \pm 87$                    & $104 \pm 52$                 & $325 \pm 82$                    & $100 \pm 76$                 & $312 \pm 110$                   \\ \midrule
			\fedprox                & $20$ & $99 \pm 61$                  & $346 \pm 12$                    & $91 \pm 40$                  & $235 \pm 41$                    & $92 \pm 21$                  & $237 \pm 93$                    \\
			                        & $40$ & $115 \pm 17$                 & $270 \pm 96$                    & $87 \pm 49$                  & $229 \pm 79$                    & $80 \pm 44$                  & $284 \pm 130$                   \\ \midrule
			\fedavgM                & $20$ & $92 \pm 15$                  & $299 \pm 85$                    & $92 \pm 46$                  & $221 \pm 29$                    & $97 \pm 37$                  & $235 \pm 129$                   \\
			                        & $40$ & $135 \pm 52$                 & $322 \pm 99$                    & $78 \pm 28$                  & $224 \pm 38$                    & $83 \pm 34$                  & $232 \pm 11$                    \\ \midrule
			\textbf{\algopt} (ours) & $20$ & $\textbf{61} \pm 24$         & $\textbf{102} \pm 42$           & $\textbf{28} \pm 10$         & $\textbf{51} \pm 4$             & $\textbf{22} \pm 1$          & $\textbf{33} \pm 18$            \\
			                        & $40$ & $\textbf{28} \pm 6$          & $\textbf{80} \pm 25$            & $\textbf{20} \pm 4$          & $\textbf{39} \pm 10$            & $\textbf{14} \pm 2$          & $\textbf{20} \pm 4$             \\
			\bottomrule
		\end{tabular}%
	}
	\vspace{-1.em}
\end{table}

\paragraph{Detailed comparison of \algopt with other SOTA \fl methods for CV tasks.}
Table~\ref{tab:cifar10_resnet8_detailed_comparison} summarizes the results
for various degrees of non-\iid data, local training epochs and client sampling fractions.
In all scenarios, \algopt requires significantly fewer communication
rounds than other SOTA methods to reach designated target accuracies.
The benefits of \algopt can be further pronounced by taking more local training epochs
as illustrated in Figure~\ref{fig:understanding_learning_behaviors_resnet8_cifar10_kt_different_non_iid_degrees}.

All competing methods have strong difficulties with increasing data heterogeneity (non-\iid data, i.e. smaller~$\alpha$),
while \algopt shows significantly improved robustness to data heterogeneity.
In most scenarios in Table~\ref{tab:cifar10_resnet8_detailed_comparison}, the reduction of $\alpha$ from $1$ to $0.1$ almost triples
the number of communication rounds for \fedavg, \fedprox and \fedavgM to reach
target accuracies, whereas less than twice the number of rounds are sufficient for \algopt.

Increasing the sampling ratio makes a more noticeable positive impact on \algopt compared to other methods.
We attribute this to the fact that an ensemble tends to improve in robustness and quality,
with a larger number of reasonable good participants, and hence results in better model fusion.
Nevertheless, even in cases with a very low sampling fraction (i.e.\ $ \clientfrac \!=\! 0.2 $),
\algopt still maintains a considerable leading margin over the closest competitor.

\paragraph{Comments on Batch Normalization.}
Batch Normalization (BN)~\cite{ioffe2015batch} is the current workhorse in convolutional deep learning tasks
and has been employed by default in most SOTA CNNs~\cite{he2016deep,huang2017densely,ma2018shufflenet,tan2019efficientnet}.
However, it often fails on heterogeneous training data.
Hsieh \textit{et al}.~\cite{hsieh2019non} recently examined the non-\iid data `quagmire' for distributed learning
and point out that replacing BN by Group Normalization (GN)~\cite{wu2018group}
can alleviate some of the quality loss brought by BN due to the discrepancies between local data distributions.

As shown in Table~\ref{tab:resnet8_cifar10_impact_of_normalization},
despite additional effort on architecture modification and hyperparameter tuning (i.e.\ the number of groups in GN),
baseline methods with GN replacement still lag much behind \algopt.
\algopt provides better model fusion which is robust to non-\iid  data, and is compatible with BN, thus
avoids extra efforts for modifying the standard SOTA neural architectures.
Figure~\ref{fig:resnet8_cifar10_impact_of_normalization} in Appendix~\ref{appendix:empirical_understanding_fedavg} shows the complete learning curves.

\begin{table*}[!t]
	\centering
	\caption{\small
		\textbf{The impact of normalization techniques} (i.e.\ BN, GN)
		for ResNet-8 on CIFAR
		($20$ clients with $\clientfrac \!=\! 0.4$, $100$ communication rounds,
		and $40$ local epochs per round).
		We use a constant learning rate and tune other hyperparameters.
		The distillation dataset of~\algopt for CIFAR-100 is ImageNet (with image size of $32$).
	}
	\vspace{-0.5em}
	\resizebox{.9\textwidth}{!}{%
		\begin{tabular}{llccccc}
			\toprule
			& 						 & \multicolumn{5}{c}{Top-1 test accuracy of different methods}            \\ \cmidrule{2-7}
			Datasets                   &                    & \fedavg, w/ BN   & \fedavg, w/ GN   & \fedprox, w/ GN  & \fedavgM, w/ GN  & \textbf{\algopt}, w/ BN   \\ \midrule
			\multirow{2}{*}{CIFAR-10}  & $\alpha \!=\! 1$   & $76.01 \pm 1.53$ & $78.57 \pm 0.22$ & $76.32 \pm 1.98$ & $77.79 \pm 1.22$ & $\textbf{80.69} \pm 0.43$ \\
			                           & $\alpha \!=\! 0.1$ & $62.22 \pm 3.88$ & $68.37 \pm 0.50$ & $68.65 \pm 0.77$ & $68.63 \pm 0.79$ & $\textbf{71.36} \pm 1.07$ \\ \midrule
			\multirow{2}{*}{CIFAR-100} & $\alpha \!=\! 1$   & $35.56 \pm 1.99$ & $42.54 \pm 0.51$ & $42.94 \pm 1.23$ & $42.83 \pm 0.36$ & $\textbf{47.43} \pm 0.45$ \\
			                           & $\alpha \!=\! 0.1$ & $29.14 \pm 1.91$ & $36.72 \pm 1.50$ & $35.74 \pm 1.00$ & $36.29 \pm 1.98$ & $\textbf{39.33} \pm 0.03$ \\
			\bottomrule
		\end{tabular}%
	}
	\label{tab:resnet8_cifar10_impact_of_normalization}
	\vspace{-1em}
\end{table*}

\begin{table}[!t]
	\caption{\small
		\textbf{Top-1 test accuracy of \fl CIFAR-10 on VGG-9 (w/o BN)},
		for $20$ clients with $\clientfrac \!=\! 0.4$, $\alpha \!=\! 1$ and $100$ communication rounds ($40$ epochs per round).
		We by default drop dummy predictors.
	}
	\label{tab:cifar10_vgg11_test_accuracy}
	\centering
	\resizebox{.9\textwidth}{!}{%
		\begin{threeparttable}
			\begin{tabular}{lccccc}
				\toprule
				& \multicolumn{4}{c}{\centering Top-1 test accuracy @ communication round} &     \\ \cmidrule{2-6}
				Methods                           & $5$                       & $10$                      & $20$                      & $50$                      & $100$                     \\ \bottomrule
				\fedavg (w/o drop-worst)          & $45.72 \pm 30.95$         & $51.06 \pm 35.56$         & $53.22 \pm 37.43$         & $29.60 \pm 40.66$         & $7.52 \pm 4.29$           \\
				\fedma (w/o drop-worst) \tnote{1} & $23.41 \pm 0.00$          & $27.55 \pm 0.10$          & $41.56 \pm 0.08$          & $60.35 \pm 0.03$          & $65.0 \pm 0.02$           \\
				\fedavg                           & $64.77 \pm 1.24$          & $70.28 \pm 1.02$          & $75.80 \pm 1.36$          & $77.98 \pm 1.81$          & $78.34 \pm 1.42$          \\
				\fedprox                          & $63.86 \pm 1.55$          & $71.85 \pm 0.75$          & $75.57 \pm 1.16$          & $77.85 \pm 1.96$          & $78.60 \pm 1.91$          \\
				\textbf{\algopt}                  & $\textbf{66.08} \pm 4.14$ & $\textbf{72.80} \pm 1.59$ & $\textbf{75.82} \pm 2.09$ & $\textbf{79.05} \pm 0.54$ & $\textbf{80.36} \pm 0.63$ \\
				\bottomrule
			\end{tabular}%
			\begin{small}
				\begin{tablenotes}
					\item[1] \fedma does not support drop-worst operation
					due to its layer-wise communication/fusion scheme.
					The number of local training epochs per layer is $5$ ($45$ epochs per model)
					thus results in stabilized training.
					More details can be found in Appendix~\ref{appendix:detailed_exp_setup}.
				\end{tablenotes}
			\end{small}
		\end{threeparttable}
	}
	\vspace{-0.5em}
\end{table}

We additionally evaluate architectures originally designed without BN (i.e. VGG),
to demonstrate the broad applicability of \algopt.
Due to the lack of normalization layers, VGG is vulnerable to non-i.i.d. local distributions.
We observe that received models on the server might
output random prediction results on the validation/test dataset
and hence give rise to uninformative results overwhelmed by large variance
(as shown in Table~\ref{tab:cifar10_vgg11_test_accuracy}).
We address this issue by a simple treatment\footnote{
	Techniques (e.g. Krum, Bulyan),
	can be adapted to further improve the robustness or defend against attacks.
}, ``drop-worst'',
i.e., dropping learners
with random predictions on the server validation dataset
(e.g.\ $10\%$ accuracy for CIFAR-10),
in each round before applying model averaging and/or ensemble distillation.
Table~\ref{tab:cifar10_vgg11_test_accuracy} examines the FL methods (\fedavg, \fedprox, \fedma
and \algopt) on VGG-9;
\algopt consistently outperforms other methods by a large margin for different communication rounds.

\paragraph{Extension to NLP tasks for \ft of DistilBERT.}
Fine-tuning a pre-trained transformer language model like BERT~\cite{devlin2018bert}
yields SOTA results on various NLP benchmarks~\cite{wang2018glue, wang2019superglue}.
DistilBERT~\cite{sanh2019distilbert} is a lighter version of BERT with only marginal quality loss on downstream tasks.
As a proof of concept,
in Figure~\ref{fig:distilbert_results}
we consider \ft of DistilBERT on non-\iid local data ($\alpha \!=\! 1$, depicted in Figure~\ref{fig:nlp_partition}).
For both AG News and SST2 datasets, \algopt achieves significantly faster convergence than \fedavg
and consistently outperforms the latter.

\begin{figure*}[!t]
	\vspace{-0.5em}
	\centering
	\subfigure[\small
		AG News.
	]{
		\includegraphics[width=.45\textwidth,]{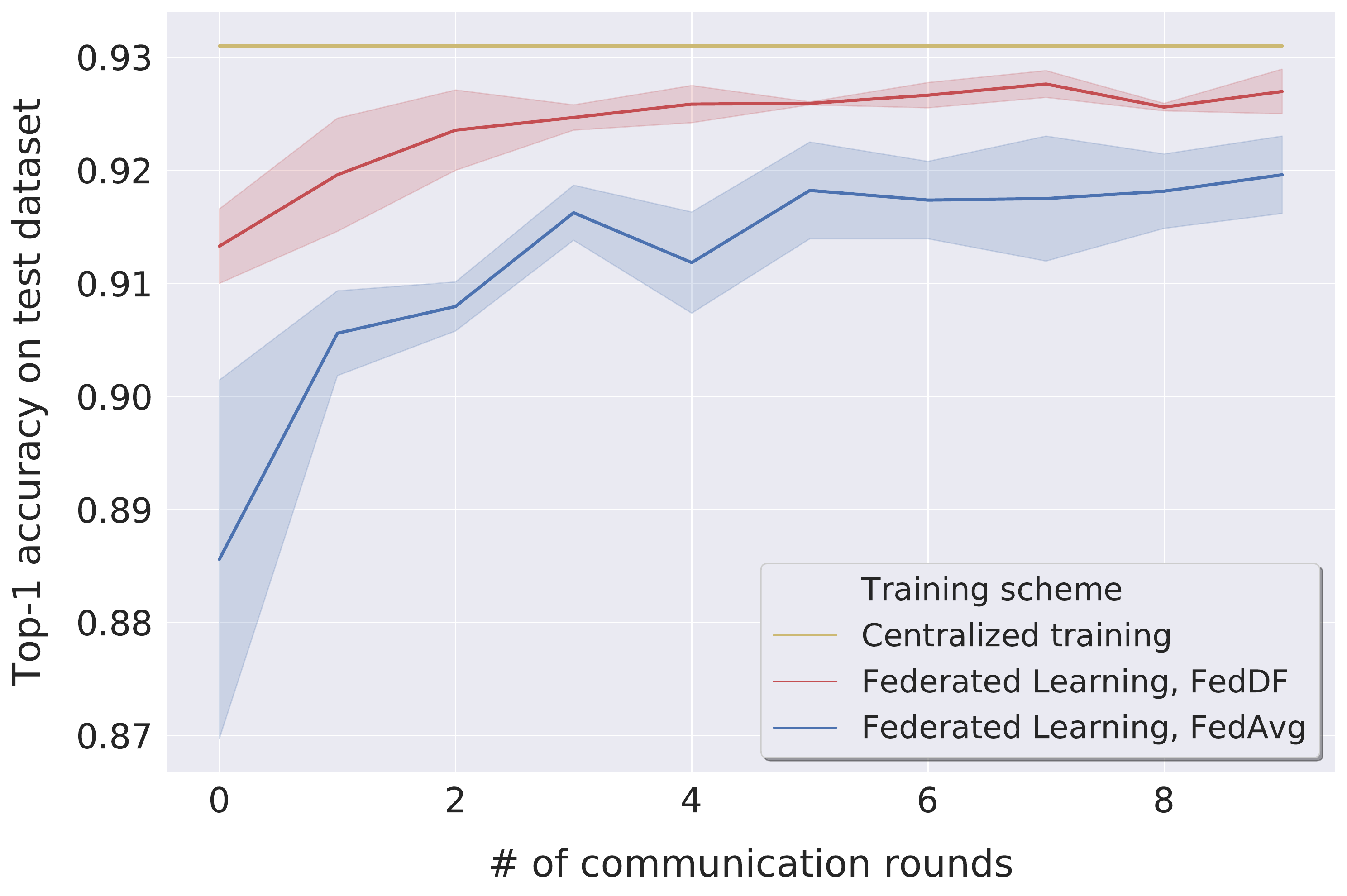}
		\label{fig:agnews_distilbert}
	}
	\hfill
	\subfigure[\small
		SST2.
	]{
		\includegraphics[width=.45\textwidth,]{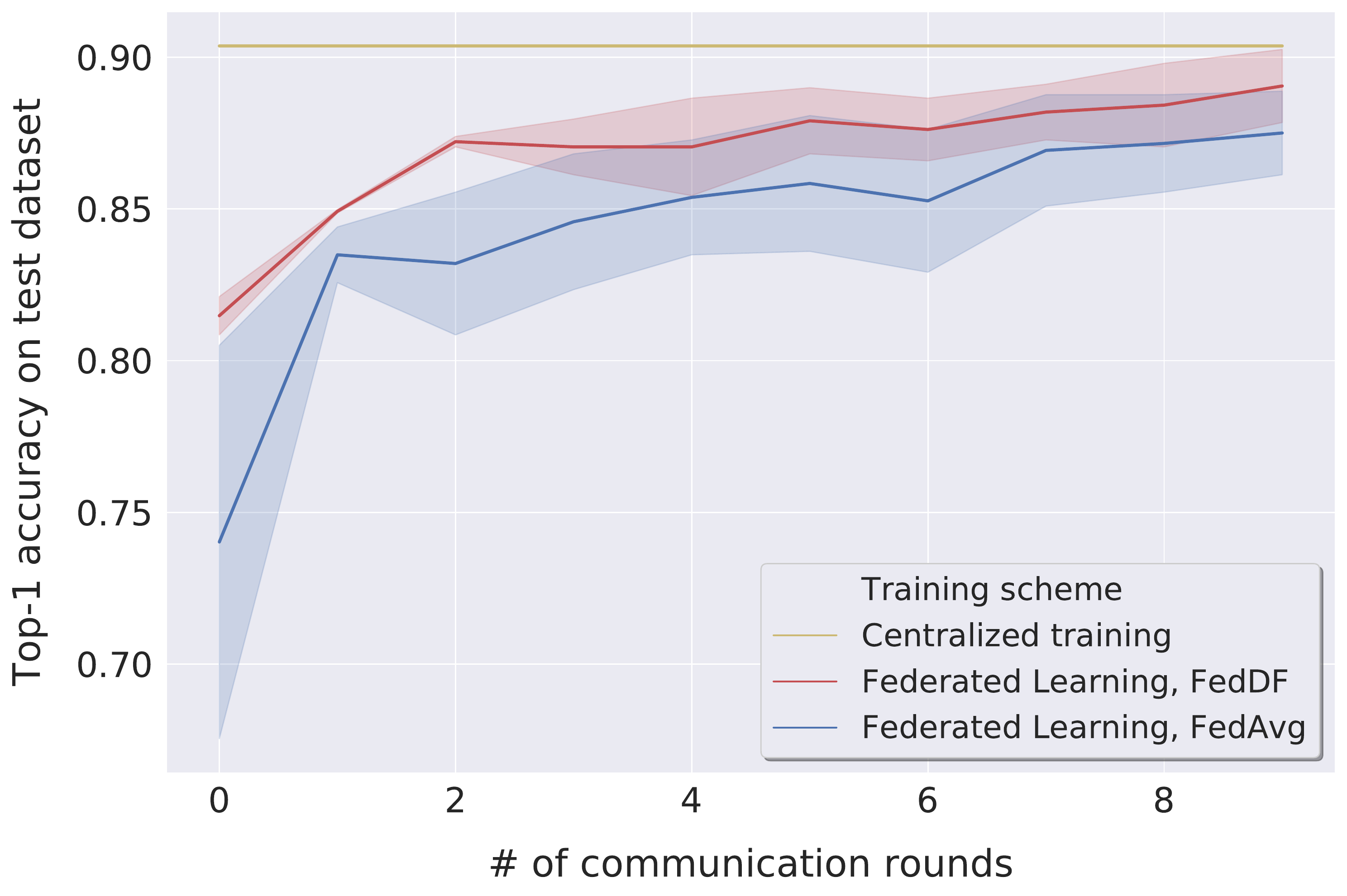}
		\label{fig:sst2_distilbert}
	}
	\vspace{-1em}
	\caption{\small
		\textbf{\Ft DistilBERT} on (a) AG News and (b) SST-2.
		For simplicity, we consider 10 clients with $C\!=\!100\%$ participation ratio and $\alpha \!=\! 1$;
		the number of local training epochs per communication round ($10$ rounds in total) is set to $10$ and $1$ respectively.
		The $50\%$ of the original training dataset is used for the \ft (for all methods)
		and the left $50\%$ is used as the unlabeled distillation dataset for \algopt.
	}
	\vspace{-1em}
	\label{fig:distilbert_results}
\end{figure*}

\subsection{Case Studies} \label{subsec:case_study}
\paragraph{\Fl for low-bit quantized models.}
FL for the Internet of Things (IoT) involves edge devices with diverse hardware, e.g.\ different computational capacities.
Network quantization is hence of great interest to FL by representing the activations/weights in low precision,
with benefits of significantly reduced local computational footprints and communication costs.
Table~\ref{tab:binary_net} examines the model fusion performance for binarized ResNet-8~\cite{rastegari2016xnor,hubara2017quantized}.
\algopt can be on par with or outperform \fedavg by a noticeable margin,
without introducing extra GN tuning overheads.

\begin{table}[!t]
	\caption{\small
		\textbf{Federated learning with low-precision models (1-bit binarized ResNet-8) on CIFAR-10}.
		For each communication round ($100$ in total),
		$40\%$ of the total $20$ clients ($\alpha \!=\! 1$) are randomly selected.
	}
	\label{tab:binary_net}
	\centering
	\resizebox{.8\textwidth}{!}{%
		\begin{tabular}{cccc}
			\toprule
			Local Epochs & ResNet-8-BN (\fedavg) & ResNet-8-GN (\fedavg)     & ResNet-8-BN (\textbf{\algopt}) \\ \midrule
			20           & $44.38 \pm 1.21$      & $\textbf{59.70} \pm 1.65$ & $59.49 \pm 0.98$               \\
			40           & $43.91 \pm 3.26$      & $64.25 \pm 1.31$          & $\textbf{65.49} \pm 0.74$      \\
			80           & $47.62 \pm 1.84$      & $65.99 \pm 1.29$          & $\textbf{70.27} \pm 1.22$      \\ \bottomrule
		\end{tabular}%
	}
	\vspace{-1em}
\end{table}

\paragraph{\Fl on heterogeneous systems.}
Apart from non-\iid local distributions,
another major source of heterogeneity in FL systems manifests in neural architectures~\cite{li2019fedmd}.
Figure~\ref{fig:heterogeneous_system} visualizes the training dynamics of \algopt and \fedavg\footnote{
	Model averaging is only performed among models with identical structures.
}
in a heterogeneous system with three distinct architectures,
i.e., ResNet-20, ResNet-32, and ShuffleNetV2.
On CIFAR-10/100 and ImageNet,
\algopt dominates \fedavg on test accuracy in each communication round with much less variance.
Each fused model exhibits marginal quality loss compared to the ensemble performance,
which suggests unlabeled datasets from other domains are sufficient for model fusion.
Besides, the gap between the fused model and the ensemble one widens
when the training dataset contains a much larger number of classes\footnote{
	\# of classes is a proxy measurement for distribution shift;
	labels are not used in our distillation procedure.
}
than that of the distillation dataset.
For instance, the performance gap is negligible on CIFAR-10,
whereas on ImageNet, the gap increases to around $6 \%$.
In Section~\ref{subsec:understanding},
we study this underlying interaction between training data and unlabeled distillation data in detail.

\begin{figure*}[!t]
	\centering
	\subfigure[\small
		CIFAR-10.
	]{
		\includegraphics[width=0.31\textwidth,]{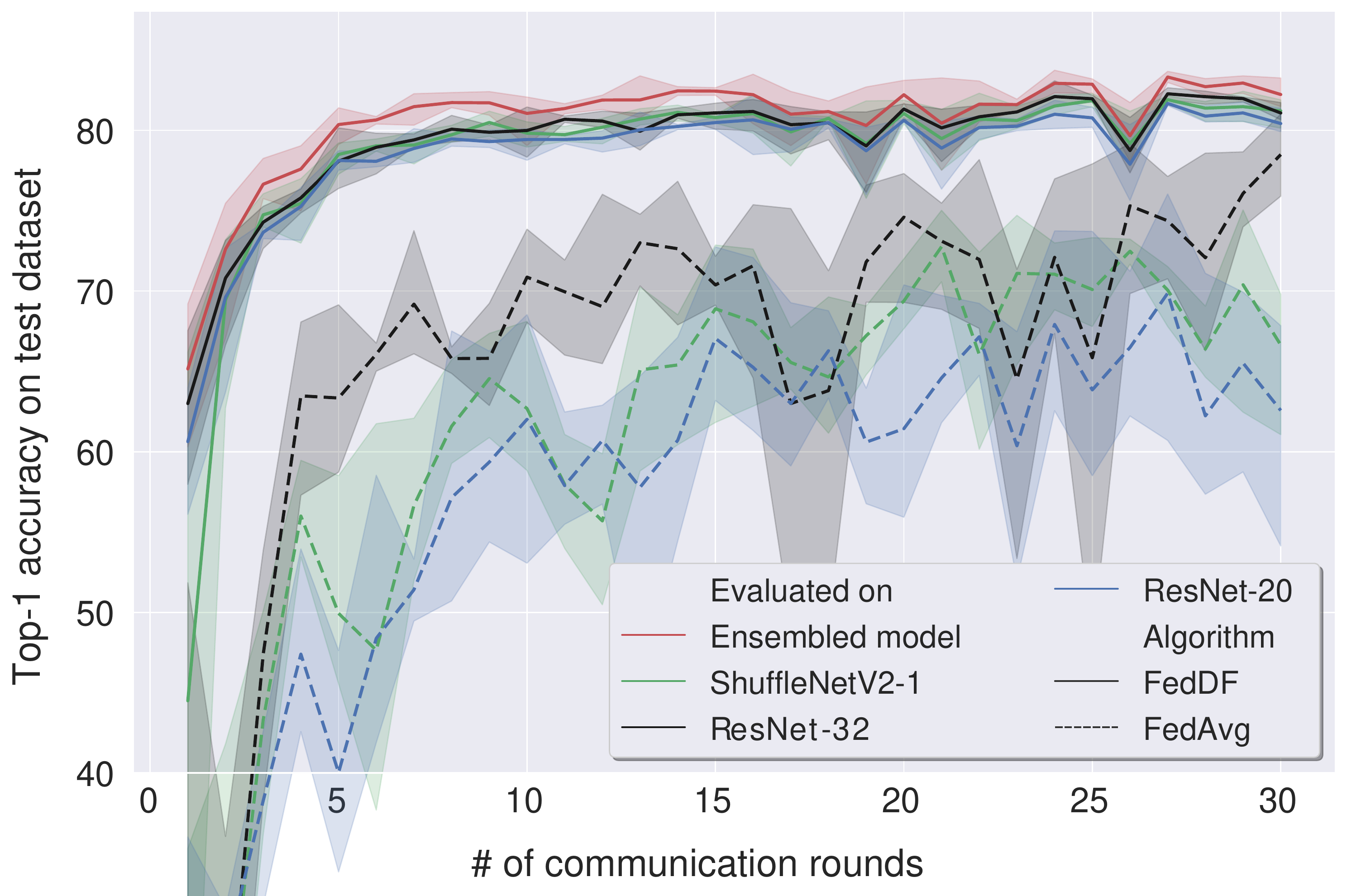}
		\label{fig:cifar10_heterogeneous_via_cifar100_feddf_vs_fedavg}
	}
	\hfill
	\subfigure[\small
		CIFAR-100.
	]{
		\includegraphics[width=0.31\textwidth,]{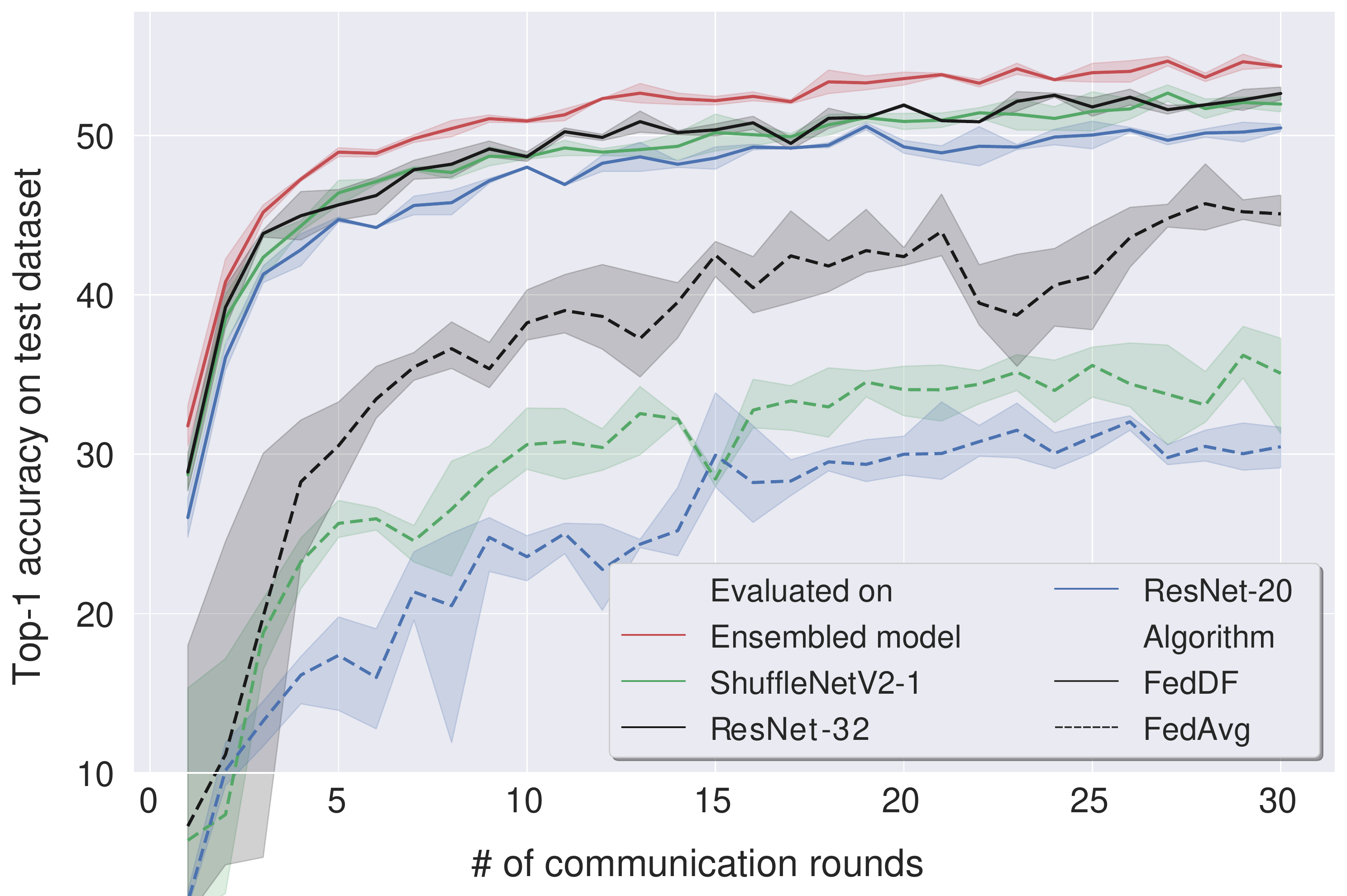}
		\label{fig:cifar100_heterogeneous_via_stl10_feddf_vs_fedavg}
	}
	\hfill
	\subfigure[\small
		ImageNet (image resolution $32$).
	]{
		\includegraphics[width=0.31\textwidth,]{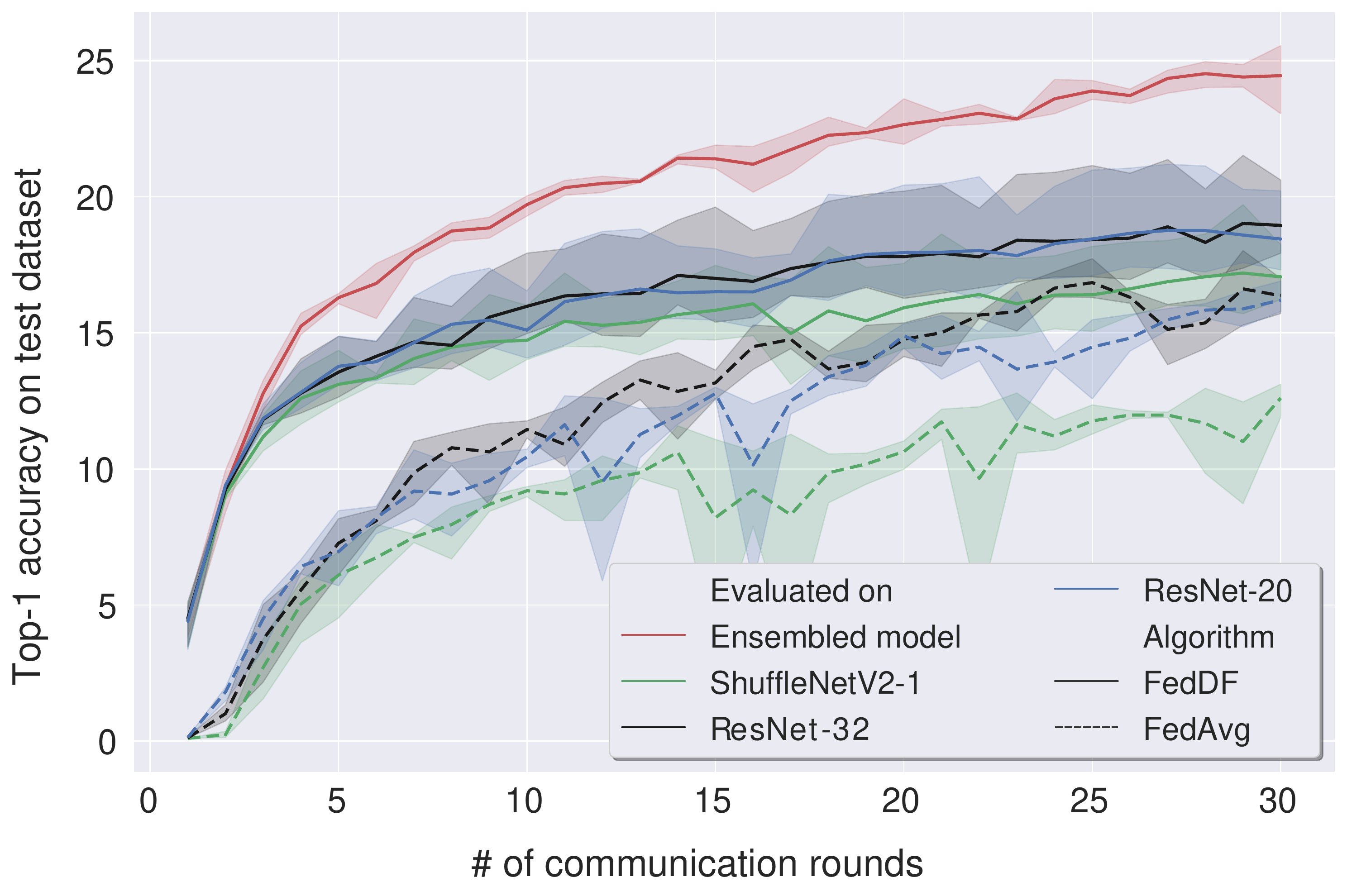}
		\label{fig:imagenet32_heterogeneous_via_cifar100_feddf_vs_fedavg}
	}
	\vspace{-1em}
	\caption{\small
		\textbf{\Fl on heterogeneous systems (model/data)},
		with three neural architectures (ResNet-20, ResNet-32, ShuffleNetV2)
		and non-\iid local data distribution ($\alpha \!=\! 1$).
		We consider $21$ clients for CIFAR (client sampling ratio $\clientfrac \!=\! 0.4$) and $150$ clients for ImageNet ($\clientfrac \!=\! 0.1$);
		different neural architectures are evenly distributed among clients.
		We train $80$ local training epochs per communication round (total $30$ rounds).
		CIFAR-100, STL-10, and STL-10 are used as the distillation datasets for CIFAR-10/100 and ImageNet training respectively.
		The \emph{solid} lines show the results of \algopt for a given communication round,
		while \emph{dashed} lines correspond to that of \fedavg;
		\emph{colors} indicate model architectures.
	}
	\vspace{-1em}
	\label{fig:heterogeneous_system}
\end{figure*}

\section{Understanding \algopt}\label{subsec:understanding}%
\algopt consists of two chief components: ensembling and knowledge distillation via out-of-domain data.
In this section, we first investigate
what affects the ensemble performance on the global distribution (test domain) through a generalization bound.
We then provide empirical understanding of how different attributes of the out-of-domain distillation dataset
affect the student performance on the global distribution.
\begin{figure*}[!ht]
	\centering
	\subfigure[\small
		CIFAR-10.
	]{
		\includegraphics[width=0.31\textwidth,]{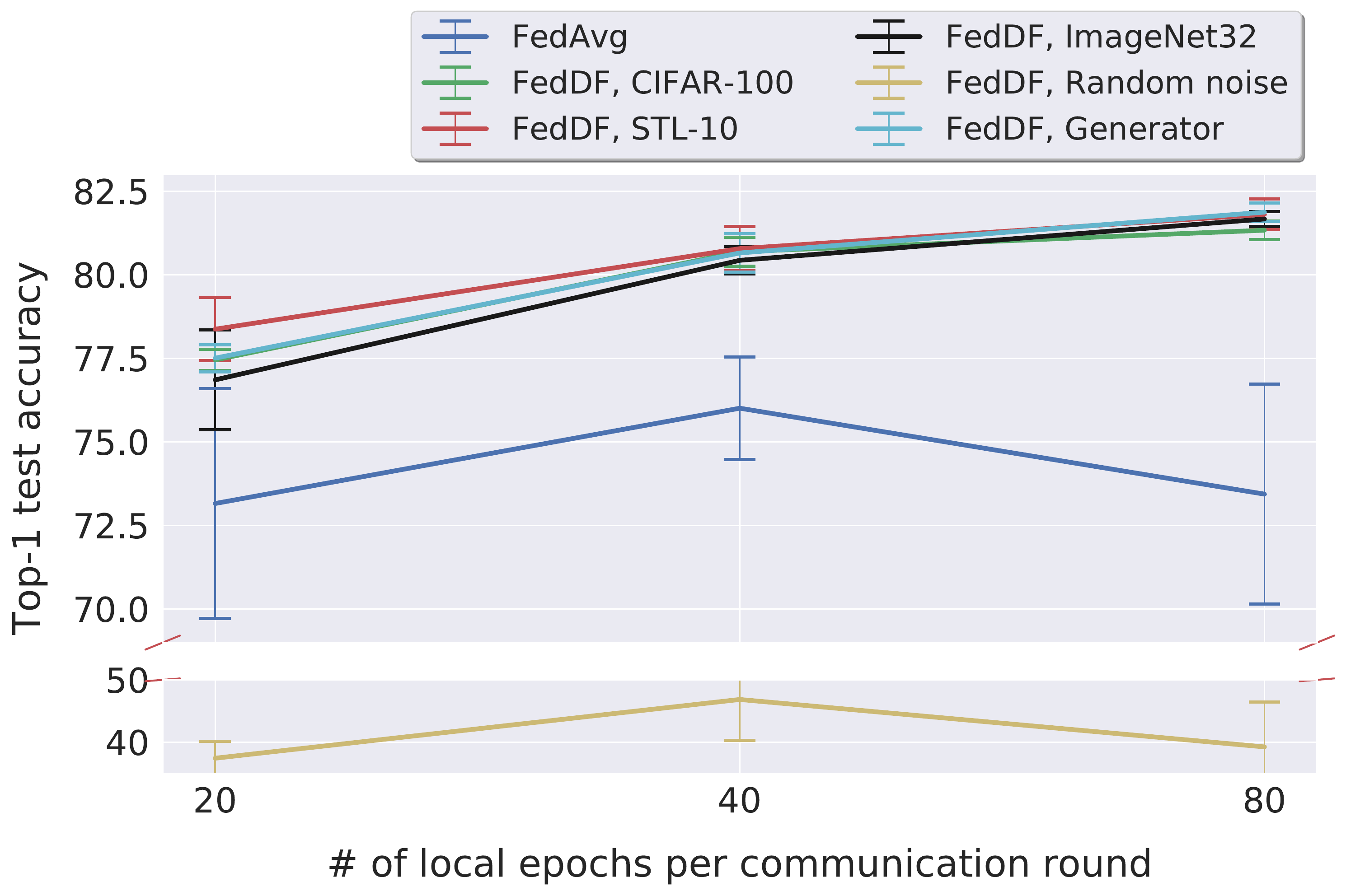}
		\label{fig:resnet8_cifar10_non_iid_1_different_unlabeled_data}
	}
	\hfill
	\subfigure[\small
		CIFAR-10 ($40$ local epochs).
	]{
		\includegraphics[width=0.31\textwidth,]{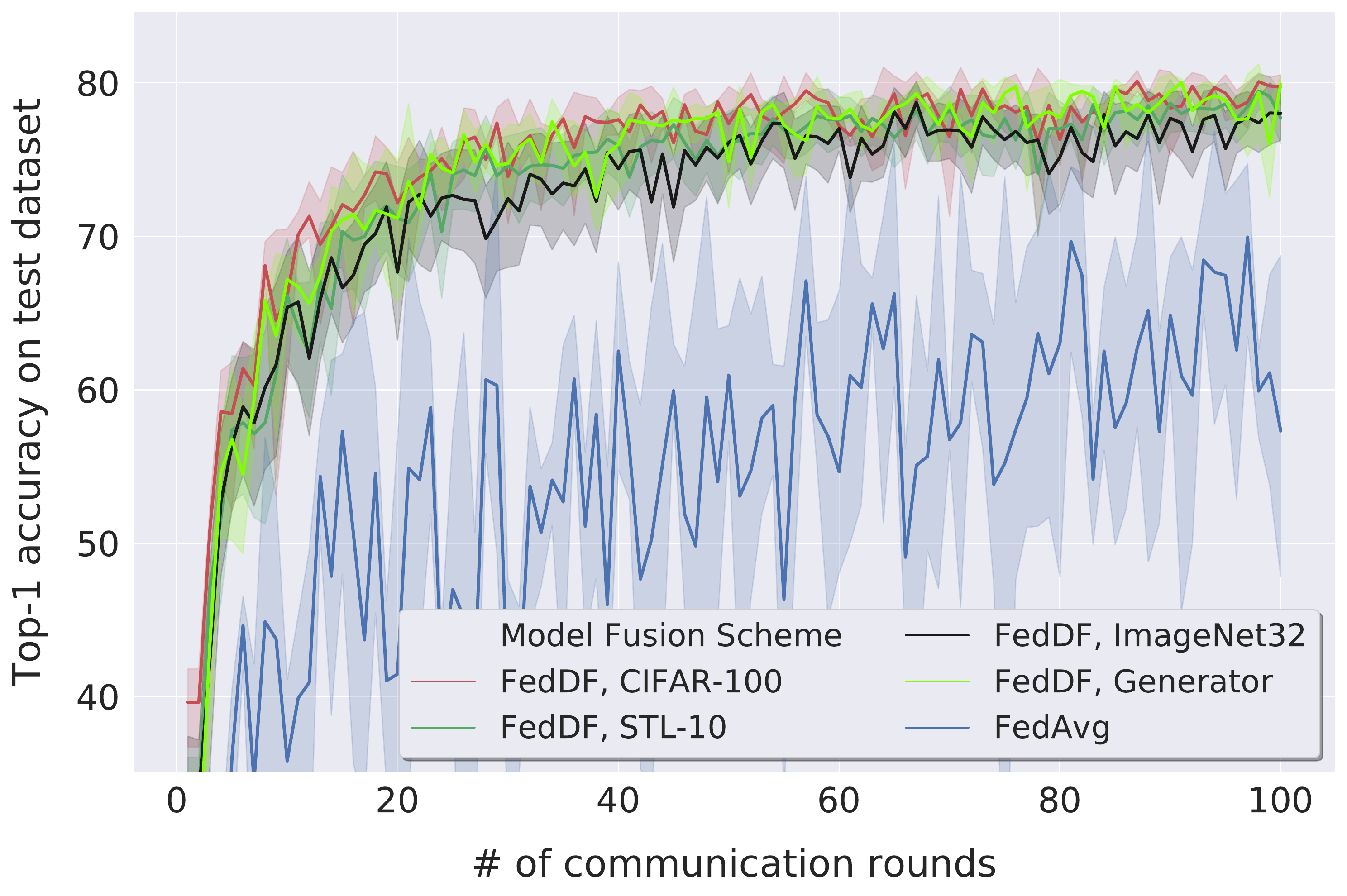}
		\label{fig:resnet8_cifar10_localdata_non_iid_1_different_datasets_40epochs.pdf}
	}
	\hfill
	\subfigure[\small
		CIFAR-100 ($40$ local epochs).
	]{
		\includegraphics[width=0.31\textwidth,]{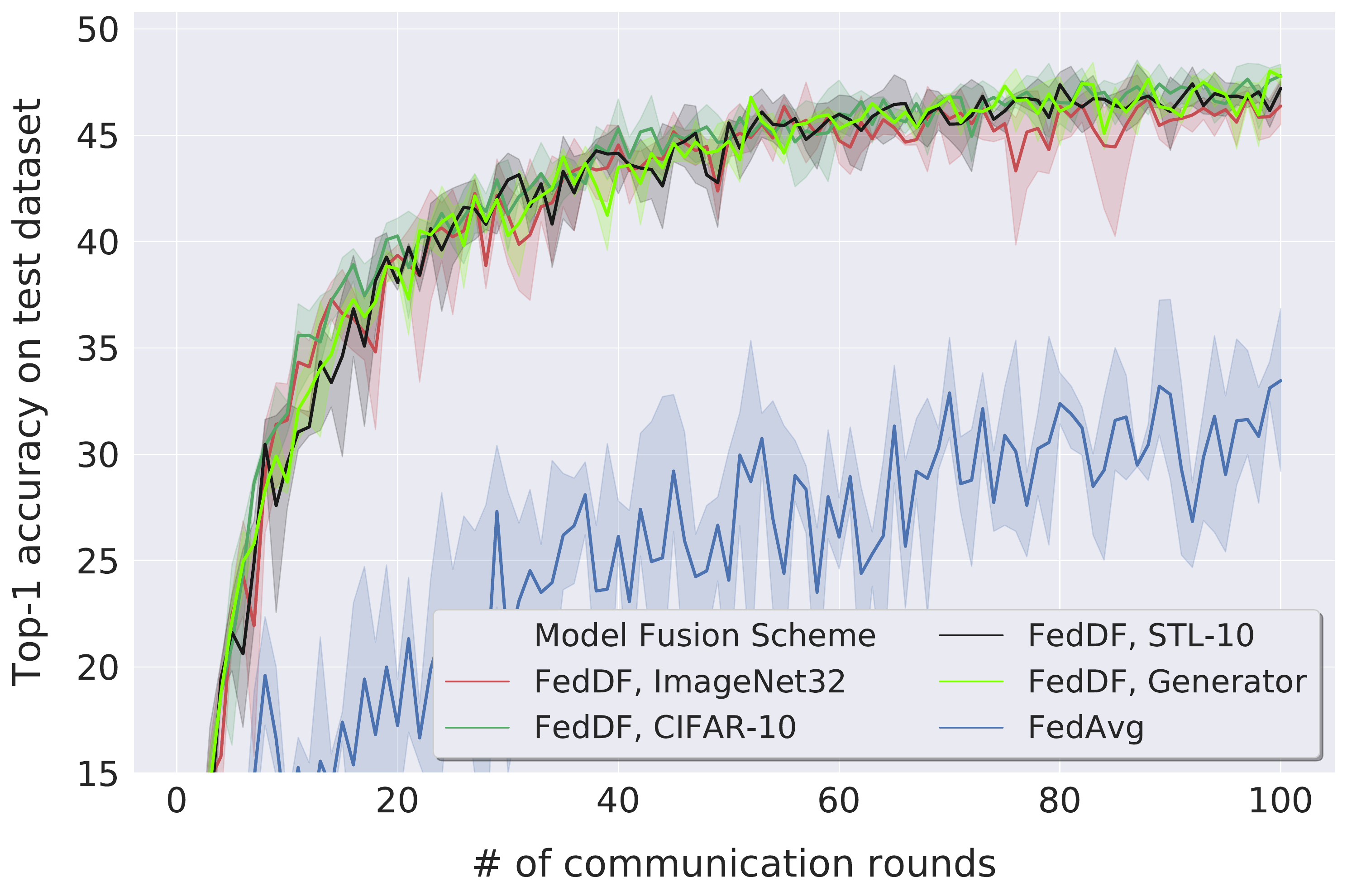}
		\label{fig:resnet8_cifar100_localdata_non_iid_1_different_datasets_40epochs}
	}
	\vspace{-1em}
	\caption{\small
		\textbf{The performance of \algopt on different distillation datasets}:
		random uniformly sampled noises, randomly generated images (from the generator),
		CIFAR, downsampled ImageNet32, and downsampled STL-10.
		We evaluate ResNet-8 on CIFAR for $20$ clients,
		with $\clientfrac \!=\! 0.4$, $\alpha \!=\! 1$ and $100$ communication rounds.
	}
	\vspace{-1em}
	\label{fig:understanding_input_domain_drift1}
\end{figure*}

\begin{figure*}[!t]
	\centering
	\subfigure[\small
		The fusion performance of \algopt through unlabeled ImageNet,
		for different numbers of classes.
	]{
		\includegraphics[width=0.31\textwidth,]{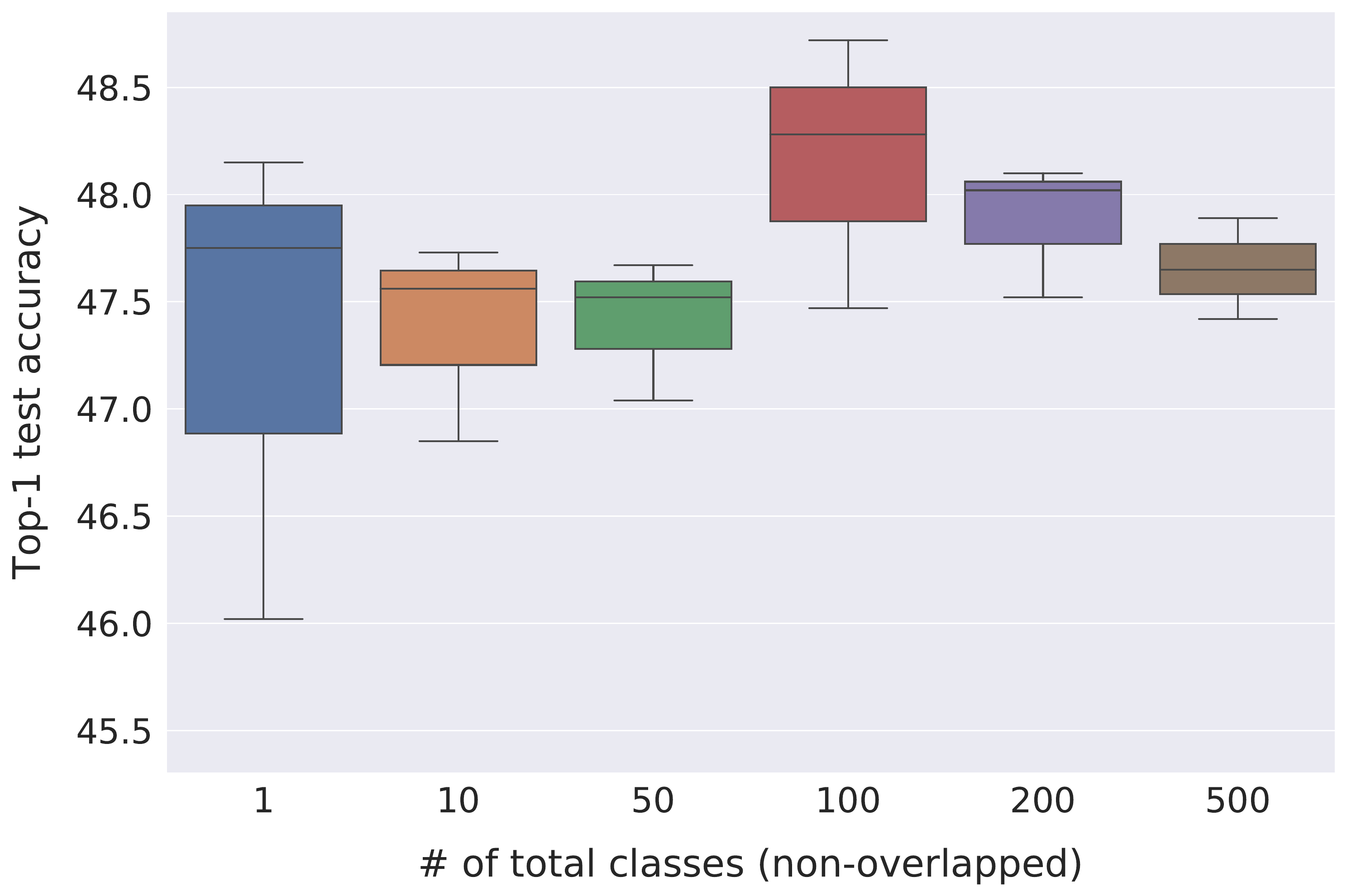}
		\label{fig:resnet8_cifar100_impact_of_num_classes}
	}
	\hfill
	\subfigure[\small
		The performance of \algopt via unlabeled ImageNet
		($100$ classes),
		for different data fractions.
	]{
		\includegraphics[width=0.31\textwidth,]{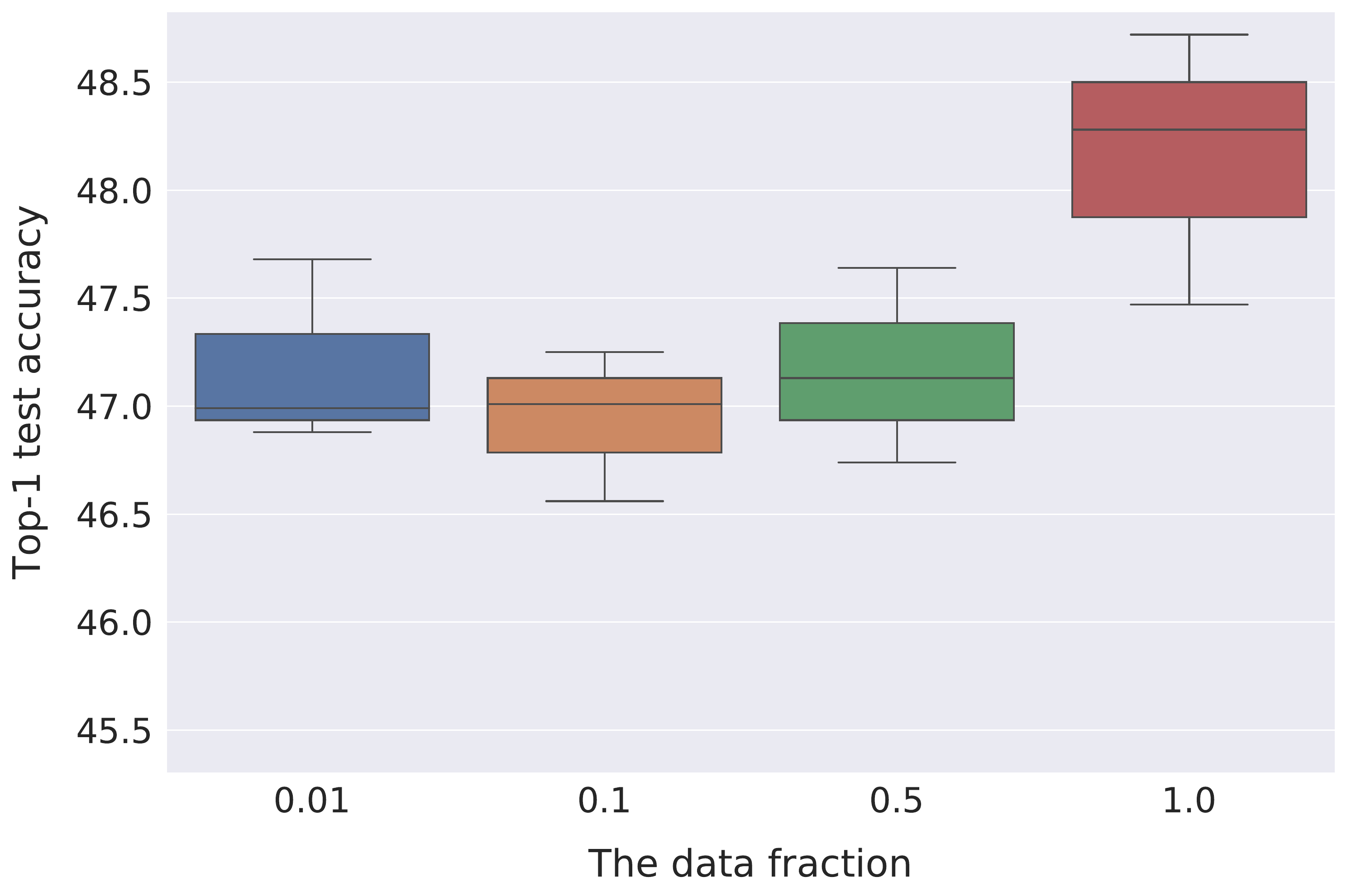}
		\label{fig:resnet8_cifar100_impact_of_data_fraction_for_100}
	}
	\hfill
	\subfigure[\small
		The fusion performance of \algopt under different numbers of distillation steps.
	]{
		\includegraphics[width=0.31\textwidth,]{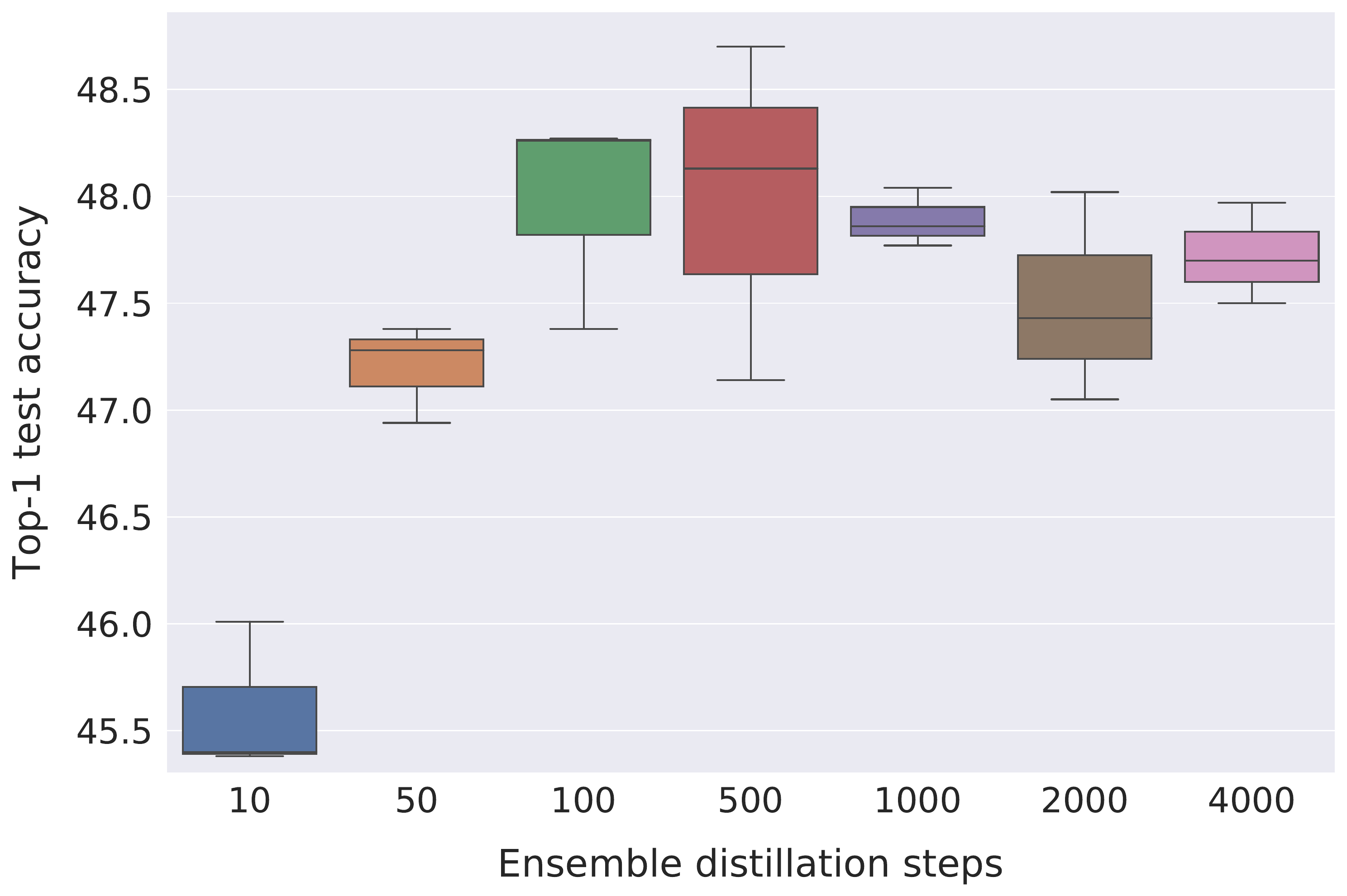}
		\label{fig:resnet8_cifar100_impact_of_distillation_steps}
	}
	\vspace{-1em}
	\caption{\small
		\textbf{Understanding knowledge distillation behaviors of \algopt} on
		\textbf{\# of classes} (\ref{fig:resnet8_cifar100_impact_of_num_classes}),
		\textbf{sizes of the distillation dataset} (\ref{fig:resnet8_cifar100_impact_of_data_fraction_for_100}),
		and \textbf{\# of distillation steps} (\ref{fig:resnet8_cifar100_impact_of_distillation_steps}),
		for \fl ResNet-8 on CIFAR-100,
		with $\clientfrac \!=\! 0.4$, $\alpha \!=\! 1$ and $100$ communication rounds ($40$ local epochs per round).
		ImageNet with image resolution $32$ is considered as our base unlabeled dataset.
		For simplicity, only classes without overlap with CIFAR-100 classes are considered,
		in terms of the synonyms, hyponyms, or hypernyms of the class name.
	}
	\vspace{-1em}
	\label{fig:understanding_input_domain_drift2}
\end{figure*}

\paragraph{Generalization bound.}
Theorem~\ref{thm:informal_risk_upper_bound_for_ensemble_model}
provides insights into ensemble performance on the global distribution.
Detailed description and derivations are deferred to Appendix~\ref{sec:detailed_bounds}.
\begin{theorem}[informal] \label{thm:informal_risk_upper_bound_for_ensemble_model}
	We denote the global distribution as $\cD$, the $k$-th local distribution
	and its empirical distribution as $\cD_k$ and $\hat{\cD}_k$ respectively.
	The hypothesis $h \in \cH$ learned on $\hat{\cD}_k$ is denoted by $\smash{h_{\hat{\cD}_k}}$.
	The upper bound on the risk of the ensemble of $K$ local models on $\cD$
	mainly consists of 1)
	the empirical risk of a model trained on the global empirical distribution $\hat{\cD} = \frac{1}{K} \sum_k \hat{\cD}_k$,
	and 2) terms dependent on the distribution discrepancy between $\cD_k$ and $\cD$,
	with the probability $1 - \delta$:
	\begin{small}
		\begin{align*}
			\textstyle
			L_\cD \Big( \frac{1}{K} \sum_k h_{\hat{\cD}_k} \Big)
			  & \leq
			L_{\hat{\cD}} ( h_{\hat{\cD}} )
			+ \frac{1}{K} \sum_{k} \left( \frac{1}{2} d_{\cH \Delta \cH} (\cD_k, \cD) + \lambda_k \right)
			+ \frac{ 4 + \sqrt{ \log ( \tau_{\cH} (2m) ) } }{ \delta / K \sqrt{ 2m } }
			\,,
		\end{align*}
	\end{small}%
	where $d_{\cH \Delta \cH}$ measures the distribution discrepancy between two distributions~\cite{ben2010theory},
	$m$ is the number of samples per local distribution,
	$\lambda_k$ is the minimum of the combined loss $\cL_{\cD} (h) \!+\! \cL_{\cD_k}(h), \forall h \in \cH$,
	and $\tau_{\cH}$ is the growth function bounded by a polynomial of the VCdim of $\cH$.
\end{theorem}

The ensemble of the local models sets the performance upper bound for the later distilled model
on the global distribution
as shown in Figure~\ref{fig:heterogeneous_system}.
Theorem~\ref{thm:informal_risk_upper_bound_for_ensemble_model} shows that
compared to a model trained on
the global empirical distribution
(ideal centralized case),
the performance of the ensemble on the global distribution
is associated with the discrepancy between local distributions $\cD_k$'s and the global distribution $\cD$.
Besides, the shift between the distillation and the global distribution
determines the knowledge transfer quality between these two distributions
and hence the test performance of the fused model.
In the following, we empirically examine how the choice of distillation data distributions
and the number of distillation steps
influence the quality of ensemble knowledge distillation.

\paragraph{Source, diversity and size of the distillation dataset.}
The fusion in \algopt demonstrates remarkable consistency across a wide range of realistic data sources
as shown in Figure~\ref{fig:understanding_input_domain_drift1},
although an abrupt performance declination is encountered when the distillation data are sampled from a dramatically different manifold
(e.g.\ random noise).
Notably, synthetic data from the generator of a pre-trained GAN does not incur noticeable quality loss,
opening up numerous possibilities for effective and efficient model fusion.
Figure~\ref{fig:resnet8_cifar100_impact_of_num_classes}
illustrates that
in general the diversity of the distillation data
does not significantly impact the performance of ensemble distillation,
though the optimal performance is achieved when two domains have a similar number of classes.
Figure~\ref{fig:resnet8_cifar100_impact_of_data_fraction_for_100}
shows the \algopt is not demanding on the distillation dataset size:
even $1\%$ of data ($\sim 48\%$ of the local training dataset)
can result in a reasonably good fusion performance.

\paragraph{Distillation steps.}
Figure~\ref{fig:resnet8_cifar100_impact_of_distillation_steps}
depicts the impact of distillation steps on fusion performance,
where \algopt with a moderate number of the distillation steps
is able to approach the optimal performance.
For example,
$100$ distillation steps in Figure~\ref{fig:resnet8_cifar100_impact_of_distillation_steps},
which corresponds to $5$ local epochs of CIFAR-100 (partitioned by 20 clients),
suffice to yield satisfactory performance.
Thus \algopt introduces minor time-wise expense.

\section*{Broader Impact} %
We believe that collaborative learning schemes such as federated learning are an important element towards enabling privacy-preserving training of ML models, as well as a better alignment of each individual's data ownership with the resulting utility from jointly trained machine learning models, especially in applications where data is user-provided and privacy sensitive~\cite{kairouz2019advances,nedic2020review}.

In addition to privacy, efficiency gains and lower resource requirements in distributed training reduce the environmental impact of training large machine learning models.
The introduction of a practical and reliable distillation technique for heterogeneous models and for low-resource clients is a step towards more broadly enabling collaborative privacy-preserving and efficient decentralized learning.

\section*{Acknowledgements}
We acknowledge funding from SNSF grant 200021\_175796, as well as a Google Focused Research Award.

\bibliography{paper}
\bibliographystyle{abbrv}

\clearpage
\appendix

\section{Detailed Related Work Discussion} \label{sec:detailed_related_work}
\paragraph{Prior work.}
We first comment on the two close approaches (FedMD and Cronus), in order to address
1) Distinctions between FedDF and prior work,
2) Privacy/Communication traffic concerns,
3) Omitted experiments on FedMD and Cronus.
\begin{itemize}[nosep,leftmargin=12pt]
	\item Distinctions between FedDF and prior work.
	      As discussed in the related work, most SOTA FL methods directly manipulate received model parameters
	      (e.g.\ FedAvg/FedAvgM/FedMA).
	      To our best knowledge, FedMD and Cronus are the only two that utilize logits information (of neural nets) for FL.
	      The distinctions from them are made below.
	\item {Different objectives and evaluation metrics}.
	      Cronus is designed for robust FL under poisoning attack,
	      whereas FedMD is for personalized FL.
	      In contrast, FedDF is intended for on-server model aggregation (evaluation on the aggregated model),
	      whereas neither FedMD nor Cronus aggregates the model on the server.
	\item Different Operations.
	      \begin{enumerate}[nosep,leftmargin=12pt]
		      \item FedDF, like FedAvg, \emph{only} exchanges models
		            between the server and clients,
		            without transmitting input data.
		            In contrast, FedMD and Cornus rely on exchanging public data logits.
		            As FedAvg,
		            FedDF can include privacy/security extensions
		            and has the same communication cost per round.
		      \item FedDF performs ensemble distillation with unlabeled data \emph{on the server}.
		            In contrast, FedMD/Cronus use averaged logits received from the server for \emph{local client training}.
	      \end{enumerate}

	\item Omitted experiments with FedMD/Cronus.
	      \begin{enumerate}[nosep,leftmargin=12pt]
		      \item FedMD requires to {locally pre-train on the \emph{labeled} public data},
		            thus the model classifier necessitates an output dimension of \# of public classes \emph{plus} \# of private classes
		            (c.f. the output dimension of \# of private classes in other FL methods).
		            We cannot compare FedMD with FedDF with the same architecture (classifier) to ensure fairness.
		      \item Cronus is shown to be consistently worse than FedAvg in normal FL (i.e.\ no attack case) in their Tab.\ IV \& VI.
		      \item Different objectives/metrics argued above.
		            We thoroughly evaluated SOTA baselines with the same objective/metric.
	      \end{enumerate}
\end{itemize}

\paragraph{Contemporaneous work.}
We then detail some contemporaneous work,
e.g.~\cite{sun2020federated, chen2020feddistill, zhou2020distilled,he2020group}.
\cite{sun2020federated} slightly extends FedMD by adding differential privacy.
In~\cite{zhou2020distilled},
the server aggregates the synthetic data distilled from clients' private dataset,
which in turn uses for one-shot on-server learning.
He \textit{et al}~\cite{he2020group} improve FL for resource-constrained edge devices
by combing FL with Split Learning (SL) and knowledge distillation:
edge devices train compact feature extractor through local SGD
and then synchronize extracted features and logits with the server,
while the server (asynchronously) uses the latest received features and logits to train a much larger server-side CNN.
The knowledge distillation is used on both the server and clients to improve the optimization quality.

FedDistill~\cite{chen2020feddistill} is very similar to us,
where it resorts to stochastic weight average-Gaussian (SWAG)~\cite{maddox2019simple}
and the ensemble distillation is achieved via cyclical learning rate schedule with SWA~\cite{izmailov2018averaging}.
In Table~\ref{tab:resnet_8_cifar10_effect_of_different_optimizers_for_distillation} below,
we empirically compare our \algopt with this contemporaneous work (i.e. FedDistill).

\section{Algorithmic Description} \label{appendix:algo_description}
Algorithm~\ref{alg:client_update} below details a general training procedure on local clients.
The local update step of \fedprox corresponds to adding a proximal term
(i.e.\ $\eta \derive{ \frac{\mu}{2} \norm{\xx_t^k - \xx_{t-1}^k }_2^2}{\xx_t^k}$) to line 5.

\begin{algorithm}[!h]
	\begin{algorithmic}[1]
		\Procedure{Client-LocalUpdate}{$k, \xx_{t-1}^k$}
		\myState{Client $k$ receives $\xx_{t-1}^k$ from server and copies it as $\xx_t^k$ }
		\For{ each local epoch $i$ from $1$ to $\lepochs$ }
		\For{mini-batch $b \subset \cP_k$}
		\myState{$\xx_t^k \leftarrow \xx_t^k - \eta \derive{\loss(\xx_t^k; b)}{\xx_t^k}$ } \Comment{can be arbitary optimizers (e.g. Adam)}
		\EndFor
		\EndFor
		\myState{\Return $\xx_t^k$ to server }
		\EndProcedure
	\end{algorithmic}
	\mycaptionof{algorithm}{
		\textbf{Illustration of local client update in \fedavg.}
		The $K$ clients are indexed by $k$; $\cP_k$ indicates the set of indexes of data points on client $k$,
		and $n_k = \abs{\cP_k}$.
		$\lepochs$ is the number of local epochs, and $\eta$ is the learning rate.
		$\loss$ evaluates the loss on model weights for a mini-batch of an arbitrary size.
	}
	\label{alg:client_update}
\end{algorithm}

Algorithm~\ref{alg:heterogeneous_framework} illustrates the model fusion of \algopt
for the FL system with heterogeneous model prototypes. The schematic diagram is presented in Figure~\ref{fig:diagram_heterogeneous_system}.
To perform model fusion in such heterogeneous scenarios,
\algopt constructs several prototypical models on the server.
Each prototype represents all clients with identical architecture/size/precision etc.

\begin{algorithm}[!h]
	\begin{algorithmic}[1]
		\Procedure{Server}{}
		\myState{initialize HashMap $\cM$: map each model prototype $P$ to its weights $\xx^P_{0}$.}
		\myState{initialize HashMap $\cC$: map each client to its model prototype.}
		\myState{initialize HashMap $\tilde{\cC}$: map each model prototype to the associated clients.}

		\For{each communication round $t = 1, \dots, T$}
		\myState{$\cS_t \leftarrow$ a random subset ($C$ fraction) of the $K$ clients}
		\For{ each client $k \in \cS_t$ \textbf{in parallel} }
		\myState{ $\hxx_{t}^k \leftarrow \text{Client-LocalUpdate}(k, \cM\left[ \cC[k] \right] )$}  \Comment{detailed in Algorithm~\ref{alg:client_update}.}
		\EndFor

		\For{ each prototype $P \in \cP$ \textbf{in parallel} }
		\myState{
			initialize the client set $\cS^P_t$ with model prototype $P$,
			where $\cS^P_t \leftarrow \tilde{\cC}[P] \cap \cS_t$
		}
		\myState{ initialize for model fusion
		$\xx^P_{t, 0} \leftarrow \sum_{ k \in \cS^P_t } \frac{n_k}{ \sum_{ k \in \cS^P_t } n_k } \hxx_{t}^k$
		}
		\For{ $j$ in $\{ 1, \dots, N \}$ }
		\myState{ sample $\dd$, from e.g.\ (1) an unlabeled dataset, (2) a generator}
		\myState{ use ensemble of $ \{ \hxx_{t}^k \}_{k \in \cS_t}$ to update server student $\xx^P_{t, j}$ through \avglogits}
		\EndFor
		\myState{ $\cM\left[ P \right] \leftarrow \xx^P_{t, N}$}
		\EndFor

		\EndFor
		\myState{\Return $\cM$}
		\EndProcedure
	\end{algorithmic}

	\mycaptionof{algorithm}{\small
	Illustration of \algopt for heterogeneous FL systems.
	The $K$ clients are indexed by $k$,
	and $n_k$ indicates the number of data points for the $k$-th client.
	The number of communication rounds is $T$,
	and $C$ controls the client participation ratio per communication round.
	The number of total iterations used for model fusion is denoted as $N$.
	The distinct model prototype set $\cP$ has $p$ model prototypes,
	with each initialized as $\xx^P_{0}$.
	}
	\label{alg:heterogeneous_framework}
\end{algorithm}

\begin{figure*}[!h]
	\centering
	\includegraphics[width=0.77\textwidth,]{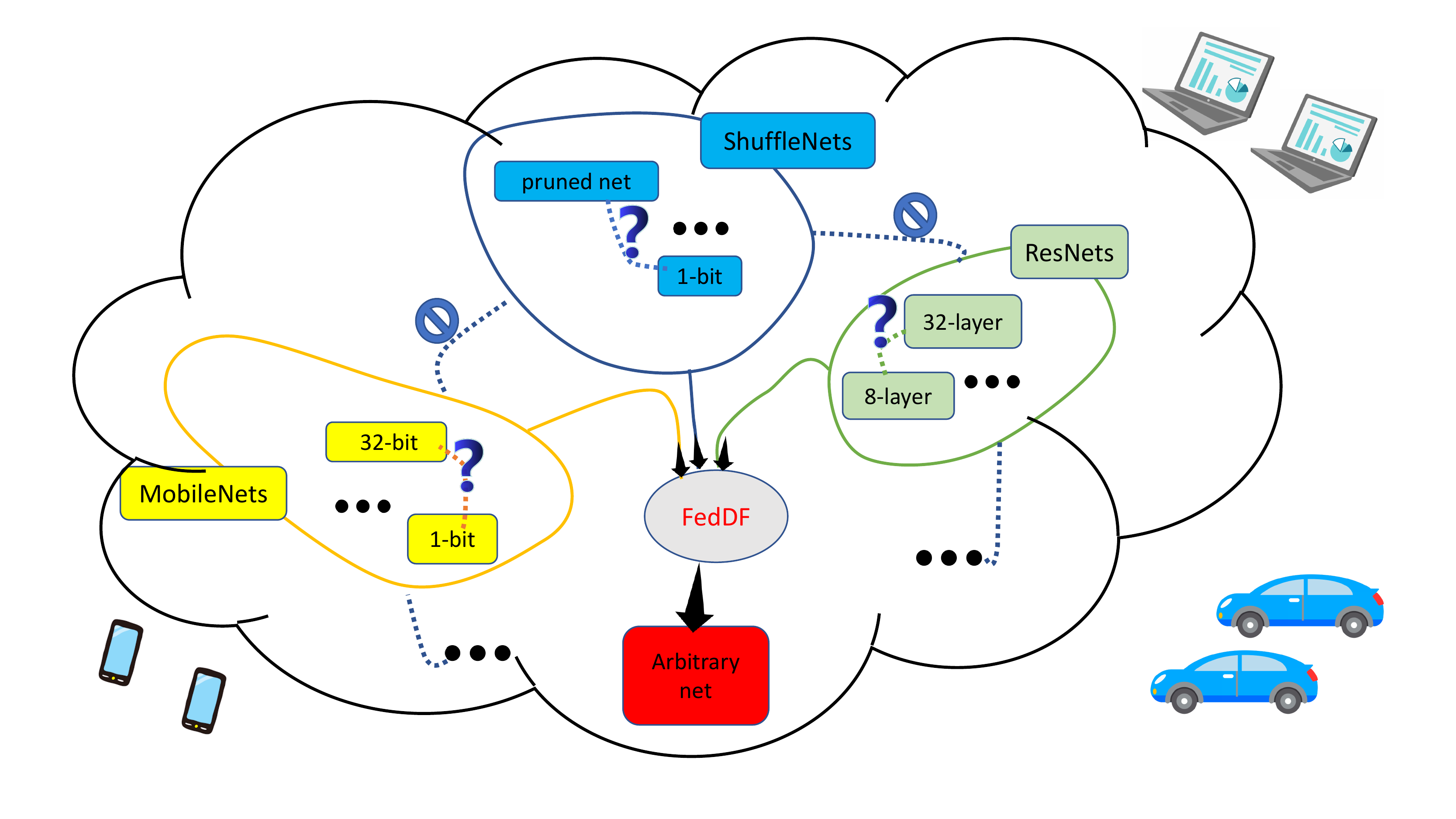}
	\caption{\small
		\textbf{The schematic diagram for heterogeneous model fusion.}
		We use dotted lines to indicate model parameter averaging FL methods such as \fedavg.
		We could notice the architectural/precision discrepancy invalidates these methods
		in heterogeneous FL systems.
		However, \algopt could aggregate knowledge from all available models without hindrance.
	}
	\label{fig:diagram_heterogeneous_system}
\end{figure*}

\clearpage
\section{Additional Experimental Setup and Evaluations} \label{appendix:more_exps}
\subsection{Detailed Description for Toy Example (Figure~\ref{fig:illustration_problems_in_fl})}
\label{subsec:description_toy_example}

Figure~\ref{fig:illustration_problems_in_fl_detail}
provides a detailed illustration of the limitation in \fedavg.

\begin{figure*}[!h]
	\centering
	\includegraphics[width=1\textwidth]{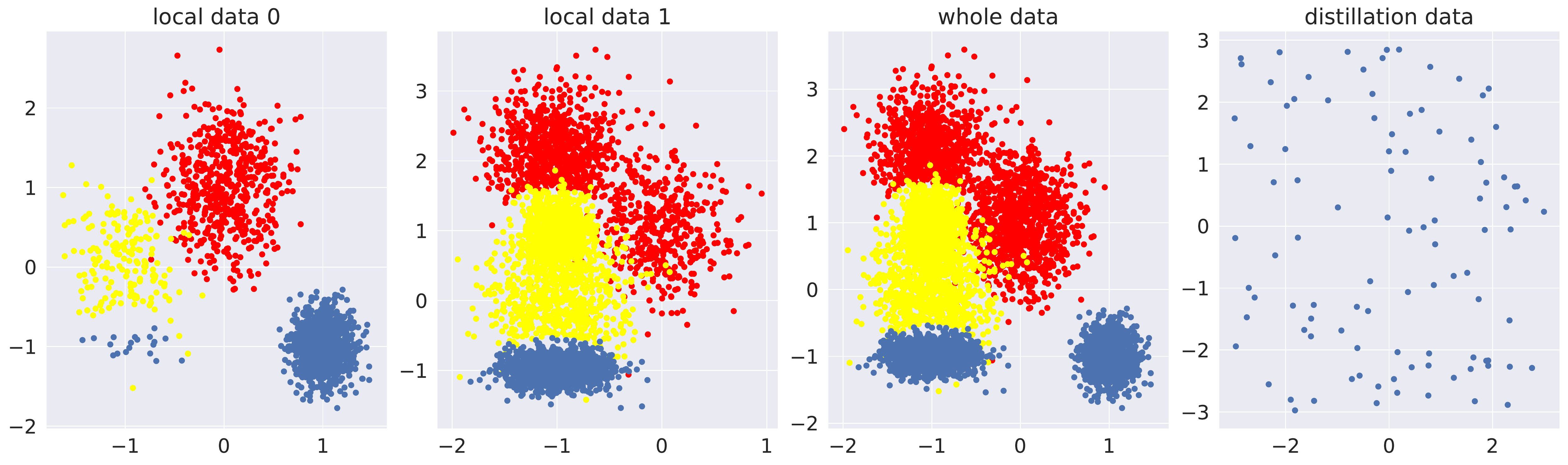}
	\hfill
	\includegraphics[width=1\textwidth]{figures/methods/the_issue_of_fedavg_perf.pdf}
	\caption{\small
		\textbf{The limitation of \fedavg.}
		We consider a toy example of a 3-class classification task with a 3-layer MLP,
		and display the decision boundaries (probabilities over RGB channels) on the input space.
		We illustrate the used datasets in the \textbf{top} row;
		the distillation dataset consists of $60$ data points,
		with each uniformly sampled from the range of $(-3, 3)$.
		In the \textbf{bottom} row,
		the left two figures consider the individually trained local models.
		The right three figures evaluate aggregated models
		and the global data distribution;
		the averaged model (\fedavg) results in much blurred decision boundaries.
	}
	\label{fig:illustration_problems_in_fl_detail}
\end{figure*}

\subsection{Detailed Experiment Setup} \label{appendix:detailed_exp_setup}
\paragraph{The detailed hyperparameter tuning procedure.}
The tuning procedure of hyperparameters ensures that the best hyperparameter lies in the middle of our search grids;
otherwise, we extend our search grid.
The initial search grid of learning rate is $\{1.5, 1, 0.5, 0.1, 0.05, 0.01 \}$.
The initial search grid of proximal factor in~\fedprox is $\{ 0.001, 0.01, 0.1, 1 \}$.
The initial search grid of momentum factor $\beta$ in \fedavgM is $\{ 0.1, 0.2, 0.3, 0.4 \}$;
the update scheme of \fedavgM follows $\Delta \vv := \beta \vv + \Delta \xx \;; \xx := \xx - \Delta \vv$,
where $\Delta \xx$ is the model difference between the updated local model and the sent global model,
for previous communication round.

Unless mentioned (i.e.\ Table~\ref{tab:cifar10_resnet8_detailed_comparison}),
otherwise the learning rate is set to $0.1$
for ResNet like architectures (e.g.\ ResNet-8, ResNet-20, ResNet-32, ShuffleNetV2),
$0.05$ for VGG and $1e{-}5$ for DistilBERT.
When comparing with other methods, e.g.\ \fedprox, \fedavgM,
we always tune their corresponding hyperparameters
(e.g.\ proximal factor in~\fedprox and momentum factor in~\fedavgM).

\paragraph{Experiment details of \fedma.}
We detail our attempts of reproducing \fedma experiments on VGG-9 with CIFAR-10 in this section.
We clone their codebase from GitHub and add functionality to sample clients after synchronizing the whole model.

Different from other methods evaluated in the paper,
\fedma uses a layer-wise local training scheme.
For each round of the local training,
the involved clients only update the model parameters from one specific layer onwards,
while the already matched layers are frozen.
The fusion (matching) is only performed on the chosen layer.
Such a layer
is gradually chosen from the bottom layer to the top layer,
following a bottom-up fashion~\cite{Wang2020Federated}.
One complete model update cycle of~\fedma requires
more frequent (but slightly cheaper) communication,
which is equivalent to the number of layers in the neural network.

In our experiments of \fedma, the number of local training epochs is $5$ epochs per layer
($45$ epochs per model update),
which is slightly larger than $40$ epochs used by other methods.
We ensure a similar\footnote{
	The other methods use $40$ local training epochs per whole model update.
	Given the fact of layer-wise training scheme in~\fedma,
	as well as the used $9$-layer VGG
	(same as the one used in~\cite{Wang2020Federated}
	and we are unable to adapt their code to other architectures
	due to their hard-coded architecture manipulations),
	we decide to slightly increase the number of local epochs per layer for \fedma.
} number of model updates in terms of the whole model.
We consider global-wise learning rate, different from the layer-wise one in Wang et al.~\cite{Wang2020Federated}.
We also turn off the momentum and weight decay during the local training for a consistent evaluation.
The implementation of VGG-9 follows~\url{https://github.com/kuangliu/pytorch-cifar/}.

\paragraph{The detailed experimental setup for \algopt (low-bit quantized models).}
\algopt increases the feasibility of robust model fusion in FL for binarized ResNet-8.
As stated in Table~\ref{tab:binary_net} (Section~\ref{subsec:case_study}),
we employ the ``Straight-through estimator''~\cite{bengio2013estimating,hintonestimator,hubara2016binarized,hubara2017quantized}
or the ``error-feedback''~\cite{Lin2020Dynamic}
to simulate the on-device local training of the binarized ResNet-8.
For each communication round,
the server of the FL system will receive locally trained and binarized ResNet-8 from activated clients.
The server will then distill the knowledge of these low-precision models to a full-precision one\footnote{
	The training of the binarized network requires to maintain a full-precision model~\cite{hubara2016binarized,hubara2017quantized,Lin2020Dynamic} for model update
	(quantized/pruned model is used during the backward pass).
}
and broadcast to newly activated clients for the next communication round.
For the sake of simplicity,
the case study demonstrated in the paper only considers reducing the communication cost (from clients to the server),
and the local computational cost;
a thorough investigation on how to perform a communication-efficient and memory-efficient FL is left as future work.

\paragraph{The synthetic formulation of non-\iid client data.}
Assume every client training example is drawn
independently with class labels following a categorical distribution
over $M$ classes parameterized by a vector $\qq$ ($q_i \geq 0, i \in [1, M]$ and $\norm{\qq}_1 = 1$).
Following the partition scheme introduced and used in~\cite{yurochkin2019bayesian,hsu2019measuring}\footnote{
	We heavily borrowed the partition description of~\cite{hsu2019measuring} for the completeness of the paper.
},
to synthesize client non-\iid local data distributions,
we draw $\alpha \sim \text{Dir} (\alpha \pp)$ from a Dirichlet distribution,
where $\pp$ characterizes a prior class distribution over $M$ classes,
and $\alpha > 0$ is a concentration parameter controlling the identicalness among clients.
With $\alpha \rightarrow \infty$, all clients have identical distributions to the prior;
with $\alpha \rightarrow 0$, each client holds examples from only one random class.

To better understand the local data distribution for the datasets we considered in the experiments,
we visualize the partition results of CIFAR-10 and CIFAR-100 on $\alpha \!=\! \{ 0.01, 0.1, 0.5, 1, 100 \}$ for $20$ clients,
in Figure~\ref{fig:class_distribution_for_alphas_cifar10_20_clients}
and Figure~\ref{fig:class_distribution_for_alphas_cifar100_20_clients}, respectively.

In Figure~\ref{fig:nlp_partition}
we visualize the partitioned local data on $10$ clients with $\alpha \!=\! 1$,
for AG News and SST-2.

\begin{figure*}[!h]
	\centering
	\subfigure[\small
		$\alpha \!=\! 100$
	]{
		\includegraphics[width=0.31\textwidth,]{figures/exp/cifar10_n_clients_20_alpha_100.pdf}
		\label{fig:cifar10_20_clients_alpha_100}
	}
	\hfill
	\subfigure[\small
		$\alpha \!=\! 1$
	]{
		\includegraphics[width=0.31\textwidth,]{figures/exp/cifar10_n_clients_20_alpha_1.pdf}
		\label{fig:cifar10_20_clients_alpha_1}
	}
	\hfill
	\subfigure[\small
		$\alpha \!=\! 0.5$
	]{
		\includegraphics[width=0.31\textwidth,]{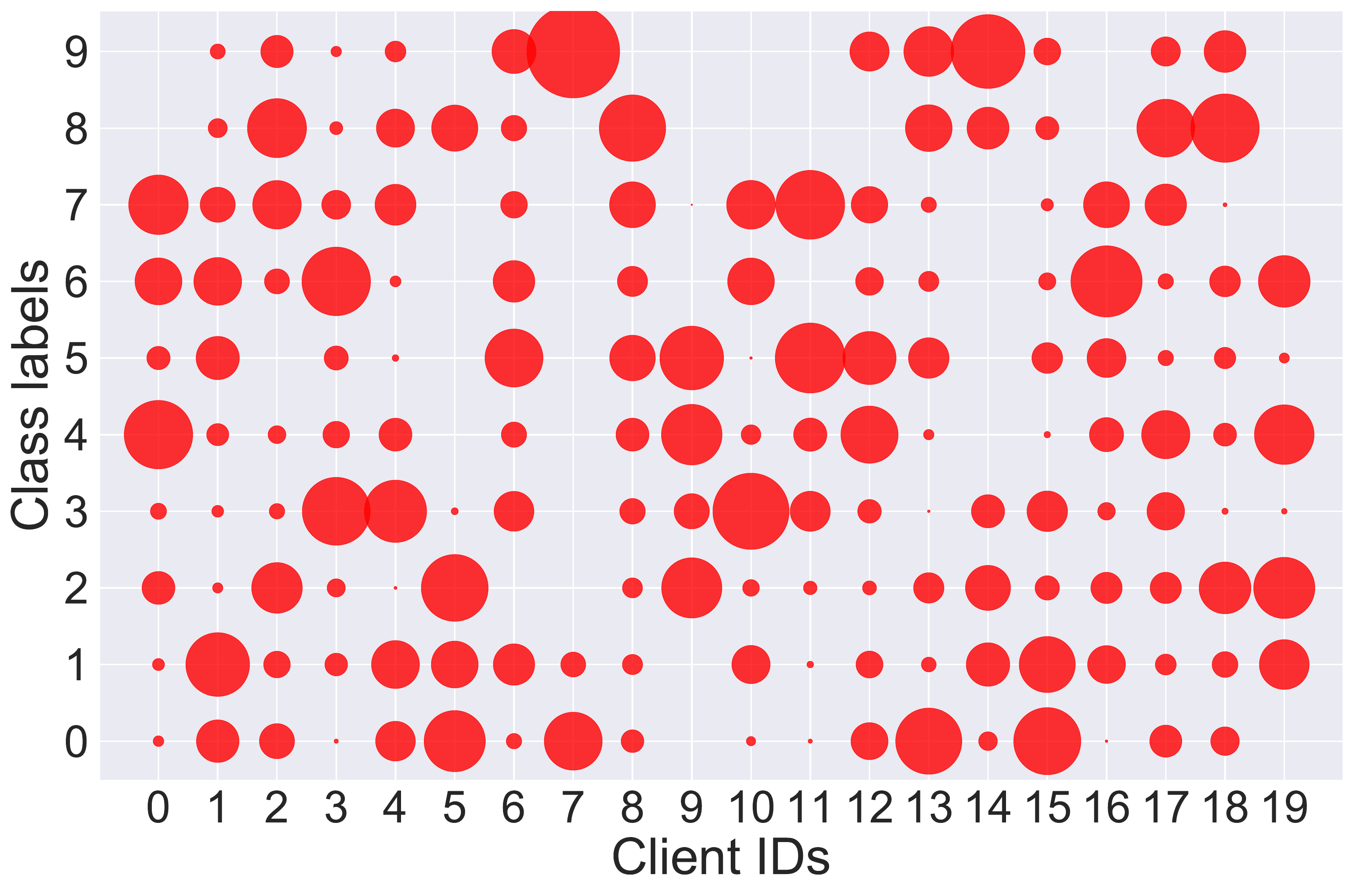}
		\label{fig:cifar10_20_clients_alpha_05}
	}
	\hfill
	\subfigure[\small
		$\alpha \!=\! 0.1$
	]{
		\includegraphics[width=0.31\textwidth,]{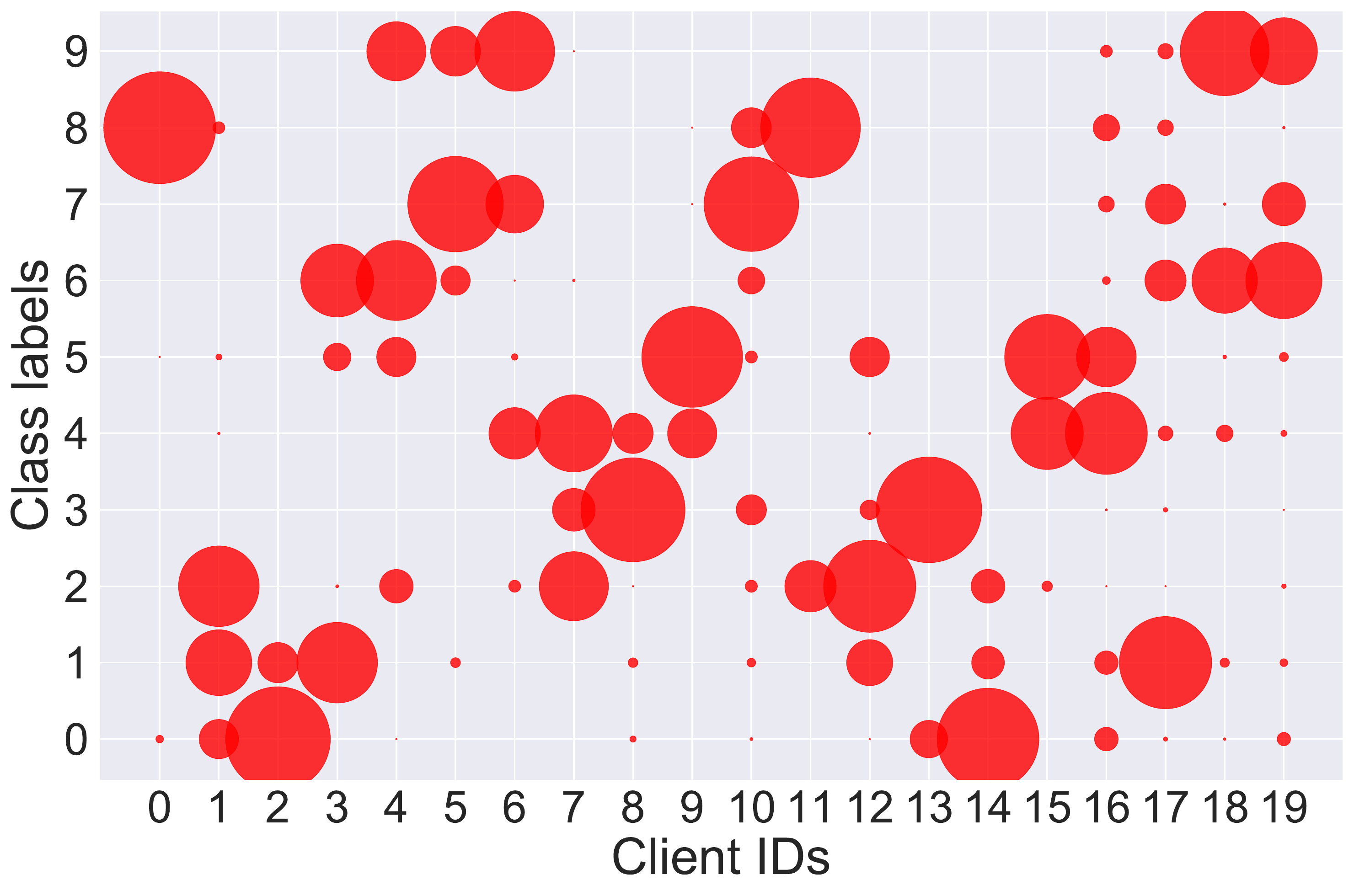}
		\label{fig:cifar10_20_clients_alpha_01}
	}
	\hfill
	\subfigure[\small
		$\alpha \!=\! 0.01$
	]{
		\includegraphics[width=0.31\textwidth,]{figures/exp/cifar10_n_clients_20_alpha_001.pdf}
		\label{fig:cifar10_20_clients_alpha_001}
	}
	\vspace{-0.5em}
	\caption{\small
		Classes allocated to each client at different Dirichlet distribution alpha values,
		for CIFAR-10 with 20 clients.
		The size of each dot reflects the magnitude of the samples number.
	}
	\vspace{-0.5em}
	\label{fig:class_distribution_for_alphas_cifar10_20_clients}
\end{figure*}

\begin{figure*}[!h]
	\centering
	\subfigure[\small
		$\alpha \!=\! 100$
	]{
		\includegraphics[width=0.31\textwidth,]{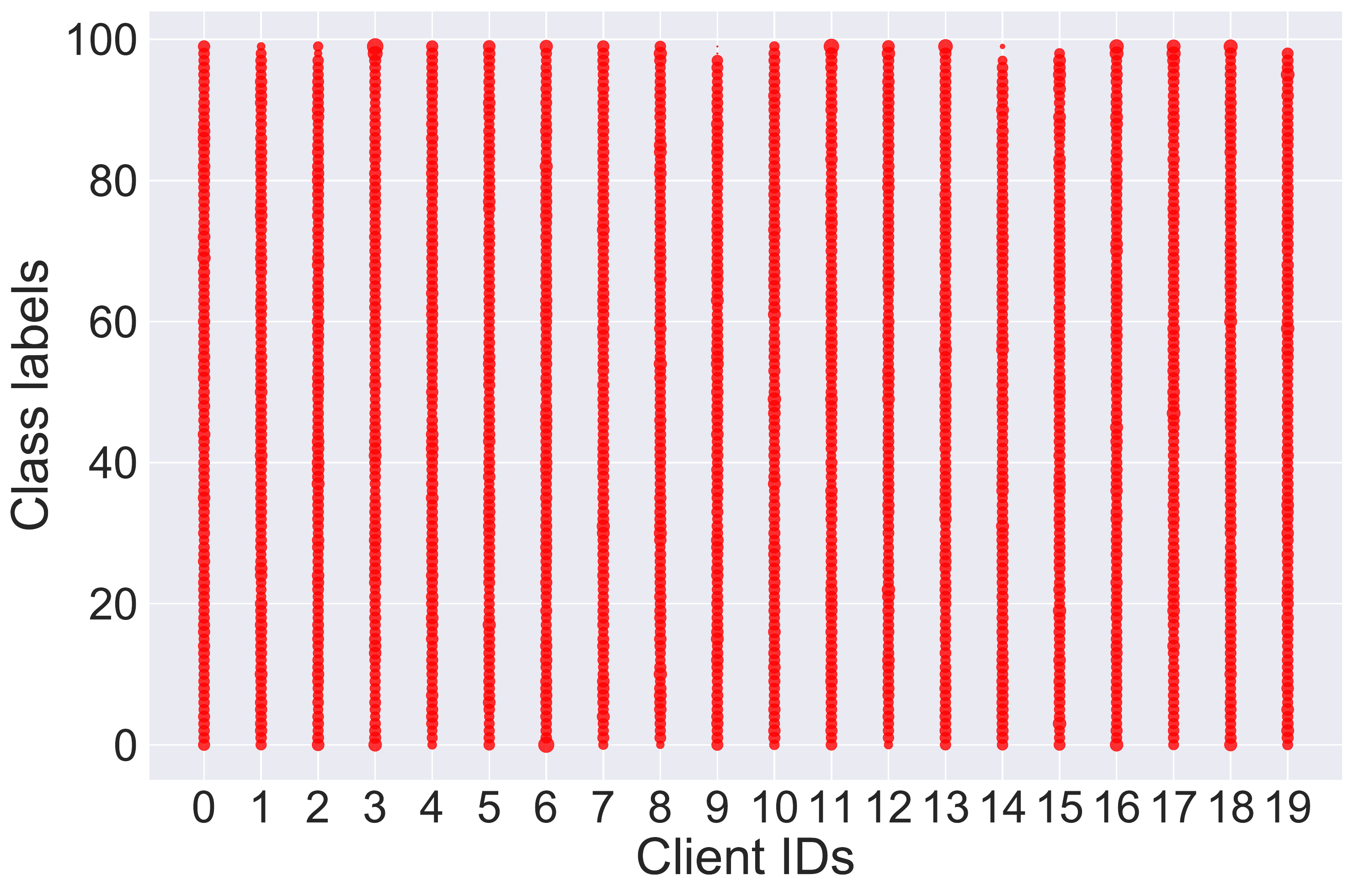}
		\label{fig:cifar100_20_clients_alpha_100}
	}
	\hfill
	\subfigure[\small
		$\alpha \!=\! 1$
	]{
		\includegraphics[width=0.31\textwidth,]{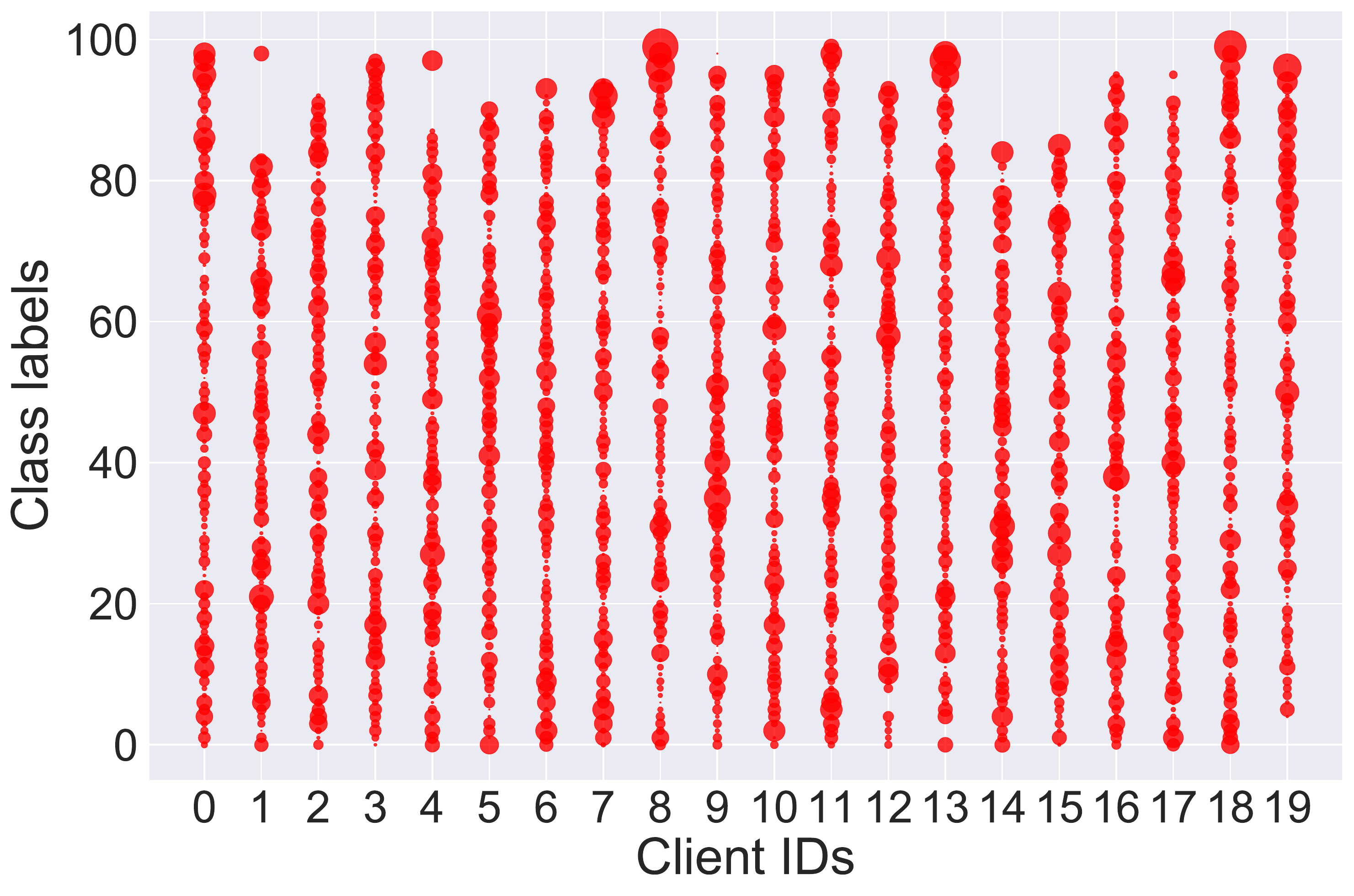}
		\label{fig:cifar100_20_clients_alpha_1}
	}
	\hfill
	\subfigure[\small
		$\alpha \!=\! 0.5$
	]{
		\includegraphics[width=0.31\textwidth,]{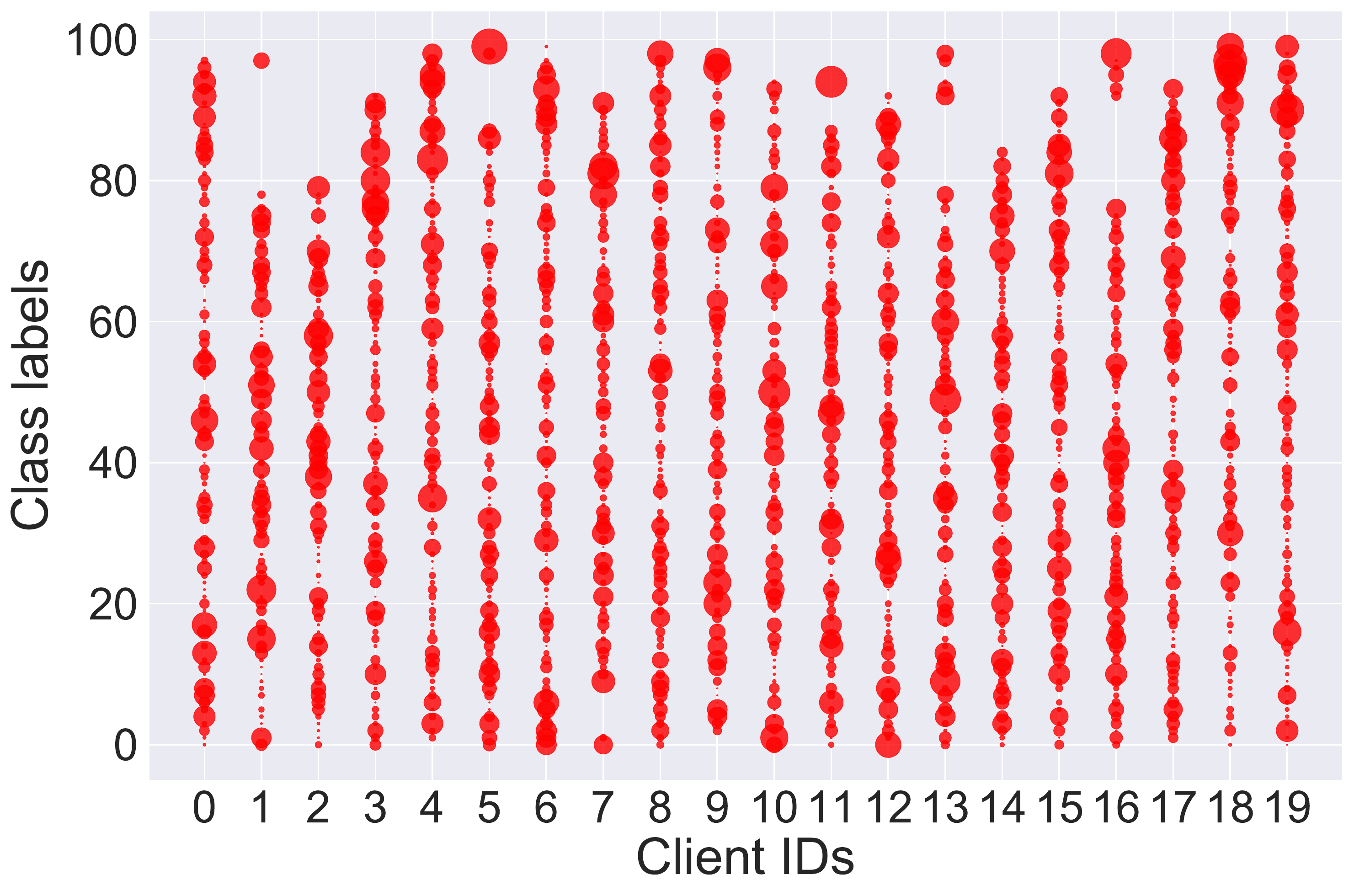}
		\label{fig:cifar100_20_clients_alpha_05}
	}
	\hfill
	\subfigure[\small
		$\alpha \!=\! 0.1$
	]{
		\includegraphics[width=0.31\textwidth,]{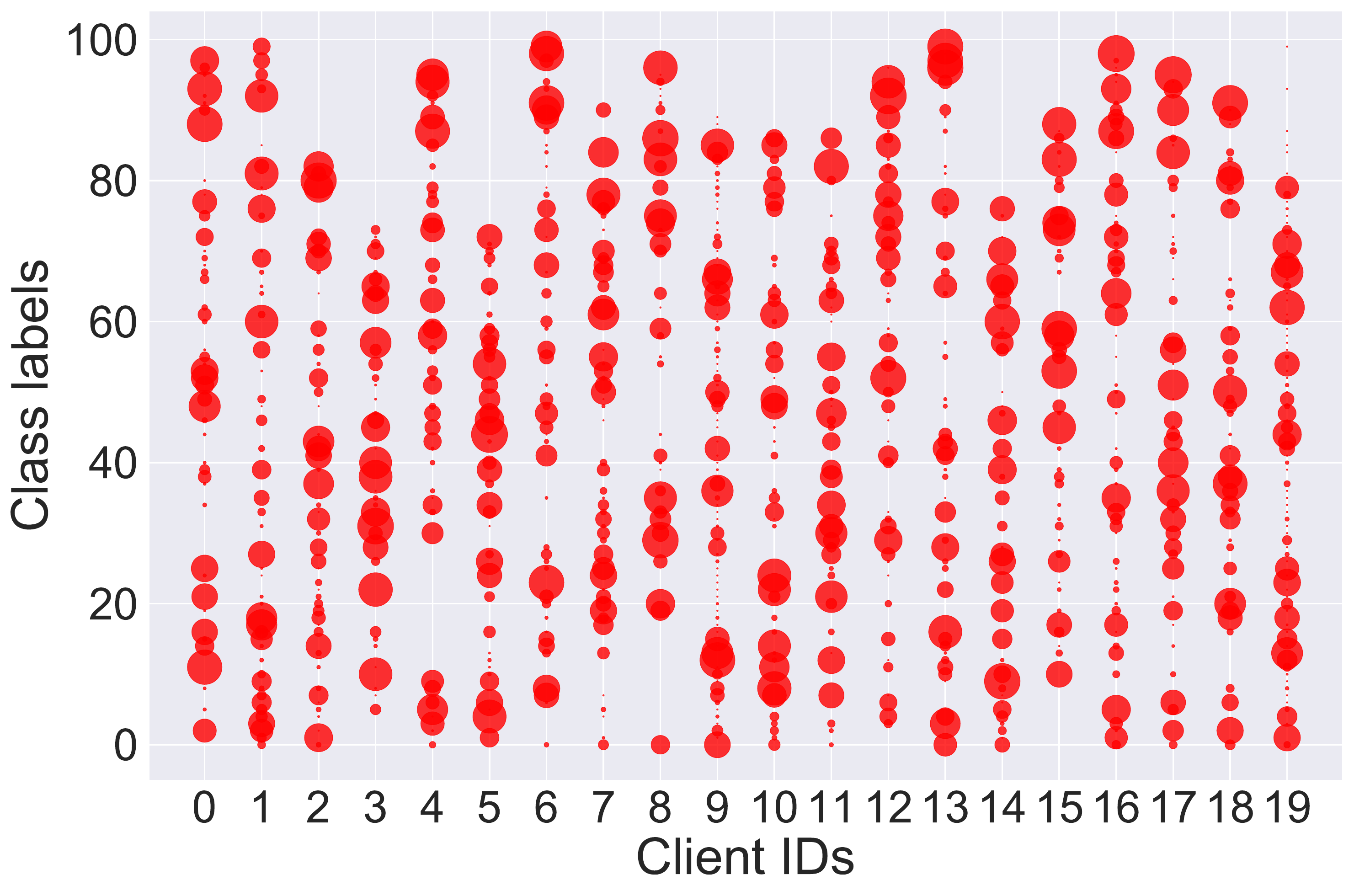}
		\label{fig:cifar100_20_clients_alpha_01}
	}
	\hfill
	\subfigure[\small
		$\alpha \!=\! 0.01$
	]{
		\includegraphics[width=0.31\textwidth,]{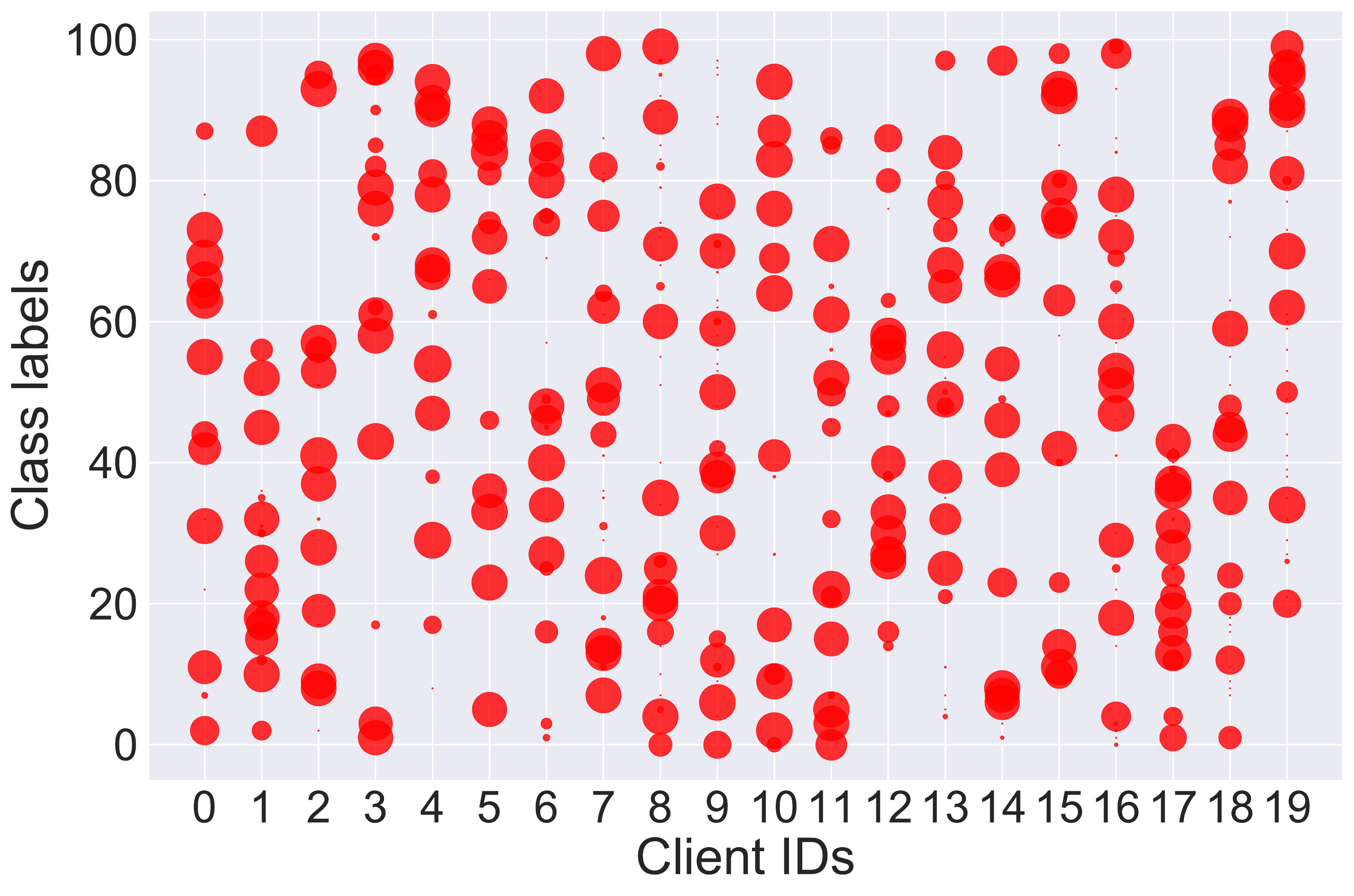}
		\label{fig:cifar100_20_clients_alpha_001}
	}
	\vspace{-0.5em}
	\caption{\small
		Classes allocated to each client at different Dirichlet distribution alpha values,
		for CIFAR-100 with 20 clients.
		The size of each dot reflects the magnitude of the samples number.
	}
	\vspace{-0.5em}
	\label{fig:class_distribution_for_alphas_cifar100_20_clients}
\end{figure*}

\begin{figure*}[!h]
	\centering
	\subfigure[\small
		AG News
	]{
		\includegraphics[width=0.46\textwidth,]{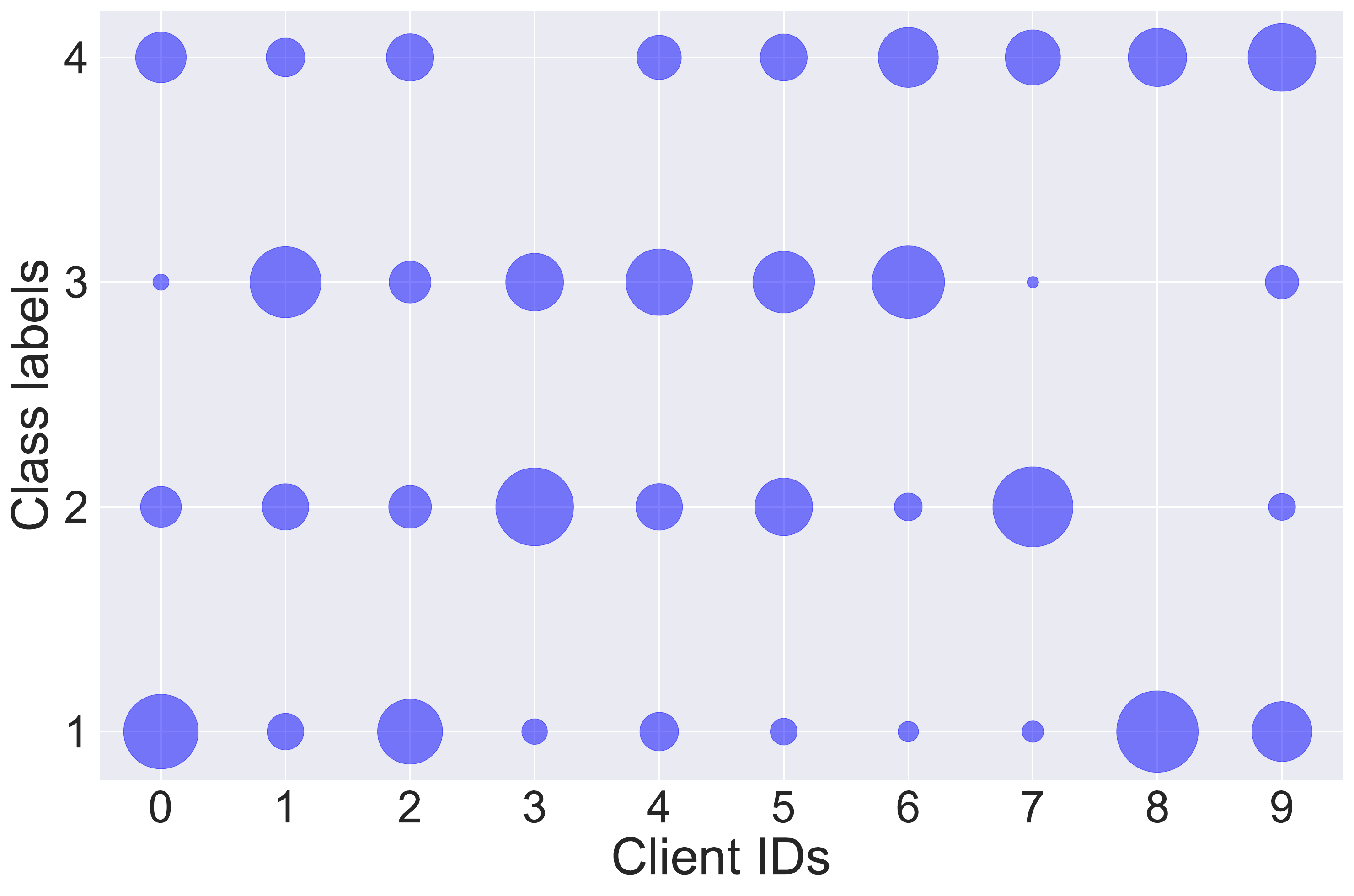}
		\label{fig:agnews_partition}
	}
	\hfill
	\subfigure[\small
		SST2
	]{
		\includegraphics[width=0.46\textwidth,]{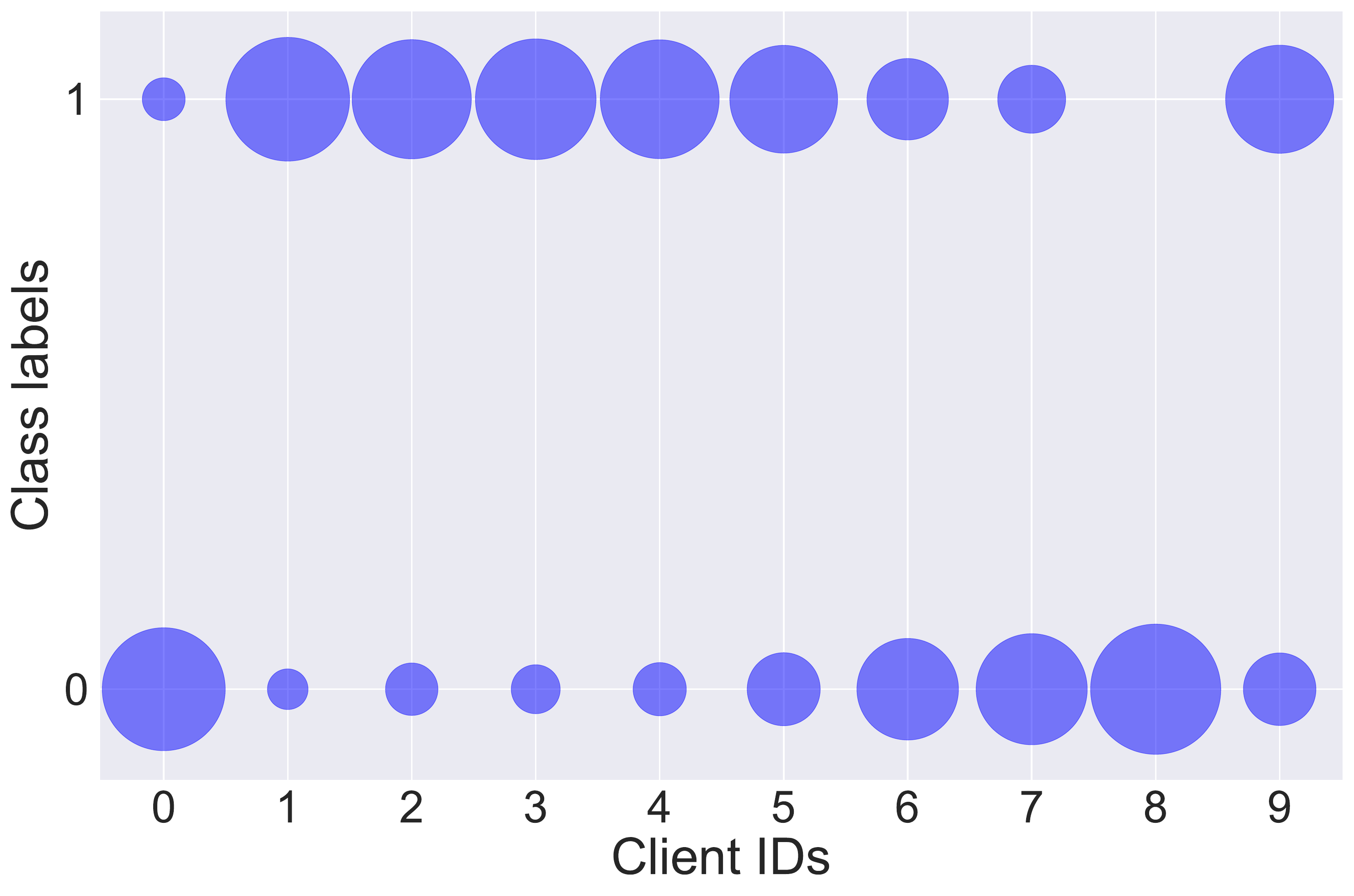}
		\label{fig:sst2_partition}
	}
	\vspace{-0.5em}
	\caption{\small
		Classes allocated to each client at Dirichlet distribution $\alpha = 1$,
		for AG News and SST2 datasets with 10 clients.
		The size of each dot reflects the magnitude of the samples number.
	}
	\vspace{-0.5em}
	\label{fig:nlp_partition}
\end{figure*}

\subsection{Some Empirical Understanding of \fedavg} \label{appendix:empirical_understanding_fedavg}
Figure~\ref{fig:ablation_study_resnet8_cifar10_fedavg_different_non_iid_different_local_epochs_different_lr_schedule}
reviews the general behaviors of \fedavg under different non-iid degrees of local data,
different local data sizes, different numbers of local epochs per communication round,
as well as the learning rate schedule during the local training.
Since we cannot observe the benefits of decaying the learning rate during the local training phase,
we turn off the learning rate decay for the experiments in the main text.

In Figure~\ref{fig:resnet8_cifar10_impact_of_normalization},
we visualize the learning curves of training ResNet-8 on CIFAR-10 with different normalization techniques.
The numerical results correspond to Table~\ref{tab:resnet8_cifar10_impact_of_normalization} in the main text.

\begin{figure*}[!h]
	\centering
	\subfigure[\small $\alpha \!=\! 100$]{
		\includegraphics[width=0.31\textwidth,]{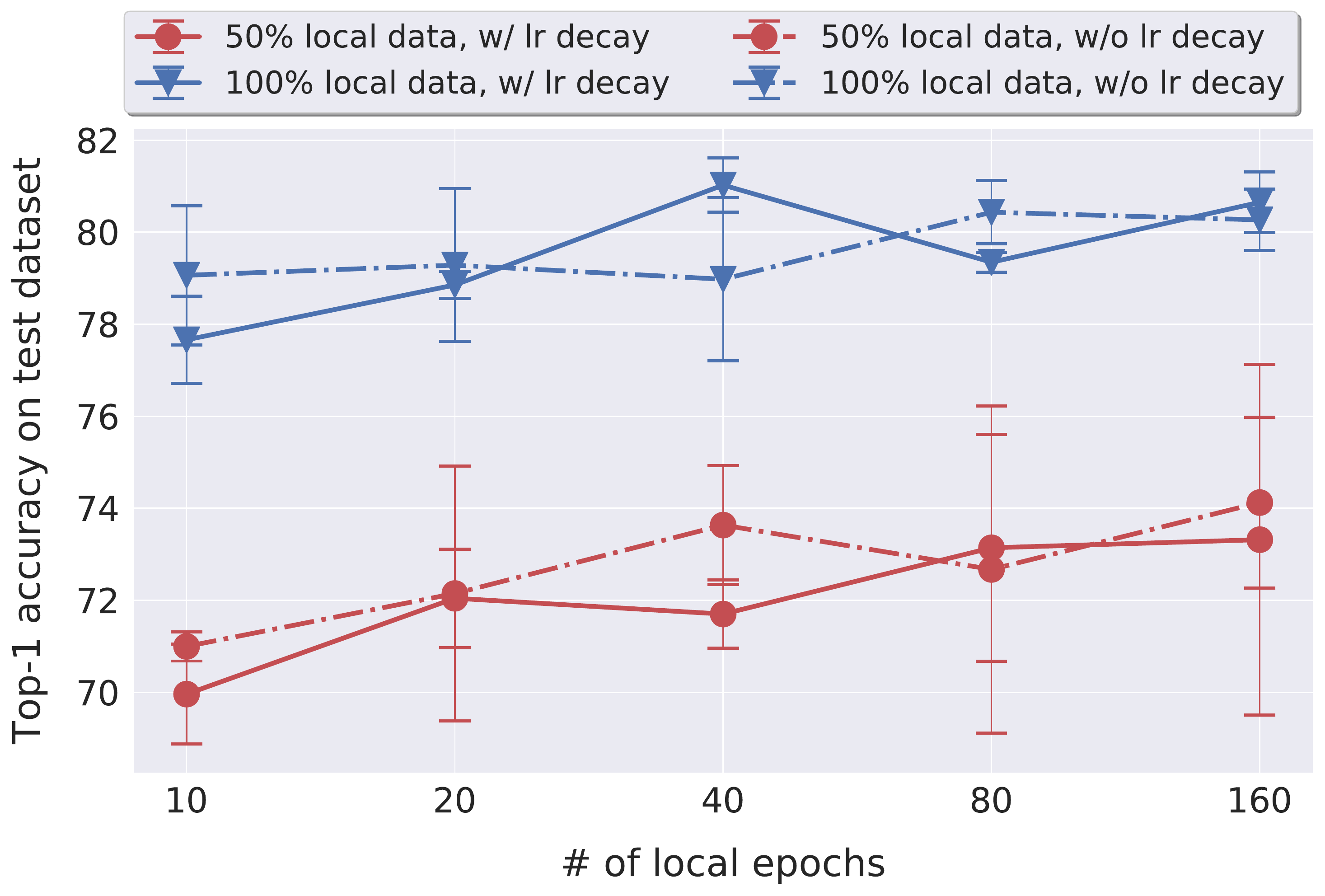}
		\label{fig:resnet8_cifar10_non_iid_100_fedavg_impact_of_lr_decay}
	}
	\hfill
	\subfigure[\small $\alpha \!=\! 1$]{
		\includegraphics[width=0.31\textwidth,]{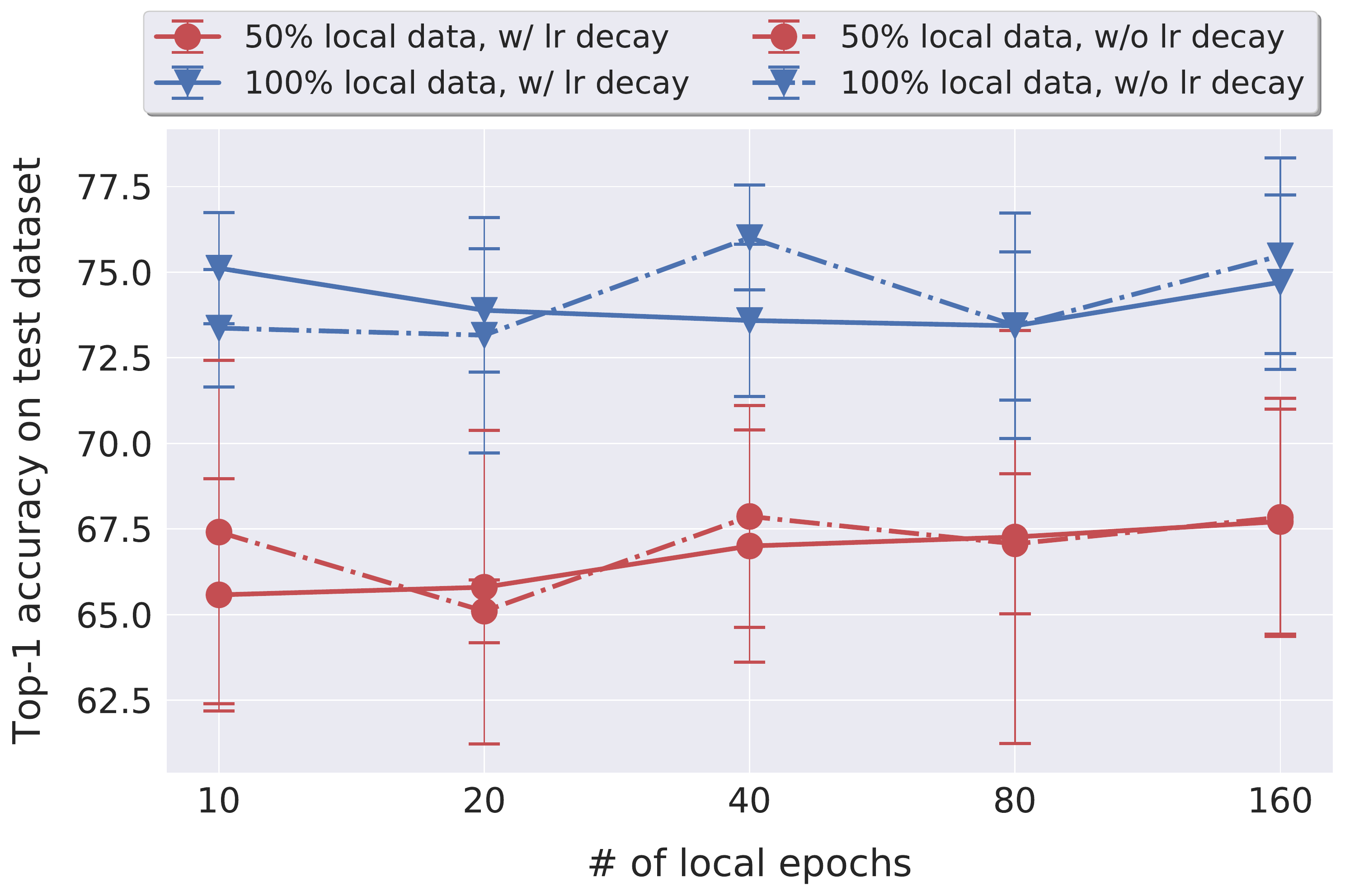}
		\label{fig:resnet8_cifar10_non_iid_1_fedavg_impact_of_lr_decay}
	}
	\hfill
	\subfigure[\small $\alpha \!=\! 0.01$]{
		\includegraphics[width=0.31\textwidth,]{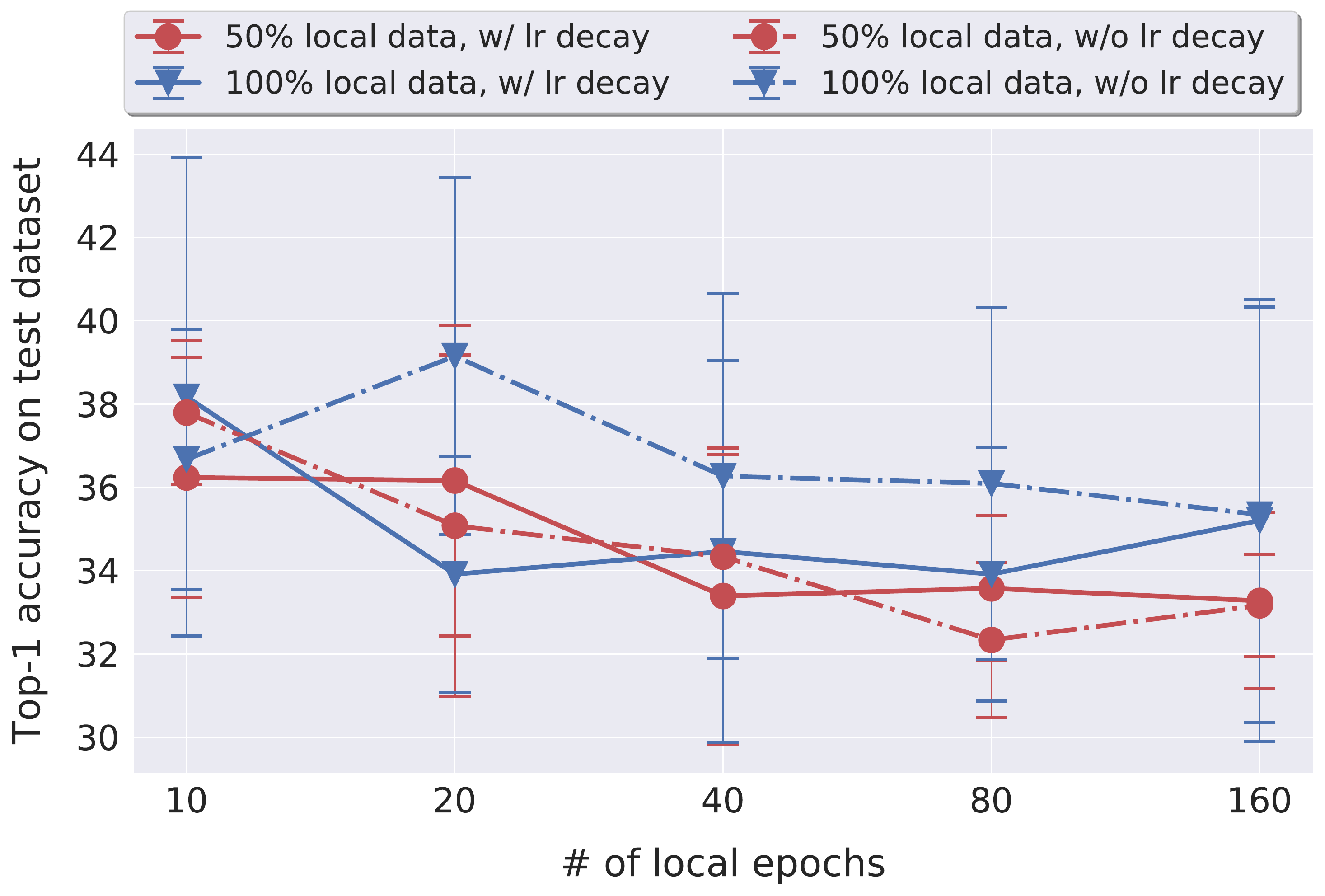}
		\label{fig:resnet8_cifar10_non_iid_001_fedavg_impact_of_lr_decay}
	}
	\vspace{-0.5em}
	\caption{\small
		\textbf{The ablation study of \fedavg for different \# of local epochs and learning rate schedules},
		for standard federated learning on CIFAR-10 with ResNet-8.
		For each communication round ($100$ in total),
		$40\%$ of the total $20$ clients are randomly selected.
		We use $\alpha$ to synthetically control the non-iid degree of the local data,
		as in~\cite{yurochkin2019bayesian,hsu2019measuring}.
		The smaller $\alpha$, the larger discrepancy between local data distributions
		($\alpha \!=\! 100$ mimics identical local data distributions).
		We report the top-1 accuracy (on three different seeds) on the test dataset.
	}
	\vspace{-0.5em}
	\label{fig:ablation_study_resnet8_cifar10_fedavg_different_non_iid_different_local_epochs_different_lr_schedule}
\end{figure*}

\begin{figure*}[!h]
	\centering
	\includegraphics[width=0.5\textwidth,]{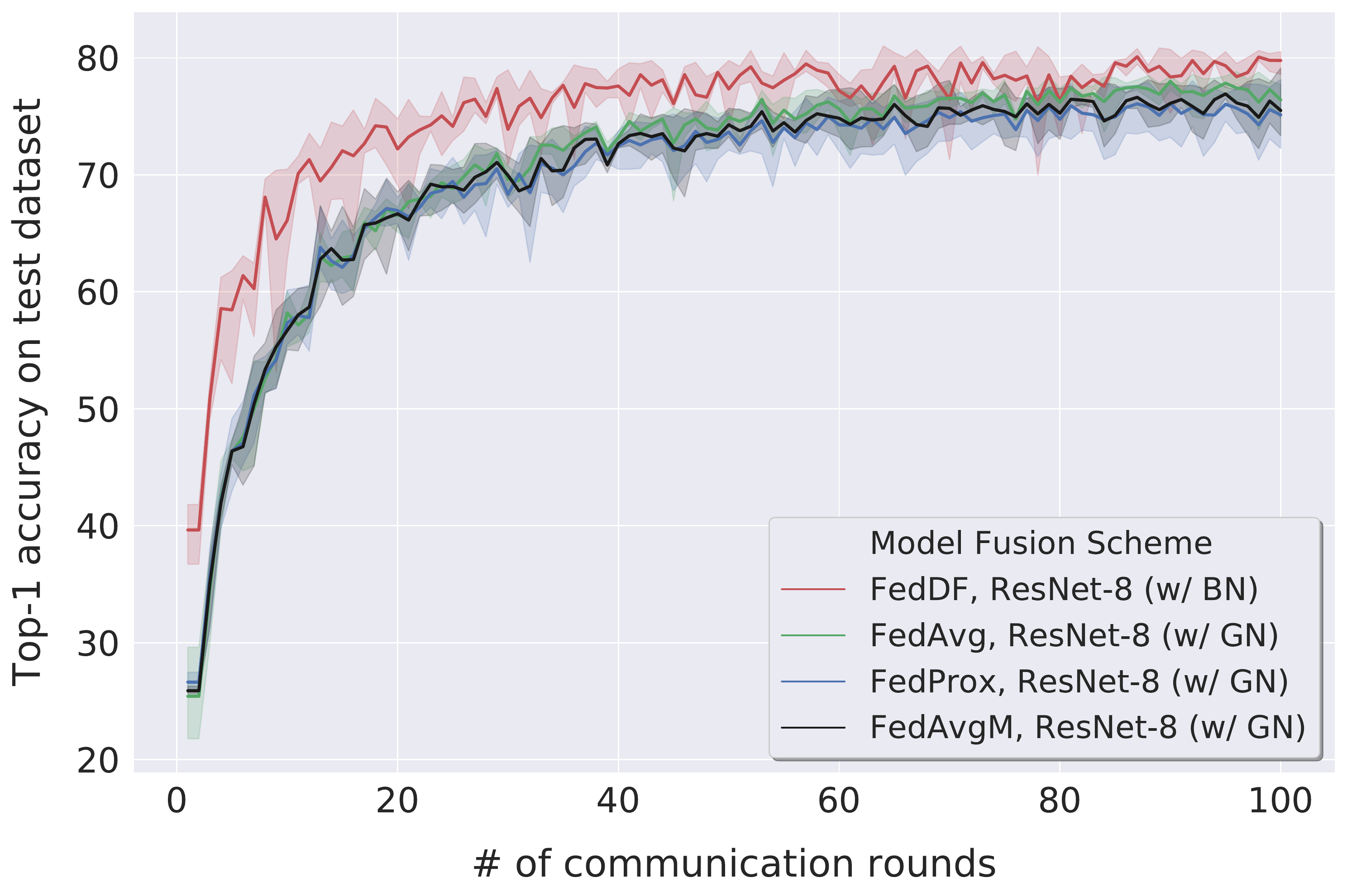}
	\vspace{-0.5em}
	\caption{\small
		The impact of different normalization techniques, i.e., Batch Normalization (BN), Group Normalization (GN),
		for \fl on CIFAR-10 with ResNet-8 with $\alpha = 1$.
		For each communication round ($100$ in total), $40\%$ of the total $20$ clients are randomly selected for $40$ local epochs.
	}
	\vspace{-0.5em}
	\label{fig:resnet8_cifar10_impact_of_normalization}
\end{figure*}

\subsection{The Advantages of \algopt} \label{appendix:standard_fl_scenario}
\subsubsection{Ablation Study} \label{appendix:ablation_study}
\paragraph{The Importance of the Model Initialization in \algopt.}
We empirically study the importance of the initialization (before performing ensemble distillation) in \algopt.
Table~\ref{tab:resnet_8_cifar10_importance_of_init}
demonstrates the performance difference of~\algopt for two different model initialization schemes:
1) ``from average'',
where the uniformly averaged model from this communication round is used as the initial model
(i.e.\ the default design choice of \algopt as illustrated in Algorithm~\ref{alg:homogeneous_framework}
and Algorithm~\ref{alg:heterogeneous_framework});
and 2) ``from previous'',
where we initialize the model for ensemble distillation
by utilizing the fusion result of \algopt from the previous communication round.
The noticeable performance differences illustrated in Table~\ref{tab:resnet_8_cifar10_importance_of_init}
identify the importance of using the uniformly averaged model\footnote{
	The related preprints~\cite{li2019fedmd,chang2019cronus}
	are closer to the second initialization scheme.
	They do not or cannot introduce the uniformly averaged model (on the server) into the federated learning pipeline;
	instead, they only utilize the averaged logits (on the same data) for each client's local training.
} (from the current communication round)
as a starting model for better ensemble distillation.

\begin{table}[!h]
	\centering
	\caption{\small
		\textbf{Understanding the importance of model initialization in \algopt}, on CIFAR-10 with ResNet-8.
		For each communication round ($100$ in total),
		$40\%$ of the total $20$ clients are randomly selected.
		The scheme ``from average'' indicates initializing the model
		for ensemble distillation from the uniformly averaged model of this communication round;
		while the scheme ``from previous'' instead uses the fused model from the previous communication round
		as the starting point.
		We report the top-1 accuracy (on three different seeds) on the test dataset.
	}
	\label{tab:resnet_8_cifar10_importance_of_init}
	\resizebox{.7\textwidth}{!}{%
		\begin{tabular}{ccccc}
			\toprule
			& \multicolumn{2}{c}{$\alpha \!=\! 1$} & \multicolumn{2}{c}{$\alpha \!=\! 0.1$}  \\ \cmidrule(lr){2-3} \cmidrule(lr){4-5}
			local training epochs & from average     & from previous    & from average     & from previous    \\ \midrule
			40                    & $80.43 \pm 0.37$ & $74.13 \pm 0.91$ & $71.84 \pm 0.86$ & $62.94 \pm 1.12$ \\
			80                    & $81.17 \pm 0.53$ & $76.37 \pm 0.60$ & $74.73 \pm 0.65$ & $67.88 \pm 0.90$ \\
			\bottomrule
		\end{tabular}%
	}
\end{table}

\paragraph{The performance gain in~\algopt.}
To distinguish the benefits of FedDF from the small learning rate (during the local training)
or Adam optimizer (used for ensemble distillation in~\algopt),
we report the results of using
Adam (lr=1e-3) for both local training and model fusion (over three seeds), on CIFAR-10 with ResNet-8,
in Table~\ref{tab:resnet_8_cifar10_different_local_training_scheme}.
Improving the local training through Adam might help \FL
but the benefit vanishes with higher data heterogeneity (e.g. $\alpha = 0.1$).
Performance gain from \algopt is robust to data heterogeneity
and also orthogonal to effects of learning rates and Adam.

\begin{table}[!h]
	\centering
	\caption{\small
		\textbf{Understanding the impact of local training quality}, on CIFAR-10 with ResNet-8.
		For each communication round ($100$ in total),
		$40\%$ of the total $20$ clients are randomly selected for $40$ local epochs.
		We report the top-1 accuracy (on three different seeds) on the test dataset.
	}
	\label{tab:resnet_8_cifar10_different_local_training_scheme}
	\resizebox{.6\textwidth}{!}{%
		\begin{tabular}{ccccc}
			\toprule
			& \multicolumn{2}{c}{$\alpha \!=\! 1$} & \multicolumn{2}{c}{$\alpha \!=\! 0.1$}  \\ \cmidrule(lr){2-3} \cmidrule(lr){4-5}
			local client training scheme & \algopt & \fedavg & \algopt & \fedavg \\ \midrule
			SGD                          & $80.27$ & $72.73$ & $71.52$ & $62.44$ \\
			Adam                         & $83.32$ & $78.13$ & $72.58$ & $62.53$ \\
			\bottomrule
		\end{tabular}%
	}
\end{table}

Table~\ref{tab:resnet_8_cifar10_effect_of_different_optimizers_for_distillation}
examines the effect of different optimization schemes on the quality of ensemble distillation.
We can witness that with two extra hyper-parameters (sampling scale for SWAG and the number of models to be sampled),
SWAG can slightly improve the distillation performance.
In contrast, we use Adam with default hyper-parameters as our design choice in~\algopt:
it demonstrates similar performance (compared to the choice of SWAG) with trivial tuning overhead.

\begin{table}[!h]
	\centering
	\caption{\small
		\textbf{On the impact of using different optimizers for ensemble distillation} in~\algopt,
		on CIFAR-10 with ResNet-8.
		For each communication round ($100$ in total),
		$40\%$ of the total $20$ clients are randomly selected for $40$ local epochs.
		We report the top-1 accuracy (on three different seeds) on the test dataset.
		``SGD'' uses the same learning rate scheduler as our ``Adam'' choice (i.e.\ cosine annealing),
		and with fine-tuned initial learning rate.
		``SWAG'' refers to the mechanism to form an approximated posterior distribution~\cite{maddox2019simple}
		where more models can be sampled from,
		and \cite{chen2020feddistill} further propose to
		use SWAG on the received client models for better ensemble distillation;
		our default design resorts to directly averaged logits
		from received local clients with Adam optimizer.
		To ensure a fair comparison,
		we use the same distillation dataset as in~\algopt (i.e., CIFAR-100) for ``SWAG''~\cite{chen2020feddistill}.
		We fine-tune other hyper-parameters in ``SWAG'':
		we use all received client models and $10$ sampled models from Gaussian distribution
		(as suggested in~\cite{chen2020feddistill}) for the ensemble distillation.
	}
	\label{tab:resnet_8_cifar10_effect_of_different_optimizers_for_distillation}
	\resizebox{.6\textwidth}{!}{%
		\begin{tabular}{ccccc}
			\toprule
			& \multicolumn{2}{c}{$\alpha \!=\! 1$} & \multicolumn{2}{c}{$\alpha \!=\! 0.1$}  \\ \cmidrule(lr){2-3} \cmidrule(lr){4-5}
			optimizer used on the server                    & \algopt & \fedavg & \algopt & \fedavg \\ \midrule
			SGD                                             & $76.68$ & $72.73$ & $57.33$ & $62.44$ \\
			Adam (our default design)                       & $80.27$ & $72.73$ & $71.52$ & $62.44$ \\
			SWAG~\cite{maddox2019simple,chen2020feddistill} & $80.84$ & $72.73$ & $72.40$ & $62.44$ \\
			\bottomrule
		\end{tabular}%
	}
\end{table}

\paragraph{The compatibility of~\algopt with other methods.}
Table~\ref{tab:resnet_8_cifar10_compatibility_with_other_methods} justifies the compatibility of~\algopt.
Our empirical results demonstrate a significant performance gain of~\algopt over the \fedavg,
even in the case of using local proximal regularizer to
avoid catastrophically over-fitting the heterogeneous local data,
which reduces the diversity of local models that \algopt benefits from.

\begin{table}[!h]
	\centering
	\caption{\small
		\textbf{The compatibility of \algopt with other training schemes}, on CIFAR-10 with ResNet-8.
		For each communication round ($100$ in total),
		$40\%$ of the total $20$ clients are randomly selected for $40$ local epochs.
		We consider the fine-tuned proximal penalty from FedDF.
		We report the top-1 accuracy (on three different seeds) on the test dataset.
	}
	\label{tab:resnet_8_cifar10_compatibility_with_other_methods}
	\resizebox{.6\textwidth}{!}{%
		\begin{tabular}{ccccc}
			\toprule
			& \multicolumn{2}{c}{$\alpha \!=\! 1$} & \multicolumn{2}{c}{$\alpha \!=\! 0.1$}  \\ \cmidrule(lr){2-3} \cmidrule(lr){4-5}
			local client training scheme & \algopt & \fedavg & \algopt & \fedavg \\ \midrule
			SGD                          & $80.27$ & $72.73$ & $71.52$ & $62.44$ \\
			SGD + proximal penalty       & $80.56$ & $76.11$ & $71.64$ & $62.53$ \\
			\bottomrule
		\end{tabular}%
	}
\end{table}

\subsubsection{Comparison with \fedavg}
Figure~\ref{fig:understanding_learning_behaviors_resnet8_cifar10_kt_different_non_iid_degrees_complete}
complements Figure~\ref{fig:understanding_learning_behaviors_resnet8_cifar10_kt_different_non_iid_degrees}
in the main text
and presents a thorough comparison between \fedavg and \algopt,
for a variety of different local training epochs, data fractions, non-\iid degrees.
The detailed learning curves of the cases in this figure
are visualized in
Figure~\ref{fig:appendix_understanding_learning_behaviors_resnet8_cifar10_kt_non_iid_100_localdata},
Figure~\ref{fig:appendix_understanding_learning_behaviors_resnet8_cifar10_kt_non_iid_1_localdata},
and Figure~\ref{fig:appendix_understanding_learning_behaviors_resnet8_cifar10_kt_non_iid_001_localdata}.

\begin{figure*}[!h]
	\centering
	\subfigure[\small
		$\alpha \!=\! 100$.
	]{
		\includegraphics[width=0.31\textwidth,]{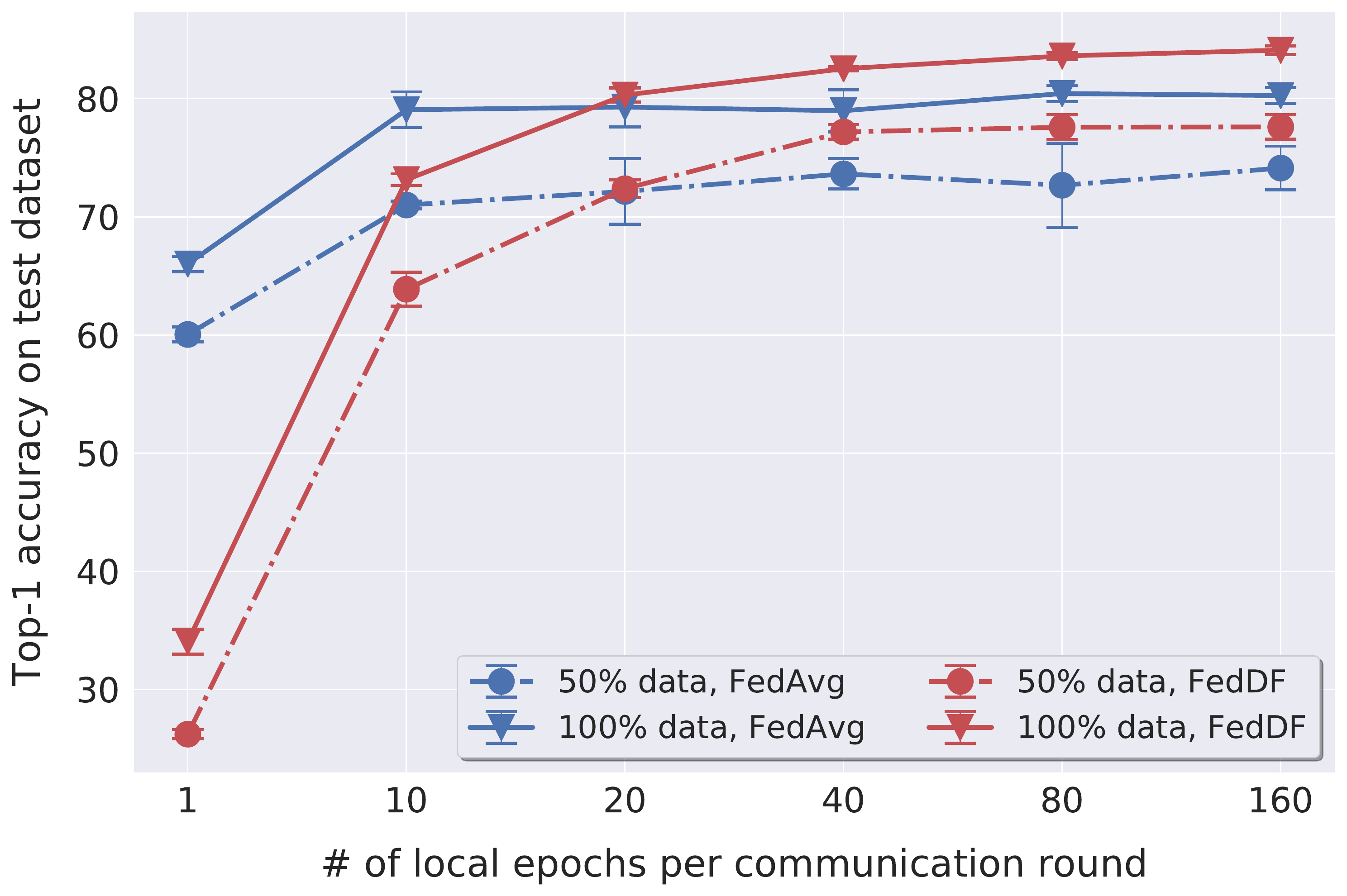}
		\label{fig:resnet8_cifar10_non_iid_100_fedavg_vs_kt_results}
	}
	\hfill
	\subfigure[\small
		$\alpha \!=\! 1$.
	]{
		\includegraphics[width=0.31\textwidth,]{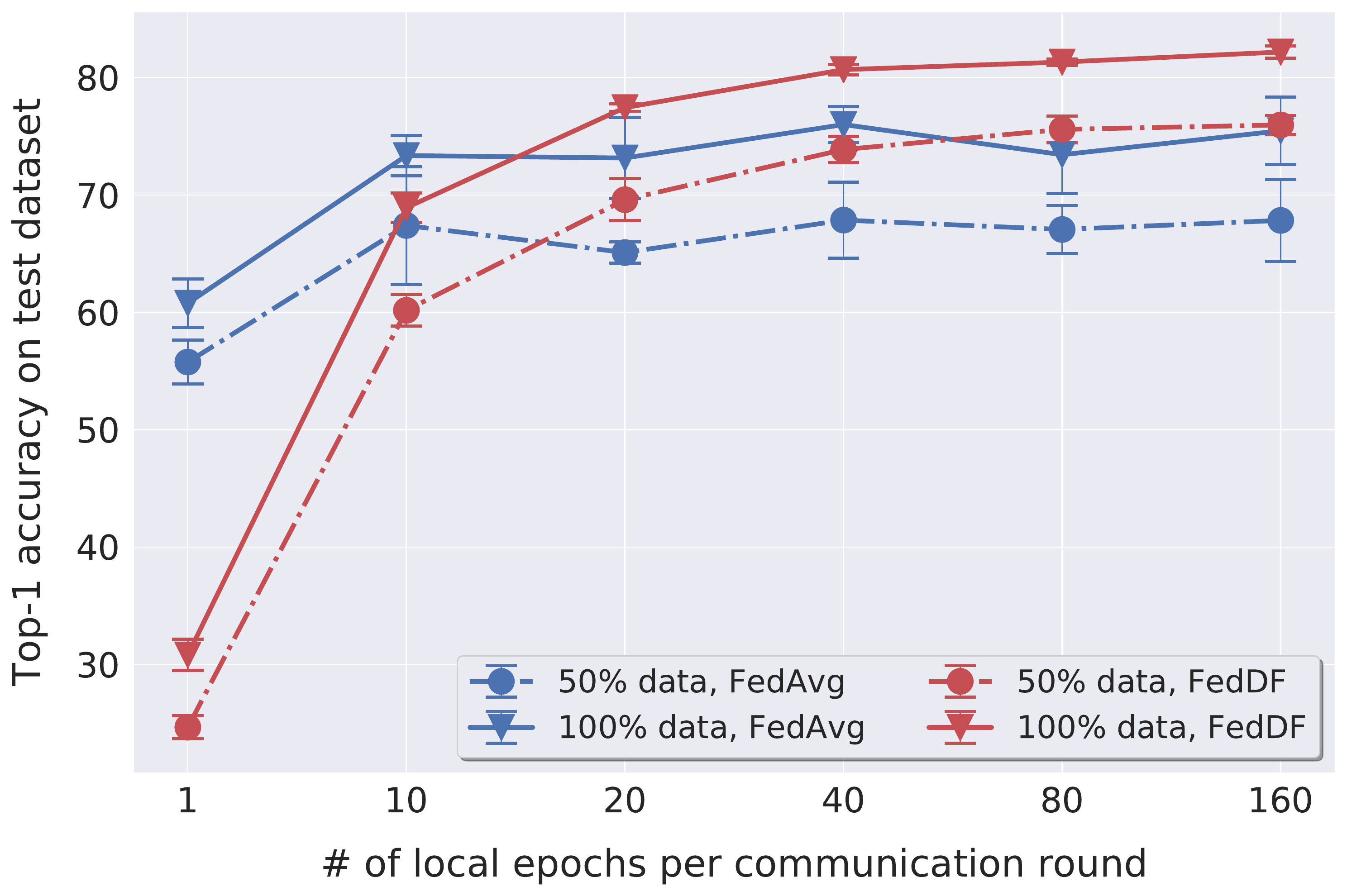}
		\label{fig:resnet8_cifar10_non_iid_1_fedavg_vs_kt_results}
	}
	\hfill
	\subfigure[\small
		$\alpha \!=\! 0.01$.
	]{
		\includegraphics[width=0.31\textwidth,]{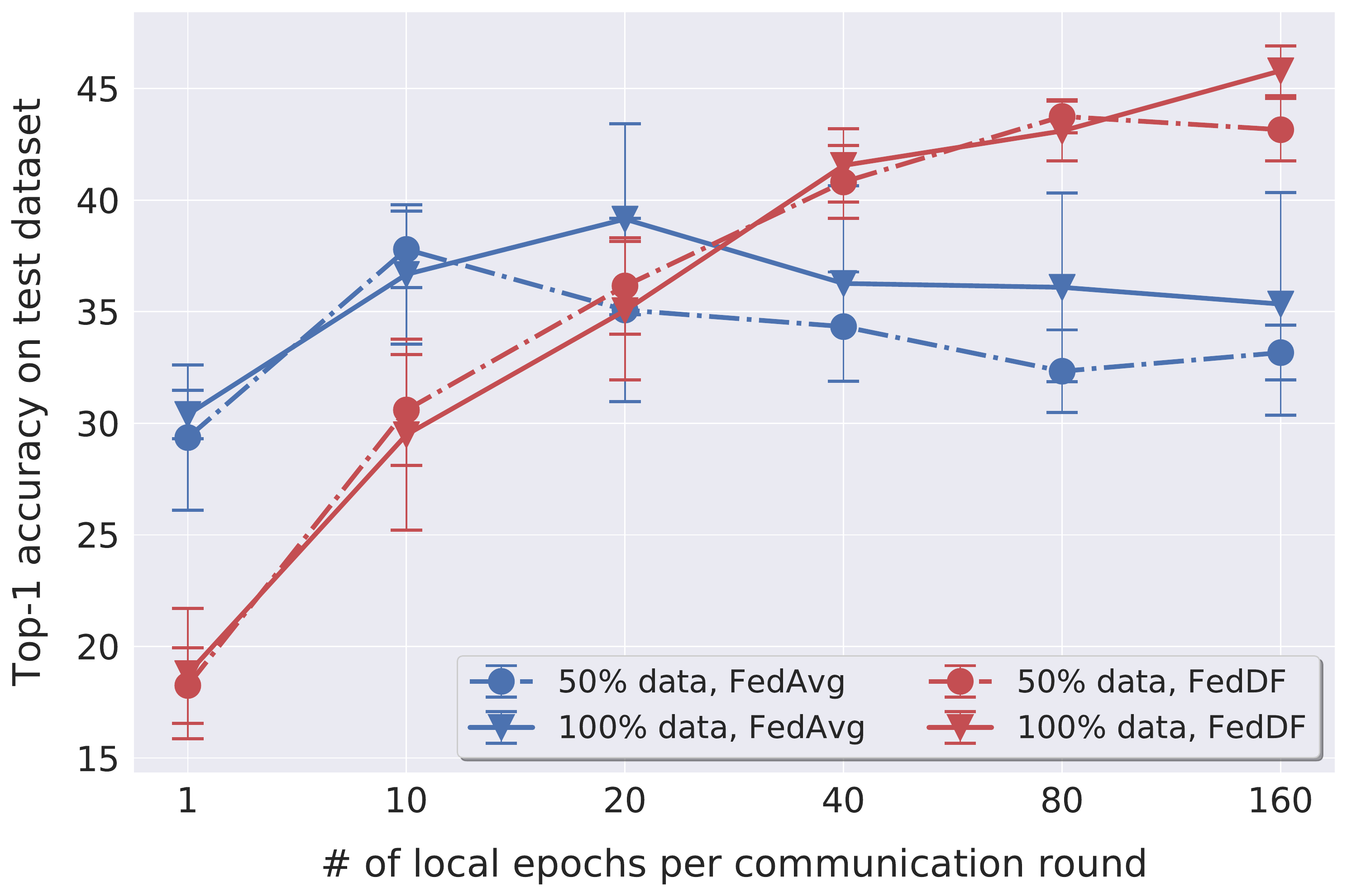}
		\label{fig:resnet8_cifar10_non_iid_001_fedavg_vs_kt_results}
	}
	\vspace{-0.5em}
	\caption{\small
		The \textbf{test performance} of \textbf{\algopt} and \textbf{\fedavg} on \textbf{CIFAR-10} with \textbf{ResNet-8},
		for different local data non-iid degrees $\alpha$, data fractions,
		and \# of local epochs per communication round.
		For each communication round ($100$ in total),
		$40\%$ of the total $20$ clients are randomly selected.
		We report the top-1 accuracy (on three different seeds) on the test dataset.
		This Figure complements Figure~\ref{fig:understanding_learning_behaviors_resnet8_cifar10_kt_different_non_iid_degrees}.
	}
	\vspace{-0.5em}
	\label{fig:understanding_learning_behaviors_resnet8_cifar10_kt_different_non_iid_degrees_complete}
\end{figure*}

\begin{figure*}[!h]
	\centering
	\subfigure[\small
		The learning behaviors of \algopt and \fedavg.
		We evaluate different \# of local epochs on $100\%$ local data.
	]{
		\includegraphics[width=0.475\textwidth,]{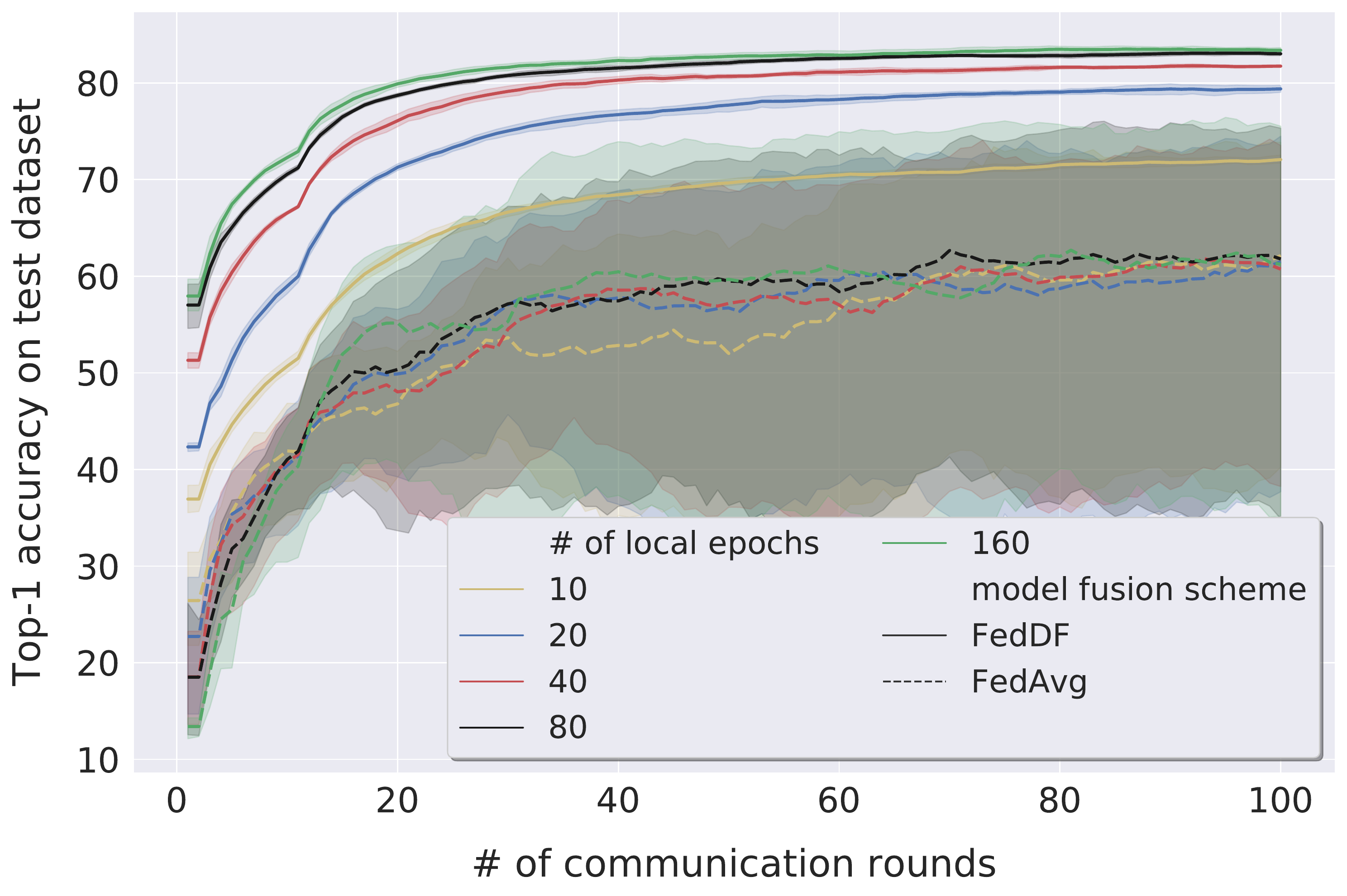}
		\label{fig:appendix_resnet8_cifar10_1_localdata_non_iid_100_fedavg_vs_kt_training_curves}
	}
	\hfill
	\subfigure[\small
		The fused model performance before (i.e.\ line 6 in Algorithm~\ref{alg:homogeneous_framework}) and after \algopt (i.e.\ line 10 in Algorithm~\ref{alg:homogeneous_framework}).
		We evaluate different \# of local epochs on $100\%$ local data.
	]{
		\includegraphics[width=0.475\textwidth,]{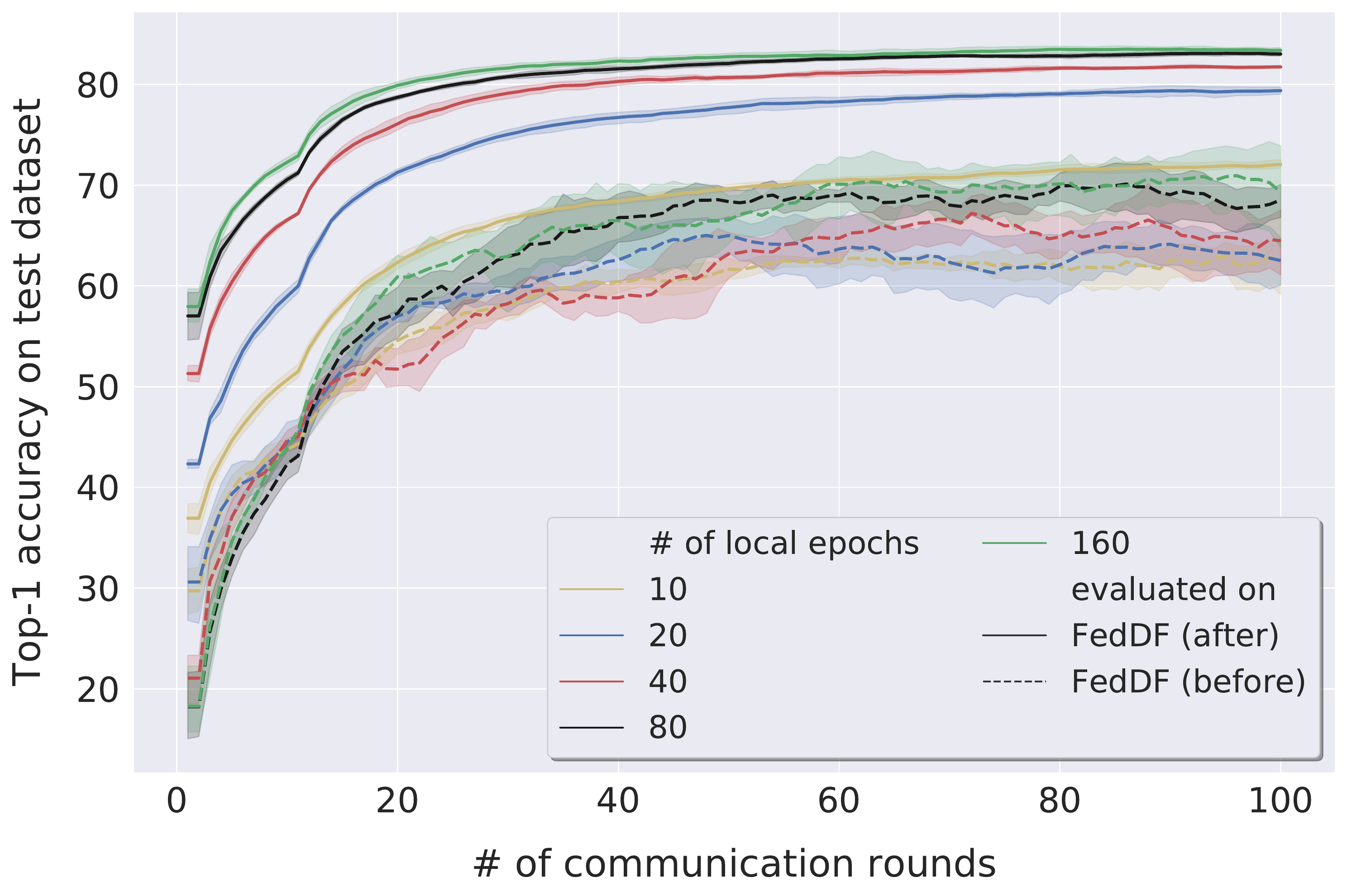}
		\label{fig:appendix_resnet8_cifar10_1_localdata_non_iid_100_kt_agg_vs_post_agg}
	}
	\subfigure[\small
		The learning behaviors of \algopt and \fedavg.
		We evaluate different \# of local epochs on $50\%$ local data.
	]{
		\includegraphics[width=0.475\textwidth,]{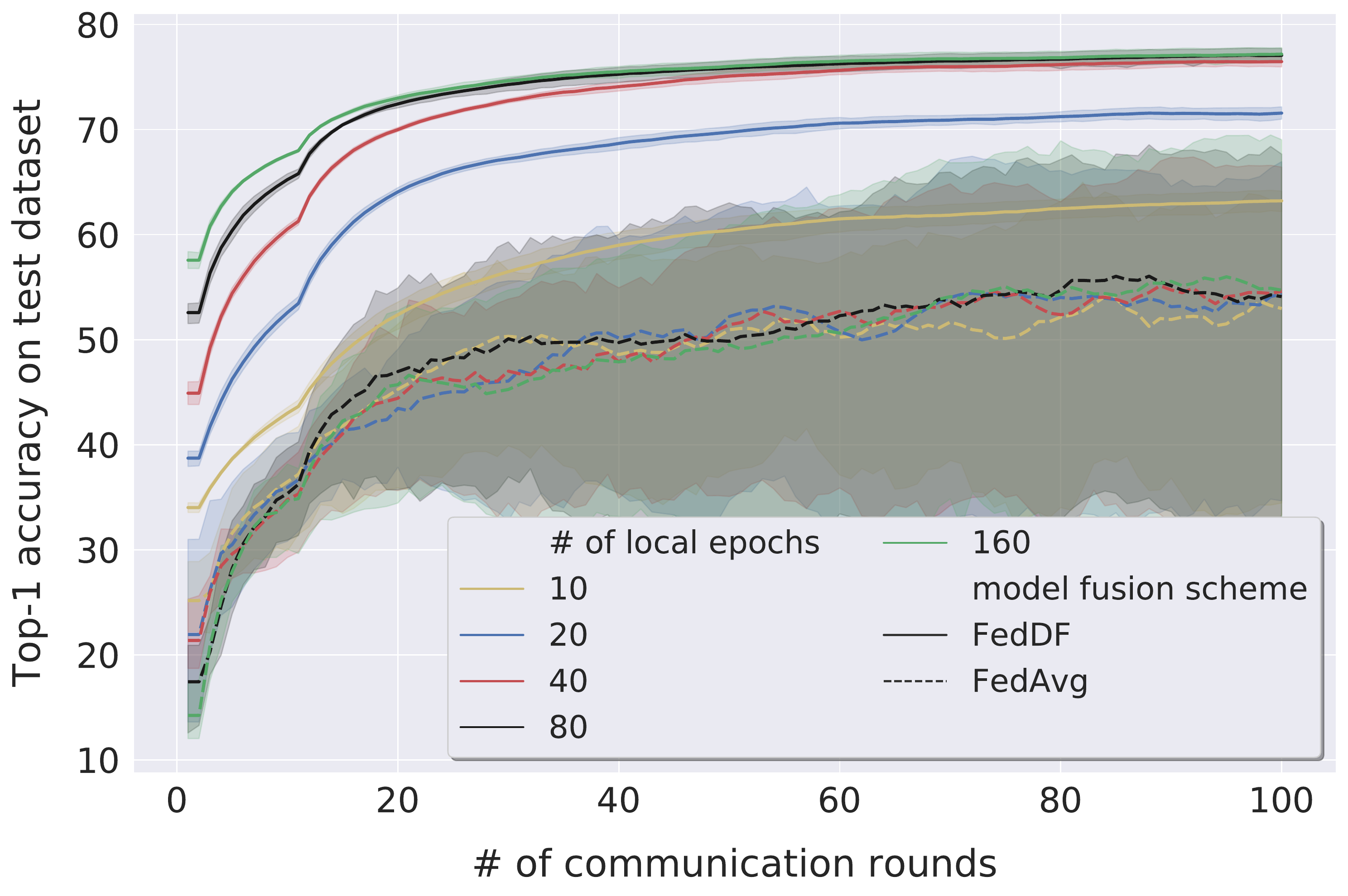}
		\label{fig:appendix_resnet8_cifar10_05_localdata_non_iid_100_fedavg_vs_kt_training_curves}
	}
	\hfill
	\subfigure[\small
		The fused model performance before (i.e.\ line 6 in Algorithm~\ref{alg:homogeneous_framework}) and after \algopt (i.e.\ line 10 in Algorithm~\ref{alg:homogeneous_framework}).
		We evaluate different \# of local epochs on $50\%$ local data.
	]{
		\includegraphics[width=0.475\textwidth,]{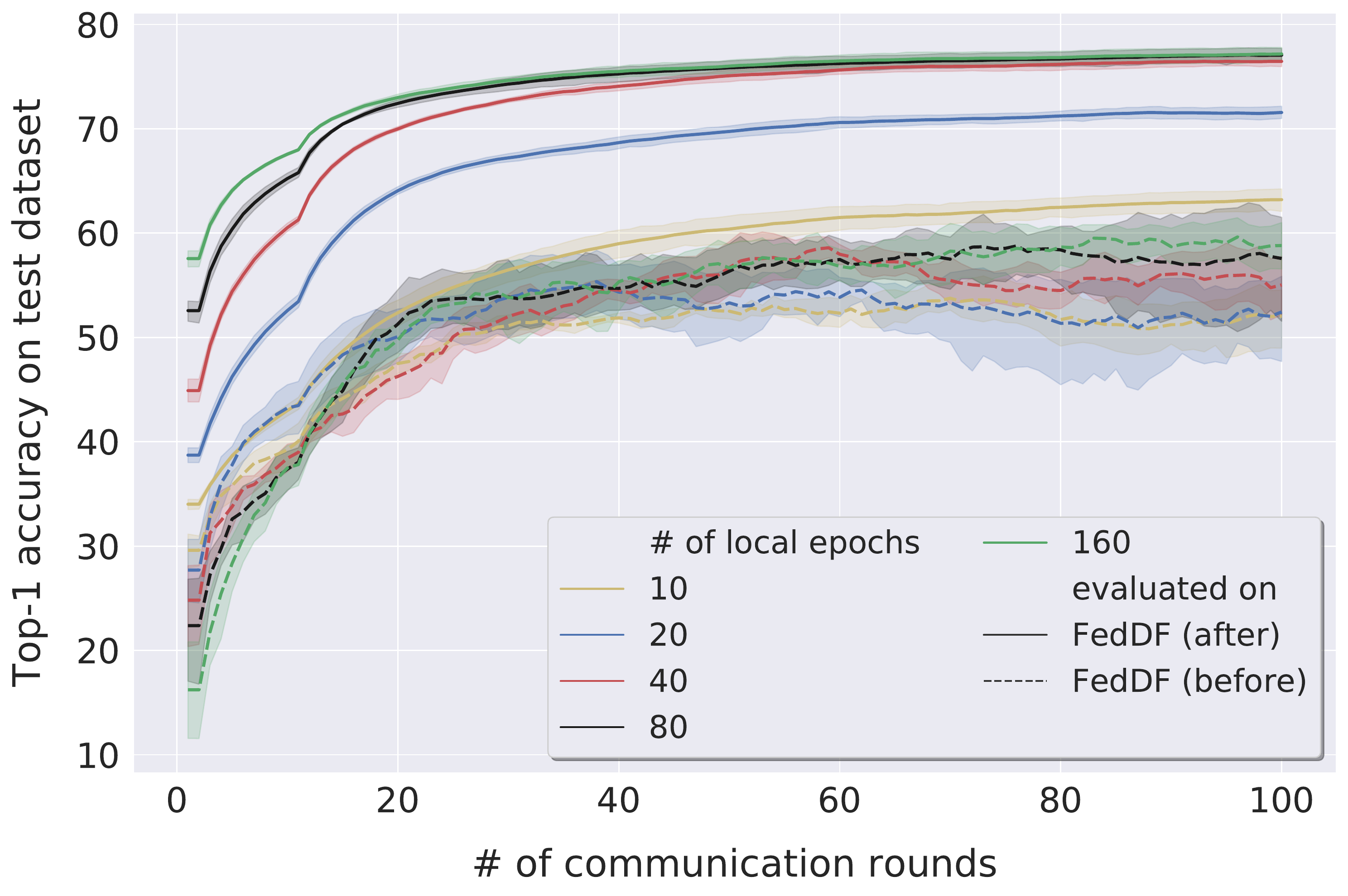}
		\label{fig:appendix_resnet8_cifar10_05_localdata_non_iid_100_kt_agg_vs_post_agg}
	}
	\vspace{-0.5em}
	\caption{\small
		\textbf{Understanding the learning behaviors of \algopt} on CIFAR-10 with ResNet-8 for $\alpha \!=\! 100$.
		For each communication round ($100$ in total),
		$40\%$ of the total $20$ clients are randomly selected.
		We report the top-1 accuracy (on three different seeds) on the test dataset.
	}
	\vspace{-0.5em}
	\label{fig:appendix_understanding_learning_behaviors_resnet8_cifar10_kt_non_iid_100_localdata}
\end{figure*}

\begin{figure*}[!h]
	\centering
	\subfigure[\small
		The learning behaviors of \algopt and \fedavg.
		We evaluate different \# of local epochs on $100\%$ local data.
	]{
		\includegraphics[width=0.475\textwidth,]{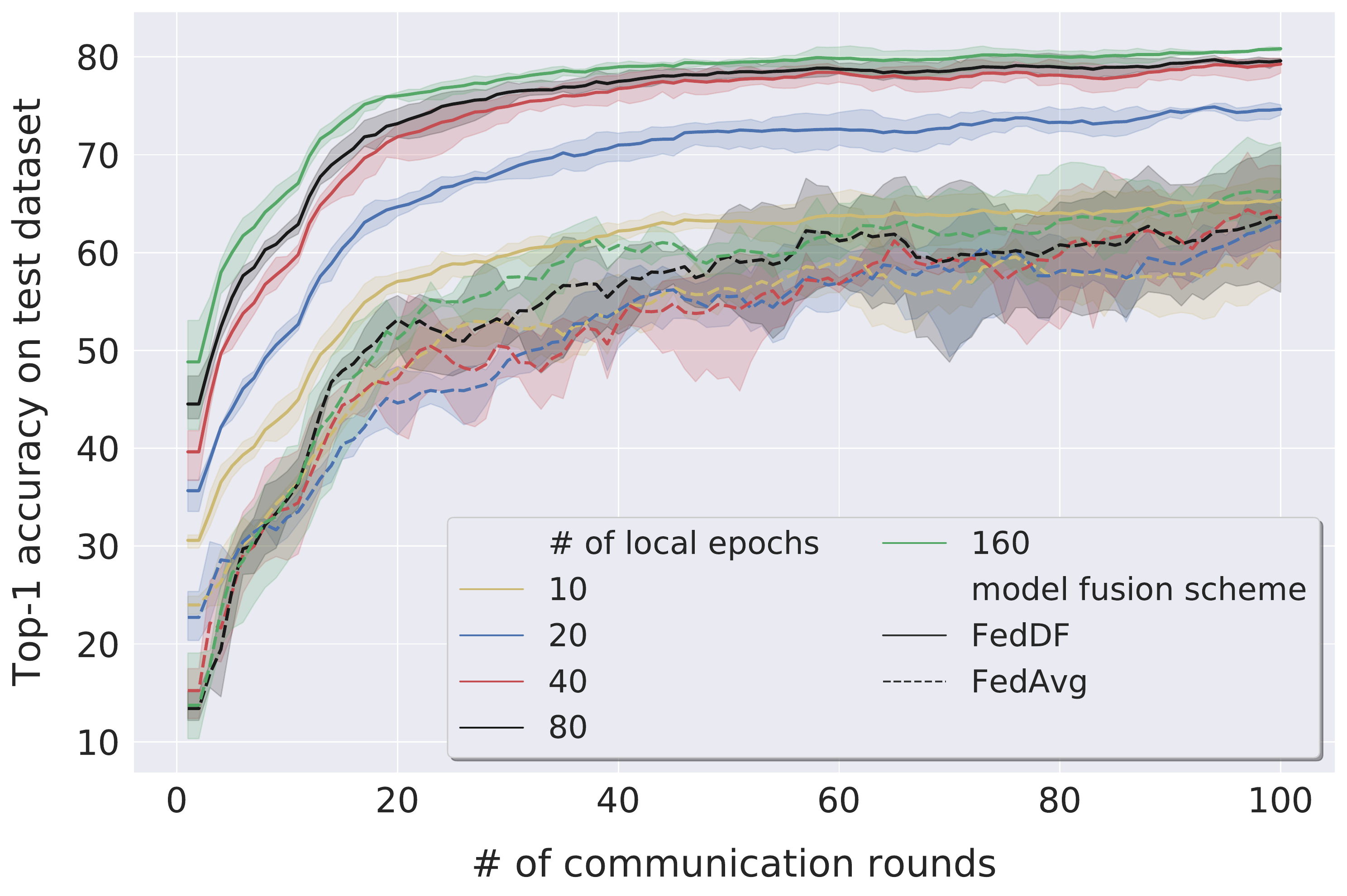}
		\label{fig:appendix_resnet8_cifar10_1_localdata_non_iid_1_fedavg_vs_kt_training_curves}
	}
	\hfill
	\subfigure[\small
		The fused model performance before (i.e.\ line 6 in Algorithm~\ref{alg:homogeneous_framework}) and after \algopt (i.e.\ line 10 in Algorithm~\ref{alg:homogeneous_framework}).
		We evaluate different \# of local epochs on $100\%$ local data.
	]{
		\includegraphics[width=0.475\textwidth,]{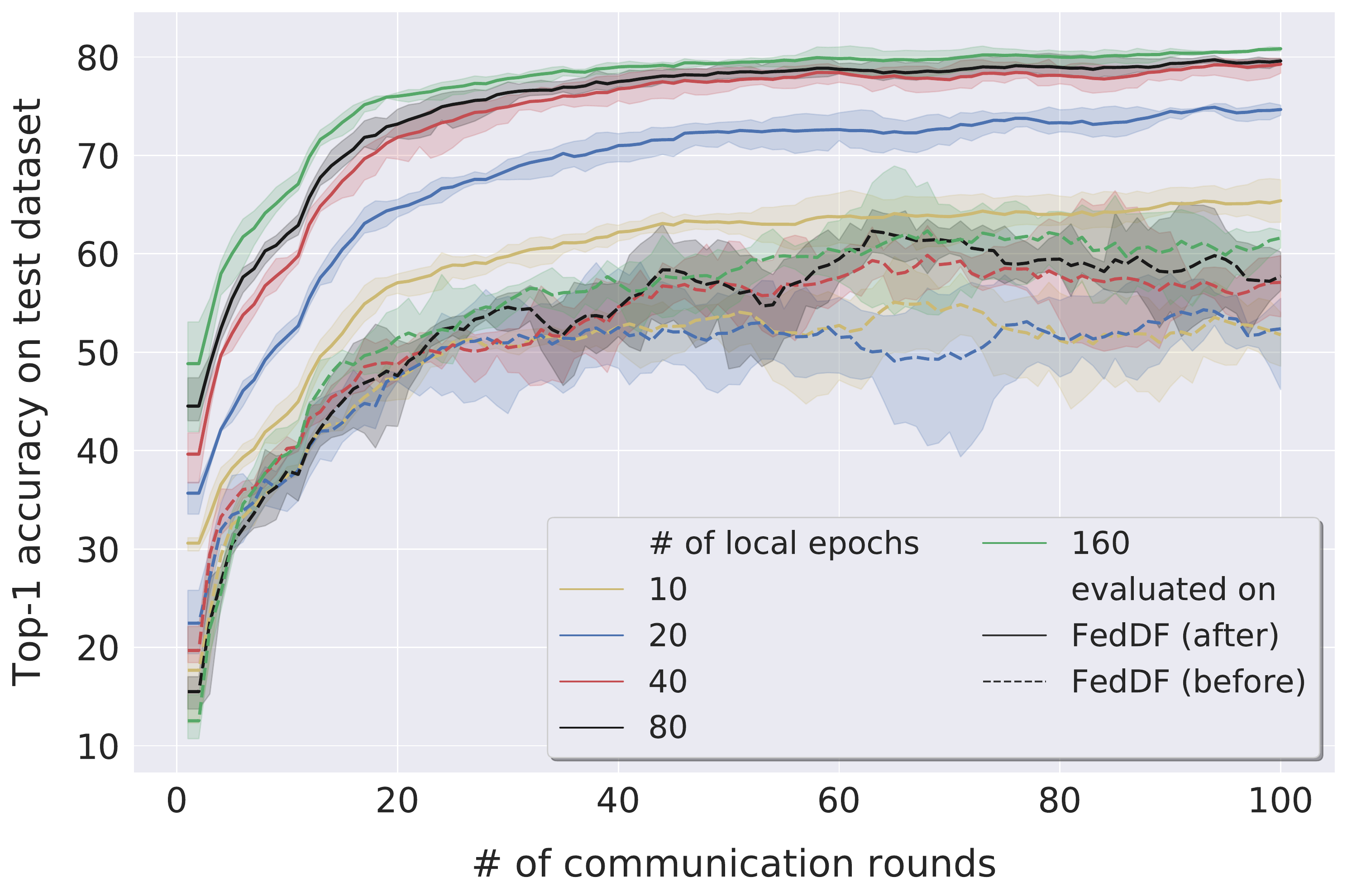}
		\label{fig:appendix_resnet8_cifar10_1_localdata_non_iid_1_kt_agg_vs_post_agg}
	}
	\subfigure[\small
		The learning behaviors of \algopt and \fedavg.
		We evaluate different \# of local epochs on $50\%$ local data.
	]{
		\includegraphics[width=0.475\textwidth,]{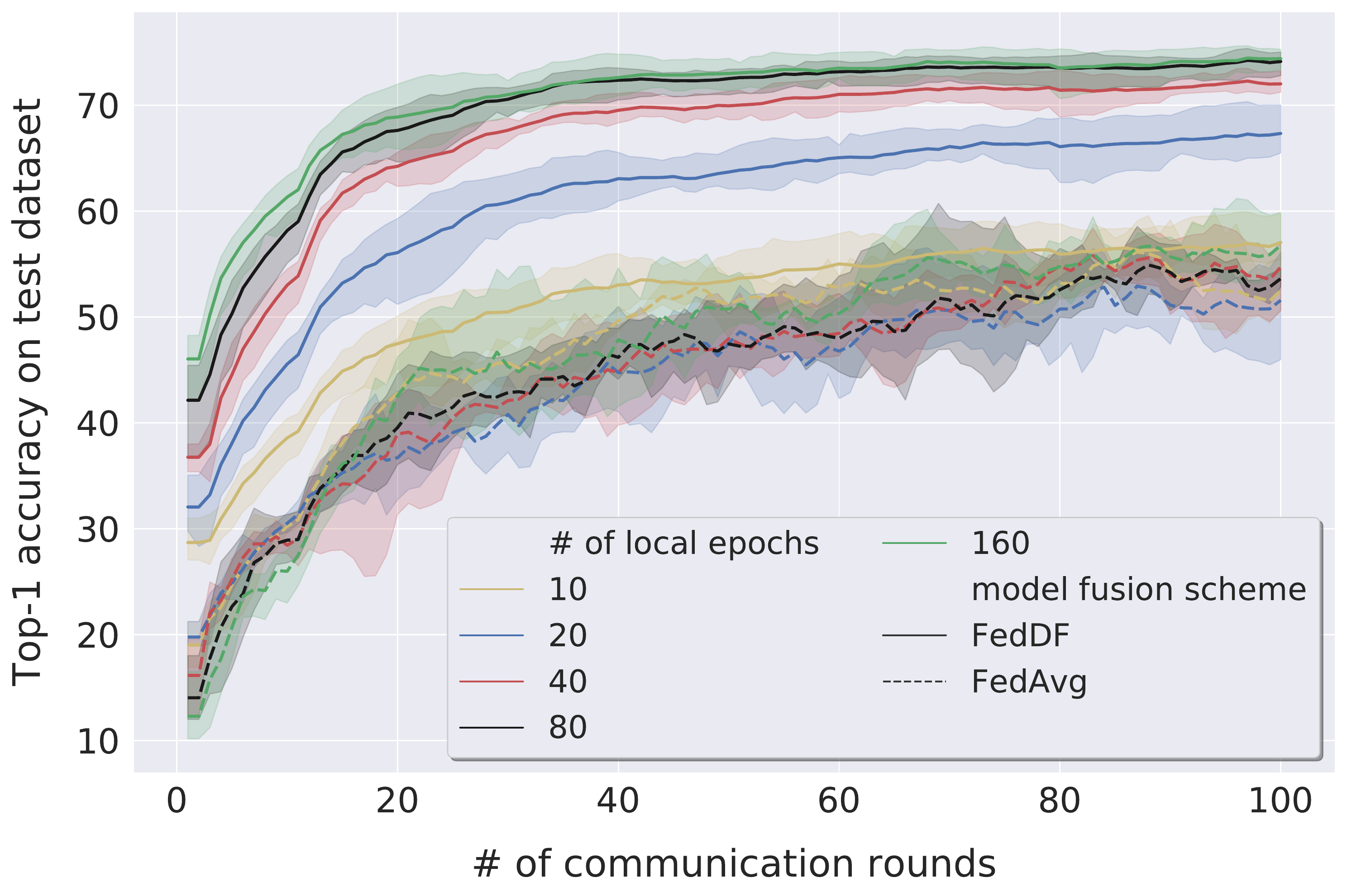}
		\label{fig:appendix_resnet8_cifar10_05_localdata_non_iid_1_fedavg_vs_kt_training_curves}
	}
	\hfill
	\subfigure[\small
		The fused model performance before (i.e.\ line 6 in Algorithm~\ref{alg:homogeneous_framework}) and after \algopt (i.e.\ line 10 in Algorithm~\ref{alg:homogeneous_framework}).
		We evaluate different \# of local epochs on $50\%$ local data.
	]{
		\includegraphics[width=0.475\textwidth,]{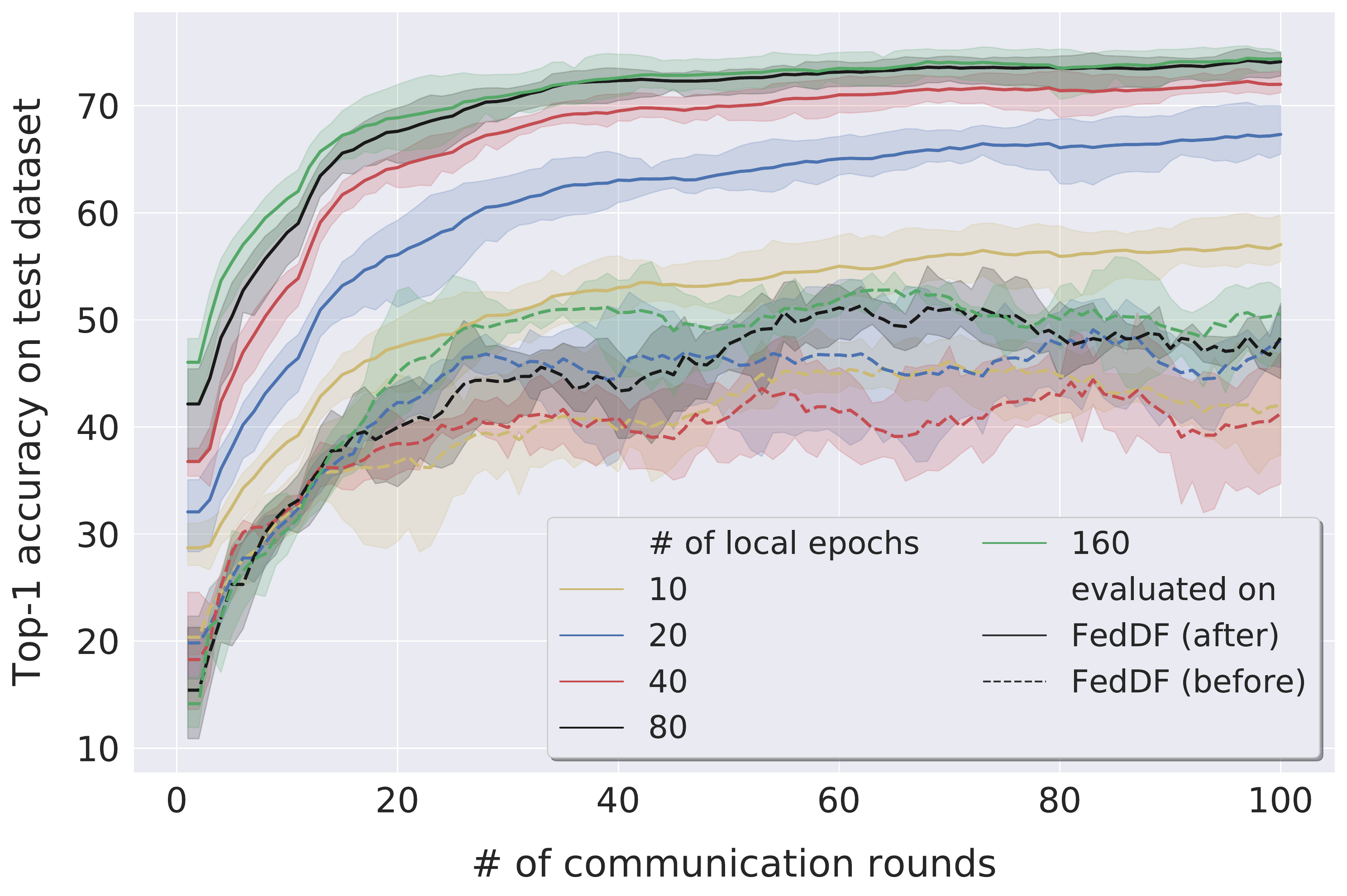}
		\label{fig:appendix_resnet8_cifar10_05_localdata_non_iid_1_kt_agg_vs_post_agg}
	}
	\vspace{-0.5em}
	\caption{\small
		\textbf{Understanding the learning behaviors of \algopt} on CIFAR-10 with ResNet-8 for $\alpha \!=\! 1$.
		For each communication round ($100$ in total),
		$40\%$ of the total $20$ clients are randomly selected.
		We report the top-1 accuracy (on three different seeds) on the test dataset.
	}
	\vspace{-0.5em}
	\label{fig:appendix_understanding_learning_behaviors_resnet8_cifar10_kt_non_iid_1_localdata}
\end{figure*}

\begin{figure*}[!h]
	\centering
	\subfigure[\small
		The learning behaviors of \algopt and \fedavg.
		We evaluate different \# of local epochs on $100\%$ local data.
	]{
		\includegraphics[width=0.475\textwidth,]{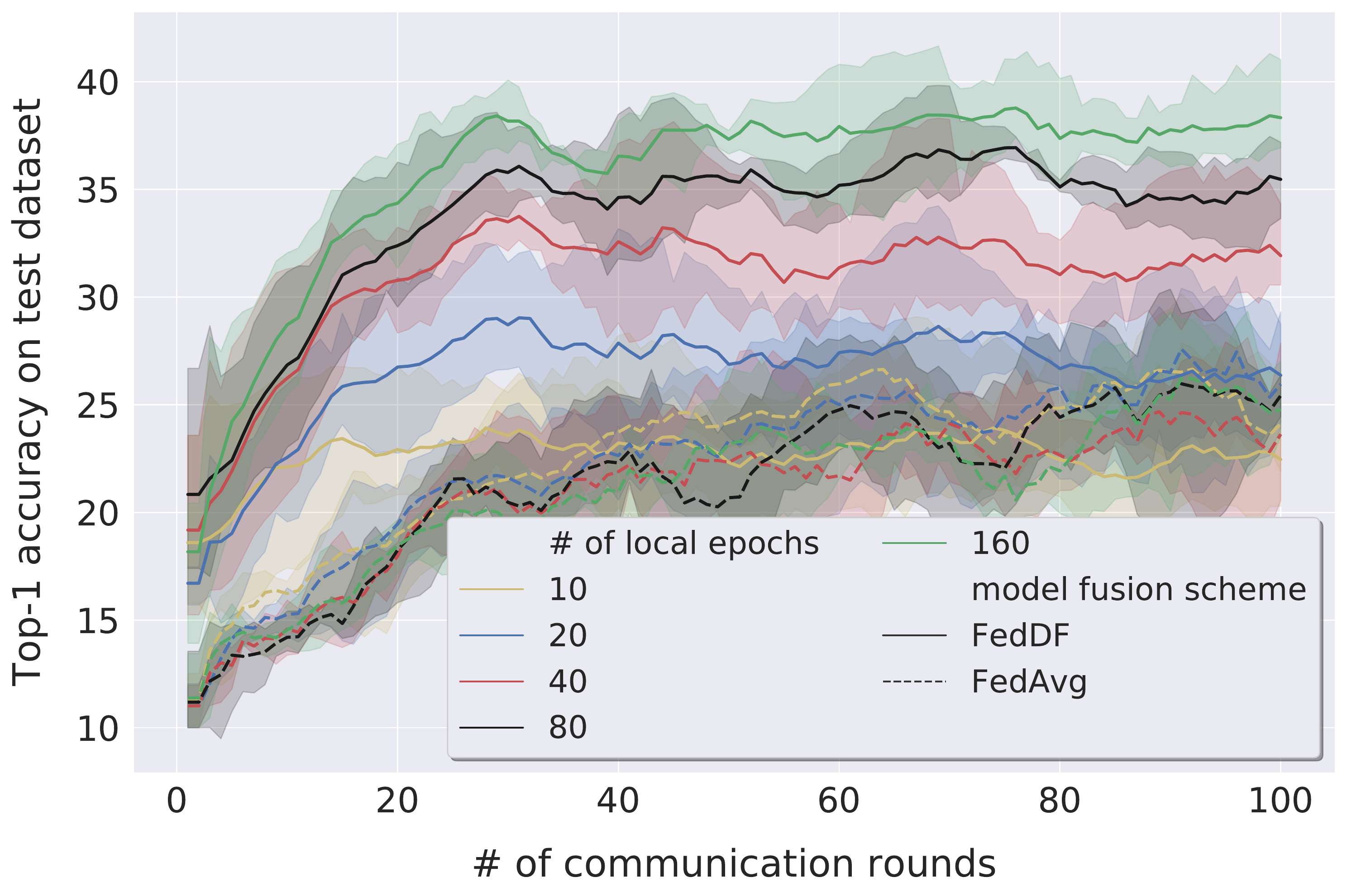}
		\label{fig:appendix_resnet8_cifar10_1_localdata_non_iid_001_fedavg_vs_kt_training_curves}
	}
	\hfill
	\subfigure[\small
		The fused model performance before (i.e.\ line 6 in Algorithm~\ref{alg:homogeneous_framework}) and after \algopt (i.e.\ line 10 in Algorithm~\ref{alg:homogeneous_framework}).
		We evaluate different \# of local epochs on $100\%$ local data.
	]{
		\includegraphics[width=0.475\textwidth,]{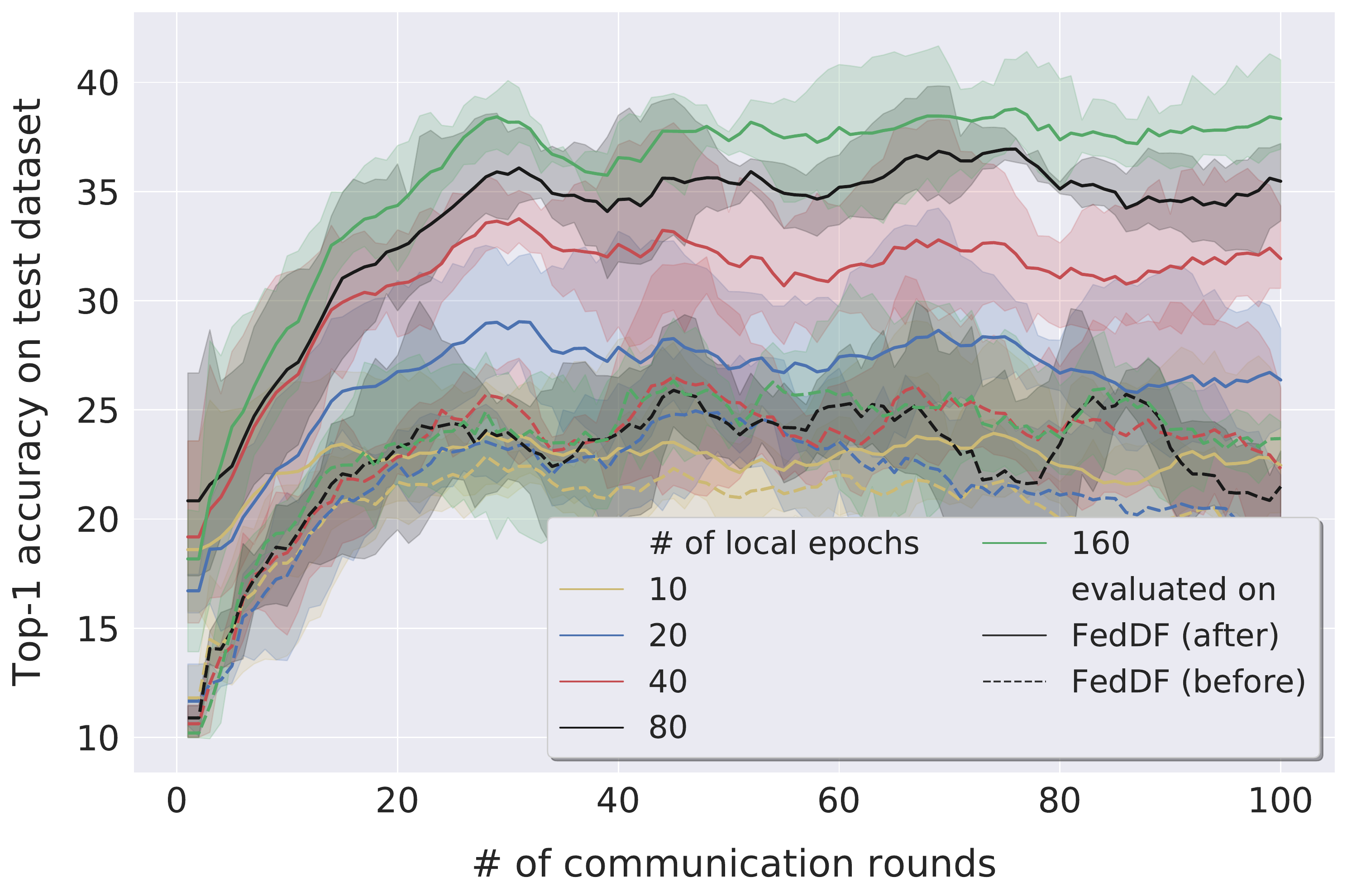}
		\label{fig:appendix_resnet8_cifar10_1_localdata_non_iid_001_kt_agg_vs_post_agg}
	}
	\subfigure[\small
		The learning behaviors of \algopt and \fedavg.
		We evaluate different \# of local epochs on $50\%$ local data.
	]{
		\includegraphics[width=0.475\textwidth,]{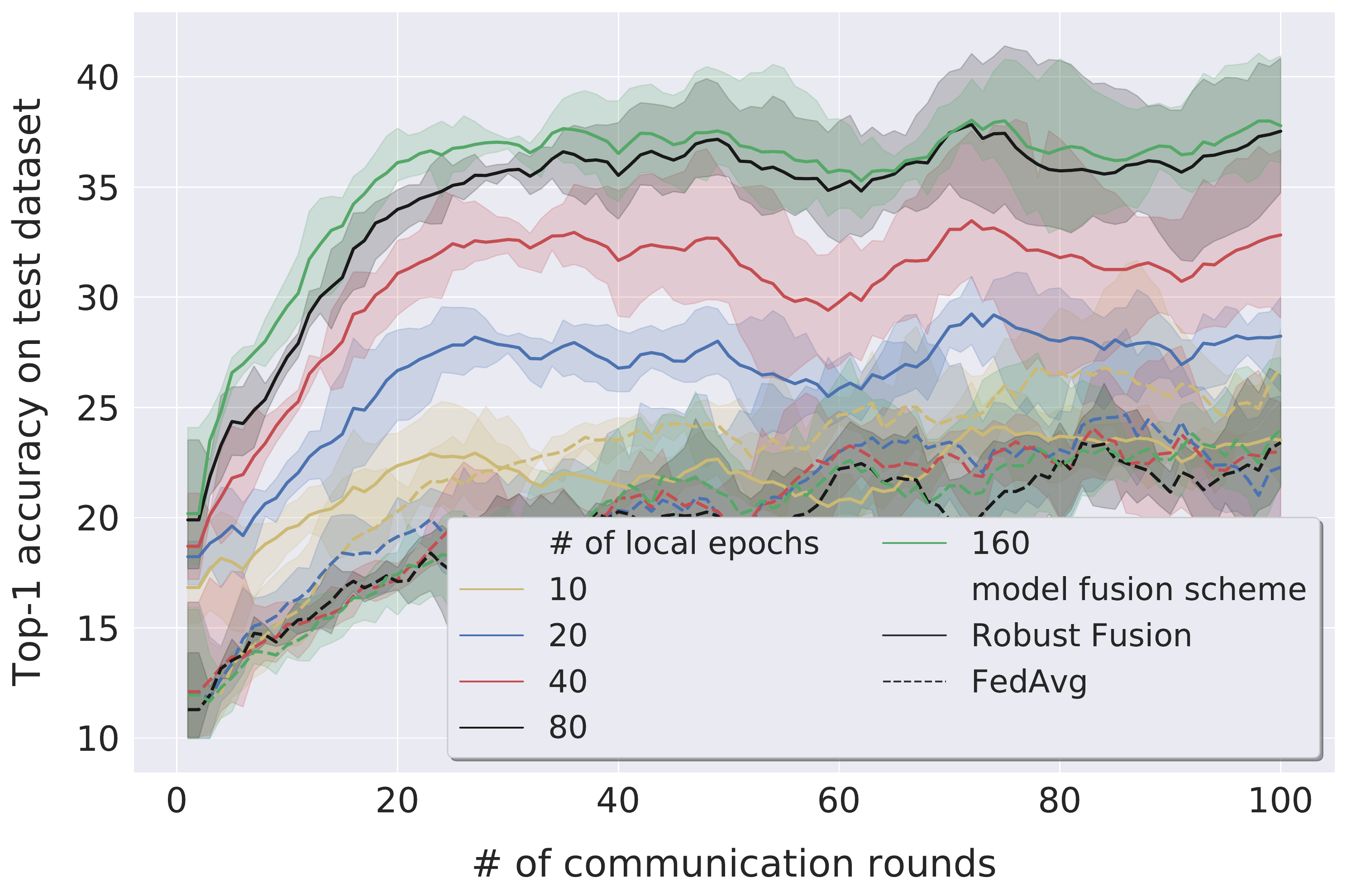}
		\label{fig:appendix_resnet8_cifar10_05_localdata_non_iid_001_fedavg_vs_kt_training_curves}
	}
	\hfill
	\subfigure[\small
		The fused model performance before (i.e.\ line 6 in Algorithm~\ref{alg:homogeneous_framework}) and after \algopt (i.e.\ line 10 in Algorithm~\ref{alg:homogeneous_framework}).
		We evaluate different \# of local epochs on $50\%$ local data.
	]{
		\includegraphics[width=0.475\textwidth,]{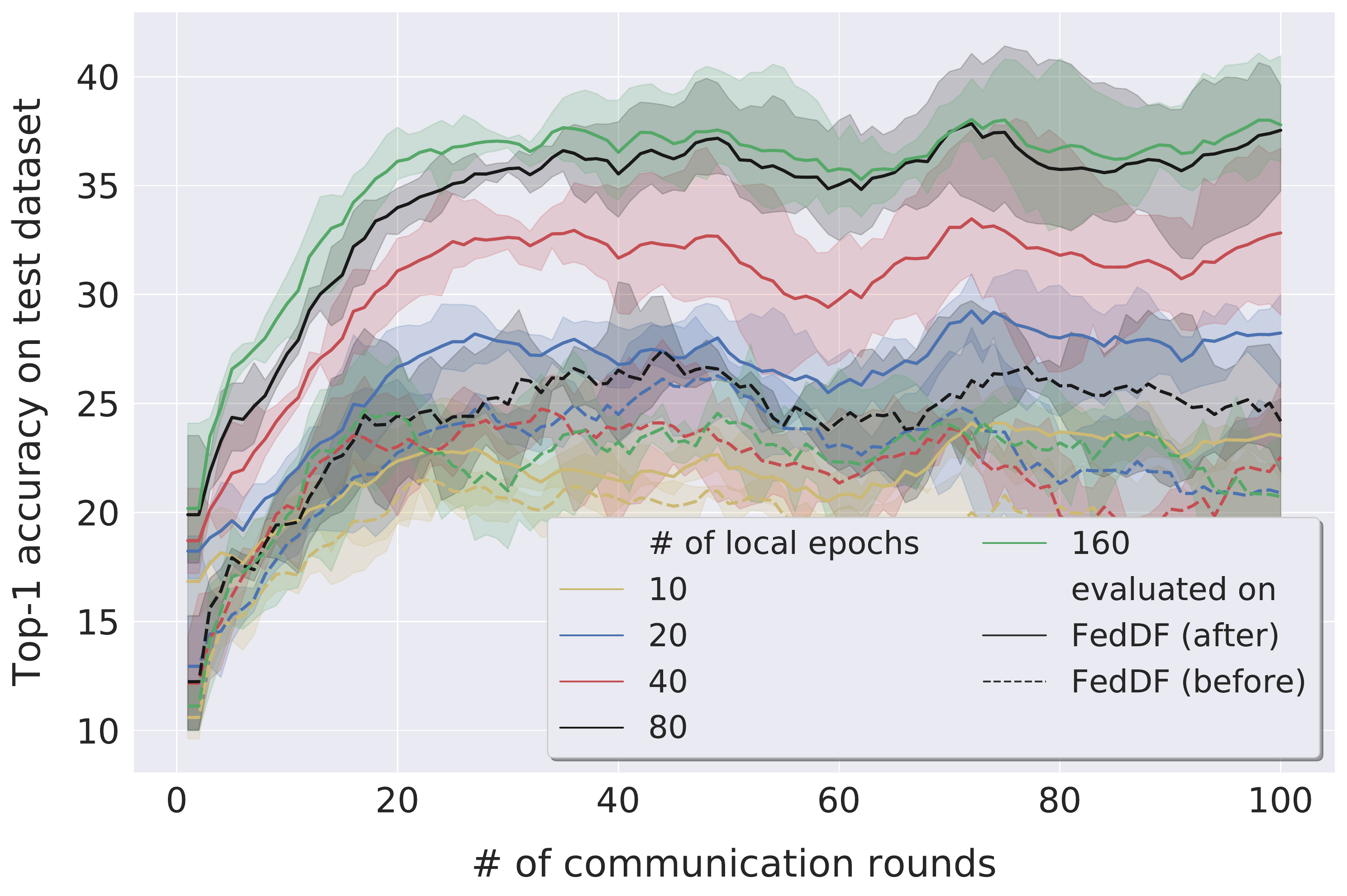}
		\label{fig:appendix_resnet8_cifar10_05_localdata_non_iid_001_kt_agg_vs_post_agg}
	}
	\vspace{-0.5em}
	\caption{\small
		\textbf{Understanding the learning behaviors of \algopt} on CIFAR-10 with ResNet-8 for $\alpha \!=\! 0.01$.
		For each communication round ($100$ in total),
		$40\%$ of the total $20$ clients are randomly selected.
		We report the top-1 accuracy (on three different seeds) on the test dataset.
	}
	\vspace{-0.5em}
	\label{fig:appendix_understanding_learning_behaviors_resnet8_cifar10_kt_non_iid_001_localdata}
\end{figure*}

\clearpage
\section{Details on Generalization Bounds} \label{sec:detailed_bounds}
The derivation of the generalization bound starts from the following notations.
In FL, each client has access to its own data distribution $\cD_i$
over domain $\Xi := \cX \times \cY$,
where $\cX \in \R^d$ is the input space and $\cY$ is the output space.
The global distribution on the server is denoted as $\cD$.
For the empirical distribution by the given dataset, we assume that each local model has access to an equal amount ($m$) of local data.
Thus, each local empirical distribution has equal contribution to the global
empirical distribution: $ \hat{\cD} = \frac{1}{K} \sum_{k=1}^K \hat{\cD}_k$,
where $ \hat{\cD}_k $ denotes the empirical distribution from client $ k $.

For our analysis we assume a binary classification task,
with hypothesis $h$ as a function $h: \cX \rightarrow \{ 0, 1 \}$.
The loss function of the task is defined as $\ell( h (\xx), y) = \abs{ \hat{y} - y }$, where $\hat{y} := h(\xx)$.
Note that $ \ell(\hat{y}, y ) $ is convex with respect to $ \hat{y} $.
We denote $ \argmin_{h \in \cH} L_{\hat{\cD}} (h) $ by $ h_{\hat{\cD}} $.

The theorem below leverages the domain measurement tools developed in multi-domain learning theory~\cite{ben2010theory}
and provides insights for the generalization bound of the ensemble\footnote{
	The uniformly weighted hypothesis average in multi-source adaptation
	is equivalent to the ensemble of a list of models,
	by considering the output of each hypothesis/model.
} of local models (trained on local empirical distribution $\hat{\cD}_i$).

\begin{theorem}
	\label{thm:formal_risk_upper_bound_for_ensemble_model}
	Let $\cH$ be a hypothesis class with $\text{VCdim} (\cH) \leq d < \infty $.
	The difference between $ L_\cD ( \frac{1}{K} \sum_k h_{\hat{\cD}_k} ) $ and $ L_{\hat{\cD}} ( h_{\hat{\cD}} ) $,
	i.e., the distance between the risk of our ``ensembled'' model in \algopt
	and the empirical risk of the ``virtual ERM'' with access to all local data,
	can be bounded with probability at least $ 1 - \delta $:
	\begin{small}
		\begin{align*}
			\begin{split}
				L_\cD \Big( \frac{1}{K} \sum_k h_{\hat{\cD}_k} \Big)
				&\leq
				L_{\hat{\cD}} ( h_{\hat{\cD}} )
				+ \frac{ 4 + \sqrt{ \log ( \tau_{\cH} (2m) ) } }{ \delta / K \sqrt{ 2m } }
				+ \frac{1}{K} \sum_{k} \left( \frac{1}{2} d_{\cH \Delta \cH} (\cD_k, \cD) + \lambda_k \right)
				\,,
			\end{split}
		\end{align*}
		where
		$ \hat{\cD} \!=\! \frac{1}{K} \sum_{k} \hat{\cD}_k $,
		$d_{\cH \Delta \cH}$ measures the domain discrepancy between two distributions~\cite{ben2010theory},
		and $\lambda_k \!=\! \inf_{h \in \cH} \left( \cL_{\cD} (h) \!+\! \cL_{\cD_k}(h) \right)$.
		$\tau_{\cH}$ is the growth function of $\cH$---it satisfies for all $m$, $ \tau_{\cH} (m) \leq \sum_{i=0}^{d} { m \choose i } $, and if $ m > d + 1 $ then $ \tau_{\cH} (m) \leq ( \frac{em}{d} )^{d} $.
	\end{small}
\end{theorem}

\begin{remark}
	Theorem~\ref{thm:formal_risk_upper_bound_for_ensemble_model}
	shows that,
	the upper bound on the risk of the ensemble of $K$ local models on $\cD$
	mainly consists of 1)
	the empirical risk of a model trained on the global empirical distribution $\hat{\cD} = \frac{1}{K} \sum_k \hat{\cD}_k$,
	and 2) terms dependent on the distribution discrepancy between $\cD_k$ and $\cD$.
\end{remark}

The ensemble of the local models
sets the performance upper bound for the later distilled model on the test domain
as shown in Figure~\ref{fig:heterogeneous_system}.
Theorem~\ref{thm:informal_risk_upper_bound_for_ensemble_model} shows that
compared to a model trained on aggregated local data (ideal case),
the performance of an ensemble model on the test distribution
is affected by the domain discrepancy between local distributions $\cD_k$'s and the test distribution $\cD$.
The shift between the distillation and the test distribution
determines the knowledge transfer quality between these two distributions
and hence the test performance of the fused model.
Through the lens of the domain adaptation theory~\cite{ben2010theory},
we can better spot the potential influence/limiting factors on our ensemble distillation procedure.

\begin{remark}
	In the area of multiple-source adaptation,
	\cite{mansour2009domain,hoffman2018algorithms} point out that
	the standard convex combinations of the source hypotheses may perform poorly on the test distribution.
	They propose combinations with weights derived from source distributions.
	However, FL scenarios require the server only access local models
	without any further local information.
	Thus we choose to uniformly average over local hypotheses as our global hypothesis.
	A privacy-preserved local distribution estimation is left for future work.
\end{remark}

\subsection{Proof for Generalization Bounds}

\begin{theorem}[Uniform Convergence~\cite{shalev2014understanding}]~\label{thm:uniform_convergence}
	Let $\cH$ be a class and let $\tau_{\cH}$ be its growth function. Then, for every $\cD$ and every $\delta \in (0, 1)$, with probability of at least $1 - \delta$ over the choice of $ S \sim \cD^{m} $, we have
	$$ \abs{ L_{\cD} (h) - L_{S} (h) } \leq \frac{ 4 + \sqrt{ \log ( \tau_{\cH} (2m) ) } }{ \delta \sqrt{ 2m } }. $$
\end{theorem}

\begin{lemma}[Sauer's Lemma~\cite{shalev2014understanding}]~\label{lemma:sauer}
	Let $\cH$ be a hypothesis class with $\text{VCdim} (\cH) \leq d < \infty $. Then for all $m$, $ \tau_{\cH} (m) \leq \sum_{i=0}^{d} { m \choose i } $. In particular, if $ m > d + 1 $ then $ \tau_{\cH} (m) \leq ( \frac{em}{d} )^{d} $.
\end{lemma}

\begin{theorem}[Domain adaptation \cite{ben2010theory}] \label{thm:simplified_adaptation_empirical_bound}
	Considering the distributions $\cD_S$ and $\cD_T$,
	for every $h \in \cH$ and any $\delta \in (0, 1)$,
	with probability at least $1 - \delta$ (over the choice of the samples),
	there exists:
	\begin{align}
		\textstyle
		L_{\cD_T}(h)
		\leq L_{\cD_S}(h) + \frac{1}{2} d_{\cH \Delta \cH} (\cD_S, \cD_T) + \lambda \,,
	\end{align}
	where $\lambda = L_{\cD_S}(h^\star) + L_{\cD_T}(h^\star)$.
	$h^\star := \argmin_{h \in \cH} L_{\cD_S} (h) + L_{\cD_T} (h) $
	corresponds to \emph{ideal joint hypothesis} that minimizes the combined error.
\end{theorem}

\begin{proof}[Proof of Theorem~\ref{thm:formal_risk_upper_bound_for_ensemble_model}]
	We start from the risk of our ``ensembled'' model $ L_\cD ( \frac{1}{K} \sum_k h_{\hat{\cD}_k} ) $
	and derive a series of upper bounds.

	\paragraph{Considering the distance between
		$L_\cD ( \frac{1}{K} \sum_k h_{\hat{\cD}_k} )$ and $L_{\hat{\cD}} ( h_{\hat{\cD}} )$.
	}

	By convexity of $ \ell $ and Jensen inequality, we have
	\begin{align}
		L_\cD ( \frac{1}{K} \sum_k h_{\hat{\cD}_k} )
		\leq  \frac{1}{K} \sum_k L_\cD ( h_{\hat{\cD}_k} ) \,.
	\end{align}

	Using the domain adaptation theory in Theorem~\ref{thm:simplified_adaptation_empirical_bound},
	we transfer from domain $ \cD $ to $ \cD_k $,
	\begin{align}
		L_\cD ( h_{\hat{\cD}_k} )
		\leq L_{\cD_k} ( h_{\hat{\cD}_k} ) + \frac{1}{2} d_{\cH \Delta \cH} (\cD_k, \cD) + \lambda_k \,,
	\end{align}
	where $\lambda_k := \cL_{\cD} (h^\star) + \cL_{\cD_k}(h^\star) $
	and $ h^\star := \argmin_{h \in \cH} \cL_{\cD} (h) + \cL_{\cD_k}(h) $.

	We can bound the risk with its empirical counterpart.
	Applying Theorem~\ref{thm:uniform_convergence} yields
	\begin{align}
		L_{\cD_k} ( h_{\hat{\cD}_k} )
		\leq L_{\hat{\cD}_k} ( h_{\hat{\cD}_k} ) + \frac{ 4 + \sqrt{ \log ( \tau_{\cH} (2m) ) } }{ \delta / K \sqrt{ 2m } } \,,
	\end{align}
	where $ \tau_{\cH} $ denotes the growth function of $\cH$~\footnote{
		Note that with Lemma~\ref{lemma:sauer}, we could bound $\tau_{\cH}$ by a polynomial function of $m$ and $d$.
	}.

	Thus for $ K $ sources, we have
	\begin{align}
		\begin{split}
			& \Prob{S_1 \sim \cD^m_1, \dots, S_K \sim \cD^m_K}{ \bigcap\limits_{k=1}^{K} \left\{ L_{\cD_k} ( h_{\hat{\cD}_k} )
				\leq L_{\hat{\cD}_k} ( h_{\hat{\cD}_k} ) + \frac{ 4 + \sqrt{ \log ( \tau_{\cH} (2m) ) } }{ \delta / K \sqrt{ 2m } } \right\} } \\
			& =
			1 - \Prob{S_1 \sim \cD^m_1, \dots, S_K \sim \cD^m_K}{ \bigcup\limits_{k=1}^{K} \left\{ L_{\cD_k} ( h_{\hat{\cD}_k} )
				\ge L_{\hat{\cD}_k} ( h_{\hat{\cD}_k} ) + \frac{ 4 + \sqrt{ \log ( \tau_{\cH} (2m) ) } }{ \delta / K \sqrt{ 2m } } \right\} } \\
			& \ge
			1 - \sum_{k=1}^K \Prob{S_1 \sim \cD^m_1, \dots, S_K \sim \cD^m_K}{ \left\{ L_{\cD_k} ( h_{\hat{\cD}_k} )
				\ge L_{\hat{\cD}_k} ( h_{\hat{\cD}_k} ) + \frac{ 4 + \sqrt{ \log ( \tau_{\cH} (2m) ) } }{ \delta / K \sqrt{ 2m } } \right\} } \\
			& \ge 1 - \delta \,.
		\end{split}
	\end{align}

	Based on the definition of ERM, we have
	$ L_{\hat{\cD}_k} ( h_{\hat{\cD}_k} ) \leq L_{\hat{\cD}_k} ( h_{\hat{\cD}} ) $,
	where $ h_{\hat{\cD}} $ corresponds to the classifier trained with data from all workers.
	By using the definition of $ \hat{ \cD } $ ($ \hat{ \cD } = \frac{1}{K} \sum_k \hat{\cD}_k $ )
	and the linearity of expectation, we have
	\begin{align}
		\frac{1}{K} \sum_k L_{\hat{\cD}_k} ( h_{\hat{\cD}_k} )
		\leq  \frac{1}{K} \sum_k L_{\hat{\cD}_k} ( h_{\hat{\cD}} ) = L_{\hat{\cD}} ( h_{\hat{\cD}} ) \,.
	\end{align}

	Putting these equations together,
	we have with probability of at least $ 1 - \delta $ over $ S_1 \sim \cD^m_1, \dots, S_K \sim \cD^m_K $
	that
	\begin{align*}
		\begin{split}
			L_\cD ( \frac{1}{K} \sum_k h_{\hat{\cD}_k} )
			&\leq  \frac{1}{K} \sum_k L_\cD ( h_{\hat{\cD}_k} ) \\
			&\leq \frac{1}{K} \sum_k \left(
			L_{\cD_k} ( h_{\hat{\cD}_k} ) + \frac{1}{2} d_{\cH \Delta \cH} (\cD_k, \cD) + \lambda_k
			\right) \\
			&\leq \frac{1}{K} \sum_k \left(
			L_{\hat{\cD}_k} ( h_{\hat{\cD}_k} ) + \frac{ 4 + \sqrt{ \log ( \tau_{\cH} (2m) ) } }{ \delta / K \sqrt{ 2m } }
			+ \frac{1}{2} d_{\cH \Delta \cH} (\cD_k, \cD) + \lambda_k
			\right) \\
			&\leq \frac{1}{K} \sum_k L_{\hat{\cD}_k} ( h_{\hat{\cD}_k} )
			+ \frac{ 4 + \sqrt{ \log ( \tau_{\cH} (2m) ) } }{ \delta / K \sqrt{ 2m } }
			+ \frac{1}{K} \sum_{k} \left( \frac{1}{2} d_{\cH \Delta \cH} (\cD_k, \cD) + \lambda_k \right) \\
			&\leq L_{\hat{\cD}} ( h_{\hat{\cD}} )
			+ \frac{ 4 + \sqrt{ \log ( \tau_{\cH} (2m) ) } }{ \delta / K \sqrt{ 2m } }
			+ \frac{1}{K} \sum_{k} \left( \frac{1}{2} d_{\cH \Delta \cH} (\cD_k, \cD) + \lambda_k \right) \,,
		\end{split}
	\end{align*}
	where $\lambda_k = \inf_{h \in \cH} \left( \cL_{\cD} (h) + \cL_{\cD_k}(h) \right)$.

\end{proof}

\end{document}